\documentclass[preprint, 5p]{elsarticle}

\usepackage[cm]{fullpage}
\usepackage[utf8]{inputenc} % allow utf-8 input
\usepackage[T1]{fontenc}    % use 8-bit T1 fonts
\usepackage{hyperref}       % hyperlinks
\usepackage{url}            % simple URL typesetting
\usepackage{booktabs}       % professional-quality tables
\usepackage{amsfonts}       % blackboard math symbols
\usepackage{nicefrac}       % compact symbols for 1/2, etc.
\usepackage{microtype}      % microtypography
\usepackage{graphicx}
\usepackage{amsmath}
\usepackage{amssymb}
\usepackage{bm}
\usepackage{bbm}
\usepackage{enumitem}
\usepackage{relsize}
\usepackage{siunitx}
\usepackage{amsthm}
\usepackage{comment}
\usepackage[style=plain]{floatrow}
\usepackage{tabularx}
\usepackage[table]{xcolor}
\usepackage{subcaption}
\usepackage{multirow}
\usepackage{makecell}
\usepackage{cancel}
\usepackage{pifont}
\usepackage{manyfoot}
\usepackage[normalem]{ulem}

\DeclareMathOperator*{\argmax}{arg\,max}
\DeclareMathOperator*{\argmin}{arg\,min}

\usepackage{soul}
\definecolor{lightred}{rgb}{1, 0.7, 0.7}

\sethlcolor{lightred}

\usepackage{listings}
\usepackage{graphicx,color}
\usepackage{booktabs}
\usepackage{courier}
\usepackage[bb=libus]{mathalpha}

\newcolumntype{s}{>{\columncolor[gray]{0.95}} p{.46\textwidth}}

\newtheorem{lemma}{Lemma}
\newtheorem{approach}{Approach}
\newtheorem{definition}{Definition}

\newcounter{formula}

\AtBeginDocument{%
  \def\thefnote{\myfnsymbol{fnote}}}
\makeatletter
\def\myfnsymbol#1{\expandafter\@myfnsymbol\csname c@#1\endcsname}
\def\@myfnsymbol#1{\ifcase #1\or $\dagger$\or $\dagger\dagger$\else \@ctrerr\fi}
\def\fntext[#1]#2{\g@addto@macro\@fnotes{%
   \refstepcounter{fnote}\elsLabel{#1}%
   \def\thefootnote{\thefnote}% <-- corrected
   \global\setcounter{footnote}{\c@fnote}%
   \footnotetext{#2}}}
\makeatother

\makeatletter
\def\ps@pprintTitle{%
  \let\@oddhead\@empty
  \let\@evenhead\@empty
  \let\@oddfoot\@empty
  \let\@evenfoot\@oddfoot
}
\makeatother

\def\fnm#1{\leavevmode\hbox{#1}}%
\def\sur#1{\unskip~\nobreak\leavevmode\hbox{#1}}%

\begin{document}
\begin{frontmatter}
\title{Towards Symbolic XAI -- Explanation Through Human Understandable Logical Relationships Between Features}

\author[1,2]{\fnm{Thomas} \sur{Schnake}\corref{cor1}\fnref{fn1}}
\ead{t.schnake@tu-berlin.de}
\cortext[cor1]{Corresponding author}

\author[1,2]{\fnm{Farnoush} \sur{Rezaei Jafari} \fnref{fn1}}

\author[1,2]{\fnm{Jonas} \sur{Lederer}}

\author[1,2]{\fnm{Ping} \sur{Xiong}}

\author[1,2,3]{\fnm{Shinichi} \sur{Nakajima}}

\author[1,2]{\fnm{Stefan} \sur{Gugler}}

\author[1,6]{\fnm{Grégoire} \sur{Montavon}\corref{cor1}}
\ead{gregoire.montavon@fu-berlin.de}

\author[1,2,4,5]{\fnm{Klaus-Robert} \sur{M\"uller}\corref{cor1}}
\ead{klaus-robert.mueller@tu-berlin.de}

\fntext[fn1]{These authors contributed equally to this work.}

\affiliation[1]{organization={Berlin  Institute  for  the Foundations  of  Learning  and  Data -- BIFOLD},
addressline={10623 Berlin},
country={Germany}}

\affiliation[2]{organization={Machine Learning Group, Technical University of Berlin},
addressline={10623 Berlin},
country={Germany}}

\affiliation[3]{organization={RIKEN AIP},
    city={103-0027 Tokyo},
country={Japan}}

\affiliation[4]{organization={Department of Artificial Intelligence, Korea University},
addressline={Seoul 136-713},
country={Korea}}

\affiliation[5]{organization={Max Planck Institut f{\"u}r Informatik},
addressline={66123 Saarbr{\"u}cken},
country={Germany}}

\affiliation[6]{organization={Department of Mathematics and Computer Science, Free University of Berlin},
adressline={14195 Berlin},
country={Germany}}

\begin{abstract}

Explainable Artificial Intelligence (XAI) plays a crucial role in fostering transparency and trust in AI systems, where traditional XAI approaches typically offer one level of abstraction for explanations, often in the form of heatmaps highlighting single or multiple input features. However, we ask whether abstract reasoning or problem-solving strategies of a model may also be relevant, as these align more closely with how humans approach solutions to problems. We propose a framework, called Symbolic XAI, that attributes relevance to symbolic queries expressing logical relationships between input features, thereby capturing the abstract reasoning behind a model’s predictions. The methodology is built upon a simple yet general multi-order decomposition of model predictions. This decomposition can be specified using higher-order propagation-based relevance methods, such as GNN-LRP, or perturbation-based explanation methods commonly used in XAI. The effectiveness of our framework is demonstrated in the domains of natural language processing (NLP), vision, and quantum chemistry (QC), where abstract symbolic domain knowledge is abundant and of significant interest to users. The Symbolic XAI framework provides an understanding of the model’s decision-making process that is both flexible for customization by the user and human-readable through logical formulas.

\end{abstract}

\end{frontmatter}

\section{Introduction}
Machine learning (ML) algorithms have increasingly become part of everyday life in both private and public sectors, playing crucial roles in data analysis, prediction, and automation. However, alongside the growing adoption of ML algorithms, there are increasing concerns about the potential risks associated with these models \cite{bengio2023managing,EU2021aiact}. A key step towards creating safer and trustworthy artificial intelligence (AI) systems is to understand their underlying decision-making strategies, particularly whether they align with human expectations on how problems should be tackled at an abstract level \cite{Panigutti2023XAIinAIAct, zhuang2020misalignedAI, Bommasani2021FoundationModels, stray2020aialignment,sucholutsky2023getting}. Since most models are large, highly nonlinear, and involve complex feature interactions, their predictions tend to remain opaque unless equipped with explanation capabilities. 

In this work, we address the case where the model is not interpretable by design, and interpretability is acquired by means of so-called `post-hoc' explanation strategies. In recent years, many different post-hoc explanation methods have been proposed, encompassing first-order \cite{baehrens2010explain,bach-plos15,lund17unified,sundar17axiomatic, samek2016evaluating}, second-order \cite{eberle22bilrp, Lund2020shapInterVal, janizek2021IntHess, ying19gnnexplainer}, and higher-order \cite{schnake22gnnlrp, Kedar2020shapTaylor} feature interactions. Notably, previous research has mostly focused on first-order explanation methods, which have been applied successfully across many disciplines (see e.g. \cite{keyl2022patient,keyl2023single,klauschen-pathology24,schuett2017dtnn,schutt2018schnet,eberle2023insightful,arras-lncs19,hense2024xmil,VielhabenPR24virt_insp_layer,becker2024audiomnist, jafari2024mambalrp}). 

\begin{figure*}
    \centering
    \includegraphics[width=.975\textwidth]{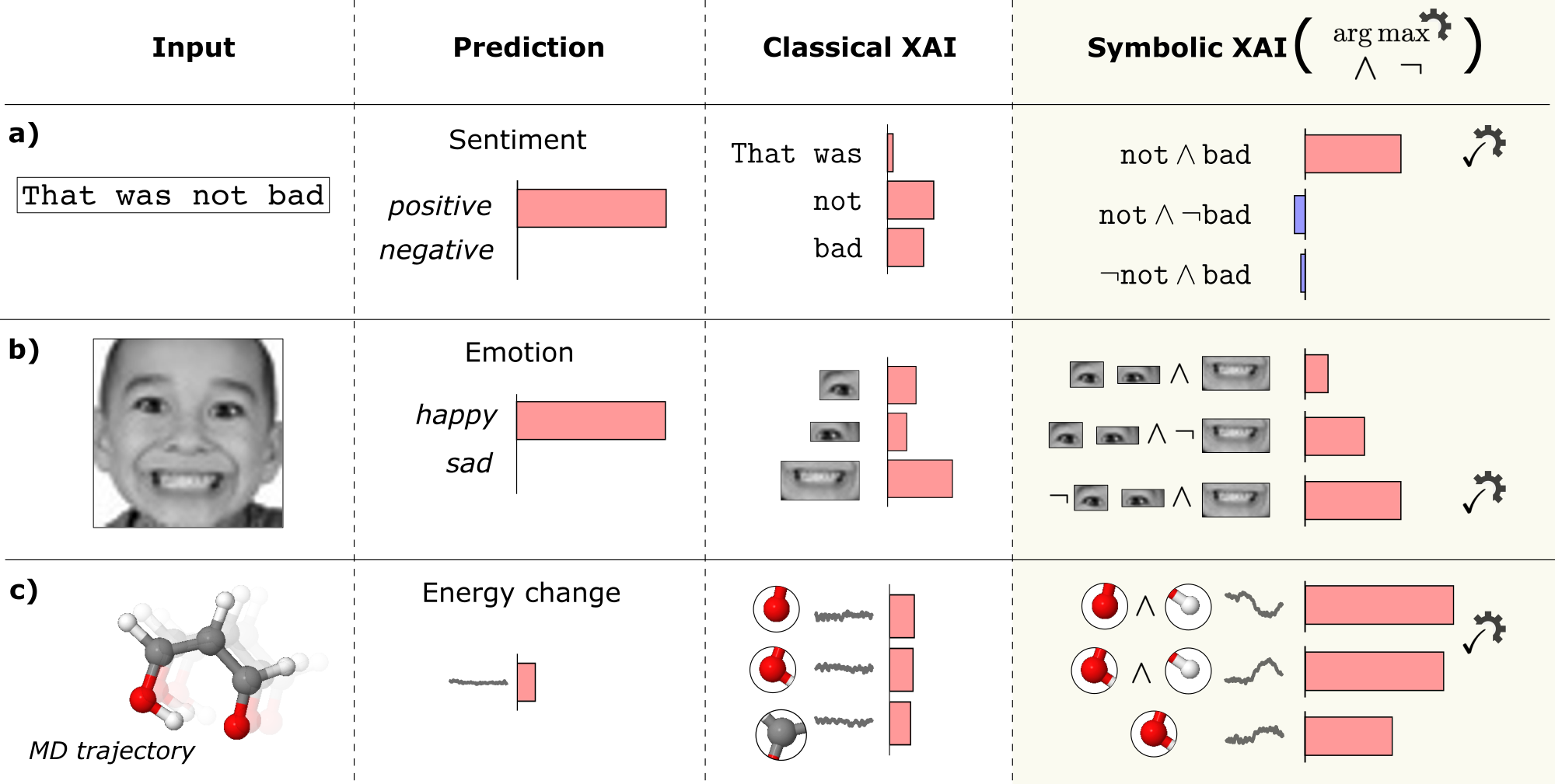}
    \caption{Schematic overview of how Symbolic XAI explains the predictions of an ML model.
    We show the input data, the model’s prediction, the output of a classic explanation method (layer-wise relevance propagation in this case), and the result from the Symbolic XAI framework for three use cases:
    \textbf{a)}, the sentiment prediction of a sentence is shown; 
    \textbf{b)}, the facial emotion prediction from an image; 
    and \textbf{c)}, the energy fluctuation over a molecular dynamics  trajectory. 
    The check mark in the column for Symbolic XAI denotes the selected query with the highest relevance.
    To measure the energy variation in \textbf{c)}, we determine the difference between the mean energy of two metastable states in the trajectory (see Section \ref{sect:qc_experiments}). \textbf{c)} provides the most relevant atoms and queries in the Classical XAI and Symbolic XAI column, respectively. For Classical XAI, we consider the relevance of all atoms, while for Symbolic XAI, we consider all conjunctive atom interactions with a maximum order of four.}
    \label{fig:intro_exemplary_example}
\end{figure*}

For inference, ML models typically do not rely on individual input features, such as a single image pixel or a word in a sentence. Instead, they often rely on complex interactions between various input features to make predictions. This complexity necessitates the development of more abstract explanations to represent these intricate relationships. 
Such explanations would offer insights that extend beyond individual feature contributions, providing a more holistic view of how different feature interactions influence the model's predictions.
Attempts to provide such abstract explanations, which go beyond feature-wise explanations, have been proposed recently in \cite{achtibat23concept_rel,chormai24relsubspace,CIRAVEGNA2023103822}. The goal is to provide explanations that are more human-understandable and closely aligned with human intuition regarding problem solving strategies. The human-understandable explanatory features range from concepts that the model has learned \cite{achtibat23concept_rel} to activation patterns that are relevant for the prediction \cite{chormai24relsubspace}. 

Another approach to generating abstract explanations is to attribute relevance to the \emph{logical relationships} between features, rather than to the features themselves. This type of abstraction is intuitive and human-readable, as it mirrors the human tendency to reconsider the logical reasoning behind a prediction. For instance, when we read a sentence, we understand its meaning not just through individual words but also through the grammatical rules that govern their combination. To build trust in the model, an explanation method should be able to recover such grammatical rules, reflecting the underlying logical structure of the model’s prediction strategy.

There exist already different explanation methods that attribute the relevance of logical relationships include logical conjunctions $(\wedge)$ \cite{janizek2021IntHess,Kedar2020shapTaylor,eberle22bilrp, concept_interactions_1, concept_interactions_2} and logical disjunctions $(\vee)$ \cite{xiong22asubgraph}. Additionally, Some ML models are designed to encode these logical relationships, making them inherently more interpretable \cite{DIAZRODRIGUEZ2022XNESYL, CIRAVEGNA2023103822}. However, there is no post-hoc explanation framework that attributes relevance to logical formulas that are composed of a \emph{functionally complete} set of logical connectives, which express any logical relationship between features.

We propose an explanation framework that specifies relevance values for both conjunctive relationships ($\wedge$) between input features and the absence of features ($\neg$). This enables us to express all possible logical relationships between features. These logical formulas, which we refer to as \emph{queries}, are configurable and can be tailored to answer any question a user might have about the model's predictions. Furthermore, we provide a method for automatically generating queries that best describe the model’s prediction strategy.

Our framework relies on decomposing the model's prediction into the relevance values of different feature subsets. This decomposition can be defined either as a propagation-based \cite{bach-plos15, montavon-pr17, schnake22gnnlrp} or a perturbation-based \cite{lund17unified, Lund2020shapInterVal, blucher2024flipping} approach.
The relevance value of a logical formula (query) is computed by filtering the feature subsets where the query holds true and summing their corresponding relevance values. To identify the query that most accurately represents the model's prediction strategy, we search for the query, represented as a binary vector indicating the truth values for feature subsets, that most closely matches the model decomposition. Since the resulting explanatory features are composed of symbolic expressions, we refer to our novel framework as \emph{Symbolic XAI}.

\medskip
Figure \ref{fig:intro_exemplary_example} previews the insights that we can gain about a model’s decision-making strategies using Symbolic XAI. 
An in-depth analysis of these experiments is provided in Section \ref{sect:showcases}.

In Figure \ref{fig:intro_exemplary_example}\textbf{a)}, we analyze the sentence ``That was not bad'', which has a positive sentiment. Classical XAI, specifically first-order relevance values achieved using Layer-wise Relevance Propagation (LRP), provides high relevance values for the words ``not'' and ``bad''. In contrast, Symbolic XAI assigns high relevance values to the conjunctive relationship between ``not'' and ``bad''. However, when evaluating the queries ``not$\wedge \neg$bad'' and ``$\neg$not$\wedge$bad'', which measure the relevance of ``not'' and ``bad'' in the absence of their counterpart, both have low relevance values. This shows that the words ``not'' and ``bad'' marginally influence the prediction without their counterpart, which is a detail that classic XAI fails to reveal.

In Figure \ref{fig:intro_exemplary_example}\textbf{b)}, we see that Classical XAI identifies the mouth as having a high contribution to the prediction, while the eyes are considered to be less important. Symbolic XAI reveals that the interaction between the mouth and eyes are less important, whereas the mouth alone, in the absence of the eyes, has a higher contribution. This suggests that the model focuses primarily on the first-order concepts and disregards the interactions between them. Without using Symbolic XAI, this information remains elusive.

When considering the relevance values of single atoms in \ref{fig:intro_exemplary_example}\textbf{c)}, the oxygen atoms have the highest variation compared to all other atoms. Despite this variation, no clear trend is visible. In Symbolic XAI, we observe that the pairwise interactions between the oxygen and hydrogen atoms play a crucial role. This result aligns with chemical intuition, as hydrogen moves from one oxygen atom to the other during the reaction. This demonstrates the flexibility of constructing various logical formulas within the Symbolic XAI framework and obtaining their relevance for the prediction. Such a result would not be possible with a rigid explanation method that only considers a fixed set of explanatory features.
\medskip

In summary, the contributions of this paper are:
\begin{enumerate}
\item We introduce a novel Symbolic XAI framework, designed to compute the relevance values of logical formulas composed of a \emph{functionally complete} set of logical connectives for the model's prediction.
This framework is integrated into existing explanation approaches, specifically perturbation-based or propagation-based methods. 

\item We propose a query search strategy aimed at identifying the most expressive queries that accurately reflect the decision-making processes of the models.

\item We demonstrate the practical usefulness of our framework across three different application domains, providing diverse use cases, such as sentiment analysis, facial expression recognition (FER), and molecular energy prediction in quantum chemistry.
\end{enumerate}

\section{Background and Related Work}\label{sect:back_rel}
Throughout this manuscript, we assume that we are given an ML model $f$, which takes $\bm{X}= (x_I)_{I \in \mathcal{N}}$ as input to obtain the prediction $f(\bm{X})$. Here, $\mathcal{N}$ is the set of feature indices and the model $f$ maps the input $\bm{X}$ to a scalar value.
In the following, we outline several families of explanation methods.

\subsection{From first- to higher-order feature relevance}

\paragraph{First-order feature relevance}
First-order explanation methods consider the impact of individual features on the model's predictions \cite{ali2022xaiTrans, bach-plos15,montavon-pr17}.  For each feature index $I \in \mathcal{N}$, we obtain a relevance value $R_I$ to quantify the contribution of $x_I$ to the prediction $f(\bm{X})$. Popular first-order feature methods include Layer-wise Relevance Propagation (LRP) \cite{bach-plos15,montavon-pr17,montavon2018methods,samek2021xaireview}, Integrated Gradients (IG) \cite{sundar17axiomatic}, SHAP \cite{lund17unified}, and others \cite{zeiler2014occlusion,BLUCHER2022PredDiff,ribeiro2016trust, pope2019gcnxai}. 

In \cite{lund17unified}, many of these methods are classified as additive explanation methods, where the sum of feature contributions reconstruct the prediction, i.e., $\sum_I R_I = f(\bm{X})$. 

In our framework, it is also possible to obtain the relevance values of individual features by formulating them as a logical formula. In fact, we observe that some first-order explanation methods, such as Occlusion Sensitivity \cite{BLUCHER2022PredDiff, zeiler2014occlusion} and SHAP \cite{lund17unified}, are special cases of our methodology.

\paragraph{Second-order feature relevance}
Second-order explanation methods consider the impact of pairwise feature interactions on the model’s predictions. These feature interactions can either be represented as tuples of two features $(I, J) \in \mathcal{N} \times \mathcal{N}$, where $\times$ is the Cartesian product, or as sets of two features $\{I, J\} \subseteq \mathcal{N}$. The difference between these representations is that for tuples, the order of the features matters \cite{ying19gnnexplainer, faber21evalGNNXAI}, whereas for sets, it does not \cite{janizek2021IntHess, Lund2020shapInterVal, eberle22bilrp}.

In both cases, we obtain a value $R_{IJ}$, which specifies the relevance of the features $I$ and $I$ to the prediction $f(\bm{X})$. The interpretation of $R_{IJ}$ varies by applications. For instance, in BiLRP \cite{eberle22bilrp}, $R_{IJ}$ expressed by design the similarity between pixels in a similarity model $f$, and can be extended to other applications (e.g., \cite{keyl2022patient,eberle2023insightful}).
When explaining graph neural networks, $R_{IJ}$ describes the relevance of edges in the input graph \cite{ying19gnnexplainer, faber21evalGNNXAI}. In other contexts, the relevance value $R_{IJ}$ quantifies the joint contribution of the features $x_I$ and $x_J$ to the prediction $f(\bm{X})$ \cite{janizek2021IntHess, Lund2020shapInterVal}.

Some of these methods are considered additive \cite{eberle22bilrp, Lund2020shapInterVal, janizek2021IntHess}, meaning that the sum of the relevance values of feature pairs equals the model’s prediction. For sets of two features, this is expressed as $\sum_{I, J \in \mathcal{N}: I < J} R_{IJ} = f(\bm{X})$. 

In our framework, we can also attribute relevance values to sets of two features expressed by logical formulas, specifically using a logical conjunction $(\wedge)$ between the features.

\paragraph{Higher-order feature relevance}
Higher-order explanation methods aim to explain how interactions among multiple features collectively influence a model’s predictions, which generalizes the notion of second-order feature relevance. As before, this class of explanation methods can be subdivided into two types: \textit{sequences of features} \cite{schnake22gnnlrp} and \textit{feature sets} \cite{Kedar2020shapTaylor, garbis1999shapInter, FUJIMOTO2006AxiomInterIndex, BLUCHER2022PredDiff}. In the case of sequences of features, we consider the relevance score $R_\mathcal{W}$ for a fixed-length sequence $\mathcal{W} = (I, J, \dots) \in \mathcal{N} \times \mathcal{N} \times \dots$. For feature sets, we consider the relevance score $R_\mathcal{S}$ for each set of feature indices $ \{I, J, ...\} = \mathcal{S} \subseteq \mathcal{N}$. The distinction is that for sequences, the order of the features matters, whereas for sets, it does not.

Higher-order explanation methods find applications in various domains, such as understanding infection chains in disease prediction \cite{xiong22asubgraph} or analyzing how chains of atoms contribute to specific properties of molecules \cite{xiong22asubgraph, schnake22gnnlrp}.

Some of these methods are additive \cite{schnake22gnnlrp,Kedar2020shapTaylor,FUJIMOTO2006AxiomInterIndex}. For instance, in \cite{schnake22gnnlrp}, the sum over all feature sequences equals the model’s prediction, i.e., $\sum_{\mathcal{W} \in \mathcal{N} \times \mathcal{N} \times \dots} R_\mathcal{W} = f(\bm{X})$. Similarly, in \cite{Kedar2020shapTaylor}, the sum over all sets of a fixed size $k$ equals the model’s prediction, i.e., $\sum_{\mathcal{S} \subseteq \mathcal{N}: |S| = k} R_\mathcal{S} = f(\bm{X})$. 

Note that in the Symbolic XAI framework, we can also attribute relevance values to higher-order interactions within a set of features by composing a logical formula that describes a logical conjunction $(\wedge)$ between the features in the set.

\subsection{Explainable AI beyond first-, second-, and higher-order feature relevance}
A branch of research dating back to the 1990s, known as \textit{rule extraction algorithms} for deep neural networks \cite{thrun1994rule_extr_distr_rep, tsukimoto200rule_extract_NNs, hailesilassie2016rule_extract_xai_review}, aims to summarize the decisions made by neural networks using rule-based algorithms, offering potentially greater interpretability. For example, in \cite{tsukimoto200rule_extract_NNs}, the authors approximate different parts of a neural network using Boolean functions and then express the entire network as a logical formula. In Section \ref{sec:automatization}, we also present logical formulas that articulate the model’s prediction strategy. However, rather than approximating the components of the model, we propose a search algorithm that identifies logical formulas that are comparable to a specific decomposition of the model's prediction.

In reinforcement learning, a fundamental principle is that the model employs a strategy (policy) to determine an action \cite{sutton2018RLintro}. In \cite{Kersting2023XSymbAI}, these policies are represented as symbolic abstractions, which correspond to the contextual information on which decision are based. Since these abstractions are human readable, this approach increases the interpretability of these models. While the use of symbolic abstractions as explanatory features is similar to our framework, our approach applies to any ML model and do not consider reinforcement learning. 

Neural-symbolic learning is an interdisciplinary approach that combines neural networks and symbolic AI to address complex problems that require both symbolic reasoning and statistical learning \cite{garcez2015neural_symb,artur2002neural_symb}. In \cite{DIAZRODRIGUEZ2022XNESYL}, the authors attempt to explain these models for a better interpretation of their decisions. Unlike our framework, which is applicable to any neural networks, their method is restricted to specific neural-symbolic model architectures.

In \cite{CIRAVEGNA2023103822}, the authors propose a model architecture that is able to explain its predictions by first-order logical formulas. The visual format of their explanatory logical formulas is similar to the queries used in our Symbolic XAI approach. However, their explanation framework is specific to a particular model architecture, whereas ours is more versatile and not tied to any specific architecture.

In \cite{concept_interactions_3}, it has been shown that the output score of a deep neural network can be decomposed into the effects of several interactive concepts. Each interactive concept represents a conjunctive relationship $(\wedge)$ between a set of input features. The authors further extend the conjunctive interaction to include the disjunctive interaction $(\vee)$. The interactive concepts are then used in \cite{concept_interactions_1} and \cite{concept_interactions_2} to derive optimal baseline values of input features and to analyze the generalization power of deep neural networks, respectively. It is known that the logical connectives $\wedge$ and $\vee$ are not functionally complete \cite{andrews1986math_logic_type}. In contrast, our framework employs the connectives ${\wedge, \neg}$, achieving functional completeness.

The authors in \cite{achtibat23concept_rel} propose a propagation-based method to identify specific parts of a neural network that correspond to particular concepts, such as ‘eyes’ or ‘snout’ in an image of a dog. The method generates heatmaps for these model-specific concepts, offering a more interpretable way to understand the model’s decision-making process. This approach is similar to our work in two ways:
First, for computing the relevance of a concept, they mask particular neurons in the model during propagation that correspond to that concept. Similarly, we mask different neurons in the propagation-based approach to obtain a decomposition of the model’s prediction (see Section \ref{subsect:decompo_model}).
Second, they provide a novel abstraction of the explanatory features by considering the relevance of concepts, while we attribute relevance to logical formulas.

In \cite{chormai24relsubspace}, the authors describe a technique for extracting relevant subspaces that represent different factors influencing the model’s predictions. These subspaces are disentangled from each other, providing insights into the complexity of the prediction task. While their approach focuses on abstract explanatory features through relevant subspaces, our approach considers the relevance of the abstraction of logical formulas to understand the model’s predictions.

\section{Our Approach}\label{sect:symbxai_framework}
In this section, we introduce Symbolic XAI (Symb-XAI), an explanation framework that attributes relevance values to logical formulas, also referred to as queries, which express different logical relationships between input features.

Throughout the rest of this manuscript, we abbreviate the notations for sets, i.e., omit the braces $\{ \}$ for sets with only a few elements, such as $\{ I \}$ or $\{ I,J \}$, and instead write $I$ or $IJ$, respectively. We also recall the specification of the ML function $f$ and its input sample $\bm{X} = (x_I)_{I \in \mathcal{N}}$ from Section \ref{sect:back_rel}. Other notations used in this work are summarized in Table \ref{tab:notation} in the Appendix.

 \begin{figure*}[ht]
    \centering
    \includegraphics[width=.9\textwidth]{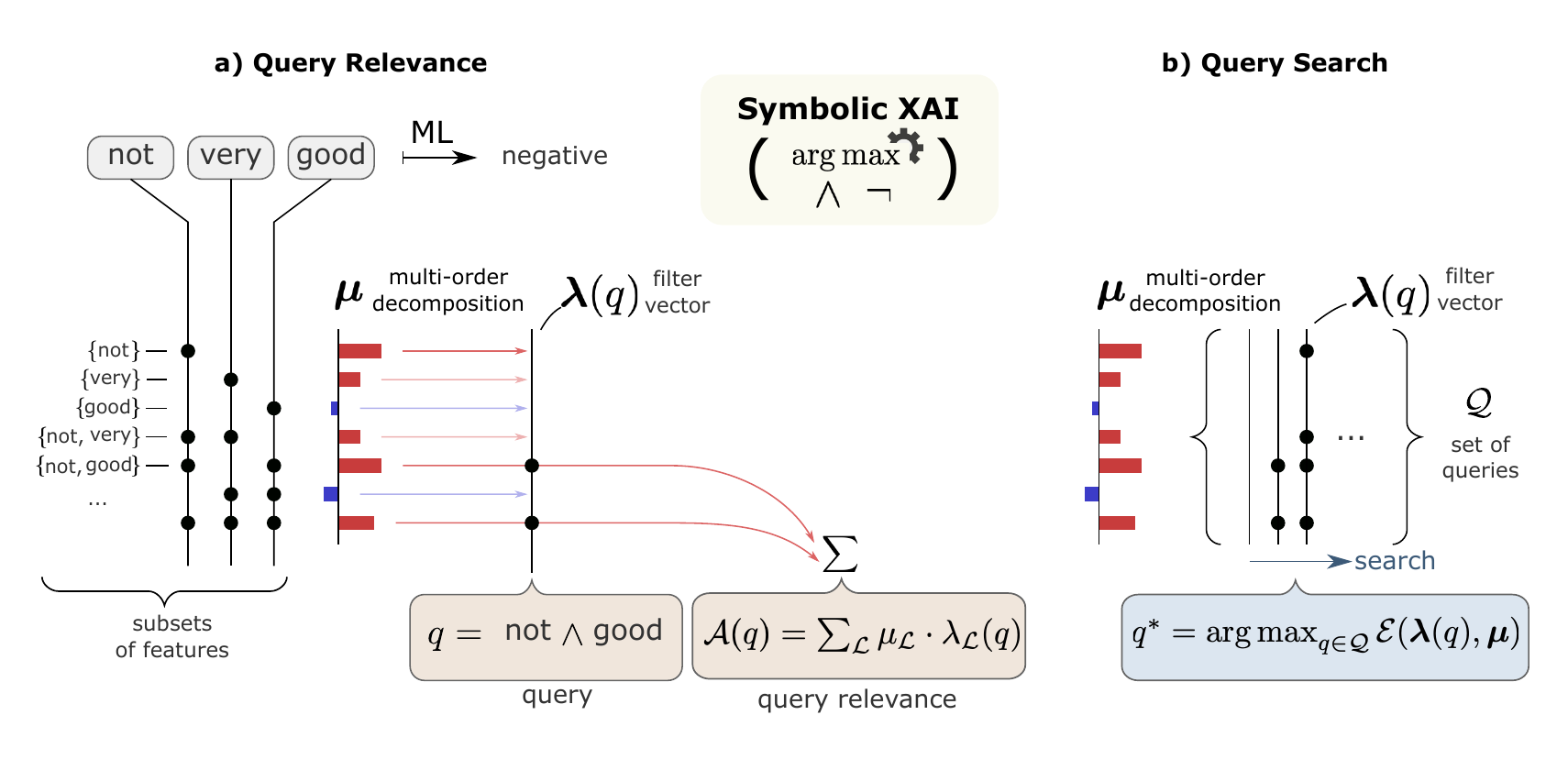}
    \caption{ Overview of the Symbolic XAI framework, illustrated to explain the prediction of a sentiment classification model.
    In \textbf{Query Relevance}, we describe how the relevance of a query to the model’s prediction is obtained. Below the input sentence ``not very good'', we explicitly visualize all subset of words, or subsets of features, and indicate the presence of each word in the set with $\bullet$. Two different objects depend on the subsets of features: The multi-order terms $\bm{\mu}$, and the filter vector  $\bm{\lambda}(q)$. The multi-order terms represent a decomposition of the model's prediction in terms of positive or negative scalar values for each subset of features, visualized here in a histogram. The filter vector is a vectorized representation of the query, or logical formula, indicating by $\bullet$ whether the query is true on the subset of feature at the same level in the image. For the query `not$\wedge$good’, the filter vector is true for subsets where both `not' and `good’ are present, specifically the sets $\{\text{not}, \text{good}\}$ and $\{\text{not}, \text{very}, \text{good}\}$. The query relevance $\mathcal{A}(q)$ is then obtained by summing the multi-order terms $\bm{\mu}$ for which the filter vector $\bm{\lambda}(q)$ is true.
    In \textbf{Query Search}, we describe how to search for queries that best express the model’s prediction. We consider a set of queries $\mathcal{Q}$, with their corresponding filter vectors $\bm{\lambda}(q)$. To find the suitable query $q^*$, we aim to maximize $\mathcal{E} (\bm{\lambda}(q), \bm{\mu})$ for any $q \in \mathcal{Q}$. $\mathcal{E}$ is generally a similarity measure between the filter vector $\bm{\lambda}(q)$ and the multi-order terms $\bm{\mu}$; in our case, it will be the correlation.
    }
    \label{fig:meth_descr}
\end{figure*}

Previous studies, as discussed in Section \ref{sect:back_rel}, explored computing the joint relevance of features for the prediction, which corresponds to the conjunctive relationship ($\wedge$) between features. Methodologically, there is a straightforward way to calculate such relevance. Assume a function $\mathcal{A}$ that measures the relevance values of sets of single features $I$ and $J$, as well as sets of feature pairs $ IJ $. In order to extend $\mathcal{A}$ to measure the relevance of the conjunction between two features, $I \wedge J$, we can utilize the  \emph{inclusion-exclusion principle} \cite{ryser1963combinatorial}:
\begin{align}\label{eq:inclu_exclu_princ}
        \mathcal{A}(I \wedge J) = \mathcal{A}(I) +  \mathcal{A}(J) -  \mathcal{A}(IJ)
\end{align}

Our goal is to extend the relevance function $\mathcal{A}$ to any logical formula that expresses logical relationships between features. Specifically, we aim to include the absence of features, described by the logical negation $(\neg)$. By combining this with the logical connective $\wedge$, we can attribute relevance to any logical formula, since $\wedge$ and $\neg$ are known to be functionally complete \cite{andrews1986math_logic_type}.

We base our framework on a decomposition of the model's prediction into detailed components (multi-order terms), a common approach in cooperative game theory \cite{FUJIMOTO2006AxiomInterIndex,garbish2000EquiRepSet}. The decomposition of a model's prediction and the different approaches for specifying it are discussed in Section \ref{subsect:decompo_model}. Moreover, we provide a general definition for attributing relevance values to logical queries, i.e., extending the relevance function $\mathcal{A}$ beyond Equation \eqref{eq:inclu_exclu_princ}, as outlined in Section \ref{sect:attribute_queries}. To gain a comprehensive understanding of which logical relationships are relevant for the model's prediction, we develop a search method to identify the queries that best describe the underlying prediction strategies. This is discussed in Section \ref{sec:automatization}.

\begin{figure*}
    \centering
    \includegraphics[width=.9\textwidth]{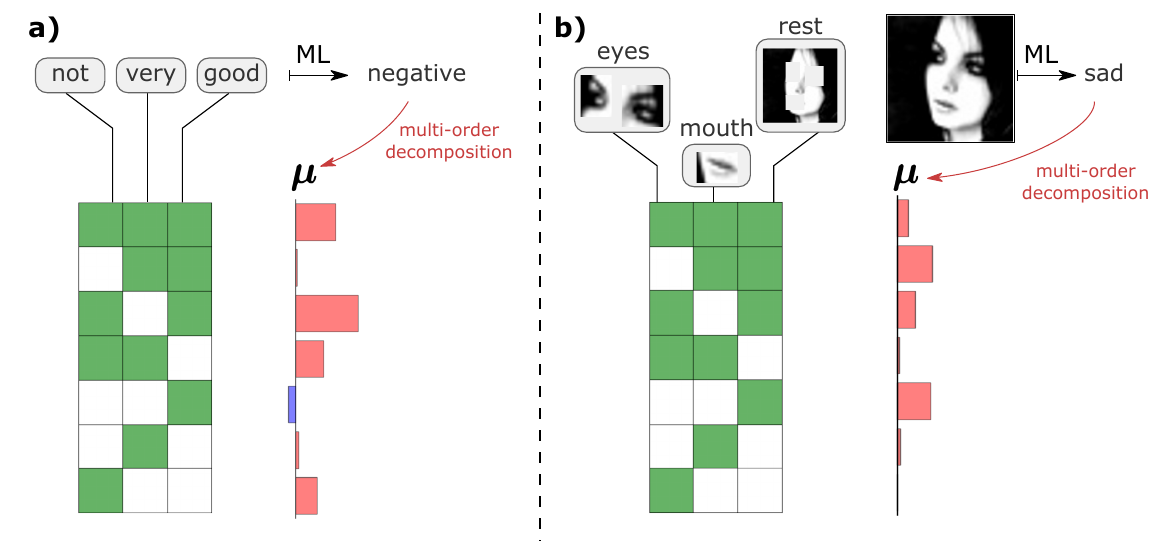}
    \caption{ The figure provides two exemplary multi-order decompositions of the model’s prediction; one for the sentiment classification task and one for the FER task, represented in subfigures \textbf{a)} and \textbf{b)}, respectively. Each subfigure provides the model’s input and its prediction. In both images, we use the definition of the multi-order decomposition $\bm{\mu}$ given by the propagation-based approach in Equation \eqref{eq:multi_order_gnnlrp}.
    In \textbf{b)}, the image is segmented into (super-) tokens of `eyes', `mouth’, and `rest', where `rest’ includes all pixels that are neither `eyes' nor `mouth’. For each example, we specify all possible subsets of the set of tokens in a table, where a green cell indicates the presence of a token in the set. For each subset, the multi-order terms are presented in a histogram.
    }
    \label{fig:symbxai_decompo}
\end{figure*}

\subsection{Multi-order model decomposition}\label{subsect:decompo_model}
As a first step to attribute relevance values to logical formulas, we aim to decompose the model's prediction $f(\bm{X})$ into components that uniquely depend on subsets of features  $\mathcal{L} \subseteq \mathcal{N}$. Such a decomposition is formally given by:
\begin{align}
    f(\bm{X}) = \sum_{\mathcal{L} \subseteq \mathcal{N}}  \mu_\mathcal{L}, \label{eq:pred_decompo_multi_order_subset}
\end{align}
 where the contribution term $\mu_\mathcal{L}$ expresses the contribution of the feature interactions in subset $\mathcal{L}$ to the prediction. It depends only on the elements in $\mathcal{L}$ and does not involve any variables outside of $\mathcal{L}$. We refer to this model decomposition in Equation \eqref{eq:pred_decompo_multi_order_subset} as \emph{multi-order decomposition}, and the terms $(\mu_\mathcal{L})_{\mathcal{L} \subseteq \mathcal{N}}$ as the \emph{multi-order terms}. We consider  $\bm{\mu} = (\mu_\mathcal{L})_{\mathcal{L} \subseteq \mathcal{N}} $ to be a vector depending on subsets $\mathcal{L}$. We also see a description of the multi-order terms  in Figure \ref{fig:meth_descr}\textbf{a)}.

 In principle, there is no straightforward way to define $\bm{\mu}$.  Therefore, we investigate two different ways to obtain such multi-order model decomposition, each from the perspective of popular explanation approaches.

\paragraph{Propagation-based approach to define $\bm{\mu}$}
LRP for graph neural networks (GNN-LRP) \cite{schnake22gnnlrp} is an explanation method designed to obtain higher-order explanations for graph neural networks (GNNs), but it is also applicable to other network architectures. This method considers walks $\mathcal{W} = (I,J, \dots)$, which are ordered sequences of input features $I, J, \ldots \in \mathcal{N}$, and assigns them a relevance score $R_\mathcal{W}$. For more details on how $R_\mathcal{W}$ is specified, particularly for architectures different from GNNs, such as Transformer models, we refer to \ref{app:ho_xai}. As discussed in Section \ref{sect:back_rel}, $R_\mathcal{W}$ expresses the relevance of higher-order interactions between the corresponding features in $\mathcal{W}$. We can use these terms to define the multi-order model decomposition.

\begin{approach}
We define the multi-order terms $\mu_\mathcal{L}$ within the propagation-based approach as the sum over all walks that encompass every feature in the subset $\mathcal{L}$. Formally, this is expressed as:
\begin{align}\label{eq:multi_order_gnnlrp}
\mu_\mathcal{L} = \sum_{\mathcal{W} : \, \text{set}(\mathcal{W}) = \mathcal{L}} R_\mathcal{W}
\end{align}
where set$(\mathcal{W})$ represents the set of all indices within the walk $\mathcal{W}$.
\end{approach}
This approach ensures that the multi-order terms effectively decompose the prediction $f(\bm{X})$, as described in Equation \eqref{eq:pred_decompo_multi_order_subset}. This property holds true because the sum of the relevance values $R_\mathcal{W}$ for all walks equals the model’s prediction $f(\bm{X})$ (see e.g. \cite{schnake22gnnlrp}).

\paragraph{Perturbation-based approach to define $\bm{\mu}$}
An alternative way to specify the multi-order terms in Equation \eqref{eq:pred_decompo_multi_order_subset} is to use a perturbation-based approach \cite{blucher2024flipping, zeiler2014occlusion, Lund2020shapInterVal, FUJIMOTO2006AxiomInterIndex, Kedar2020shapTaylor}. This approach estimates the model’s prediction by focusing on smaller areas of the input, perturbing the features outside the area of interest.

Formally, we consider the input sample $\bm{X}_\mathcal{S}$, which is equivalent to the original input $\bm{X}$ for features in the subset $\mathcal{S}$, i.e.,  $(\bm{X}_{\mathcal{S}})_{ I} = x_I$ for $I \in \mathcal{S}$. The rest of $\bm{X}_\mathcal{S}$ is perturbed, for example, by setting it to a constant value or using an inpainting method \cite{lund17unified, BLUCHER2022PredDiff}. The prediction $f(\bm{X}_\mathcal{S})$ of the original model $f$ on this modified input $\bm{X}_\mathcal{S}$ estimates the model’s prediction for the smaller input area. Additionally, if all information is perturbed, $\mathcal S = \{ \varnothing \}$, we assume that the model predicts zero, i.e.,  $f(\bm{X}_\varnothing) = 0$. We can use this approach to specify the multi-order decomposition of the model's prediction.

\begin{approach}
We define the multi-order terms $\mu_\mathcal{L}$ within perturbation-based approach as the \emph{Harsanyi dividend} of $f(\bm{X}_\mathcal{L})$, given by:
\begin{align}\label{eq:multi_ord_def_moebius}
\mu_\mathcal{L} = \sum_{\mathcal{S} \subseteq \mathcal{L}} (-1)^{| \mathcal{L}| - |\mathcal{S} |} \, f(\bm{X}_\mathcal{S})
\end{align}
The Harsanyi dividends are common in cooperative game theory \cite{harsanyi1963dividedCoprGame}.
\end{approach}

The Harsanyi dividend specifically measures the part of the contribution that arises from the \textit{dependencies} among all features in $\mathcal{L}$. Denoting the input features by $I, J, …$, the first two orders of the Harsanyi dividends are given as:
\begin{align*}
\mu_{I}  &= f(\bm{X}_I)\\
\mu_{IJ}  &= f(\bm{X}_{IJ})  - f(\bm{X}_I) -  f(\bm{X}_J)
\end{align*}

It is well known that Harsanyi dividends ensure that the multi-order terms effectively decompose the prediction, as desired in Equation \eqref{eq:pred_decompo_multi_order_subset} (see e.g. \cite{FUJIMOTO2006AxiomInterIndex}).

\begin{figure*}[ht]
    \centering
    \includegraphics[width=.9\textwidth]{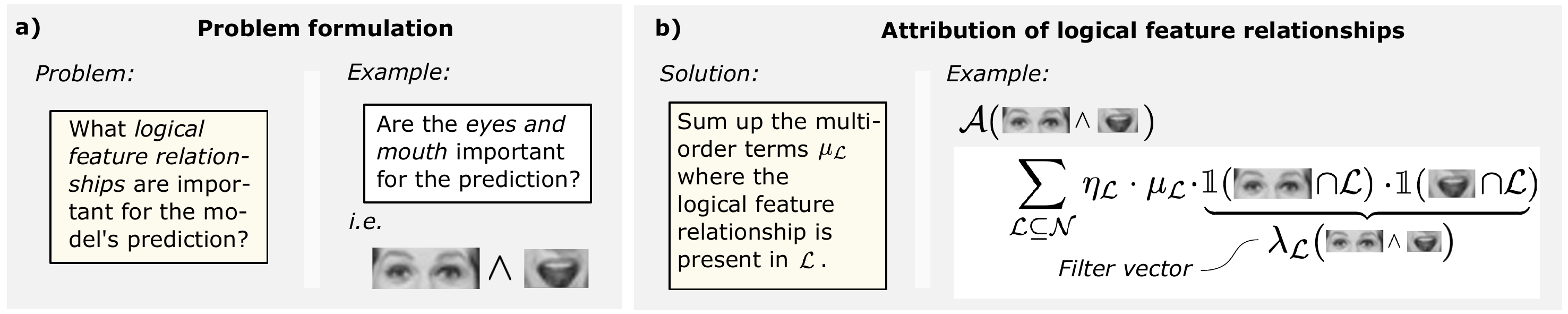}
    \caption{Description of how queries are composed and how to calculate their relevance. \textbf{a)} shows the problem formulation and an example for the importance of abstract dependencies, aka queries. In \textbf{b)}, the solution to the problem is specified with an example showing how to calculate such the relevance of a query.}
    \label{fig:symbxai_prob_attri}
\end{figure*}

\medskip

In fact, the two approaches to specify  $\bm{\mu}$ in Equations \eqref{eq:multi_order_gnnlrp} and \eqref{eq:multi_ord_def_moebius} are related. In the context of higher-order propagation-based explanation methods, the common way to specify the relevance of a subset of features $\mathcal{S}$ is by summing over all walks that can be composed with features in $\mathcal{S}$, i.e., $R_\mathcal{S} = \sum_{\mathcal{W} \in \mathcal{S} \times \mathcal{S} \times \dots} R_\mathcal{W}$. This concept was proposed in \cite{schnake22gnnlrp} and further developed with computational acceleration in \cite{xiong22asubgraph}.

We can now use the subgraph relevance to connect the two approaches. If we replace $f(\bm{X}_\mathcal{S})$ with $R_\mathcal{S}$ in Equation \eqref{eq:multi_ord_def_moebius} and then compute the Harsanyi dividends of $R_\mathcal{S}$, we obtain exactly the multi-order terms in Equation \eqref{eq:multi_order_gnnlrp}. This implies that it is methodologically consistent to use the relevance of subsets $R_\mathcal{S}$ to estimate the model’s prediction on a smaller area of the input, as used in the context of perturbation-based method. We provide a proof of this finding in \ref{app:express_multi_with_walks}.

\medskip
\paragraph{Showcasing multi-order terms on synthetic data}
Figure \ref{fig:symbxai_decompo} shows multi-order terms generated from the predictions of sentiment classification and FER models. For the sentence ``not very good'', Figure \ref{fig:symbxai_decompo}\textbf{a)} reveals relatively high values for the multi-order terms $\mu_{\text{not very good}}$ (first row) and $\mu_{\text{not good}}$ (third row), while the other terms play weaker roles. This indicates that the model finds the combination of ``not'' and ``good'' to be significant for the prediction, highlighting the negation of a positive word.

In the FER experiment shown in Figure \ref{fig:symbxai_decompo}\textbf{b)}, the multi-order terms $\mu_{\text{rest}}$ (fifth row) and $\mu_{\text{mouth rest}}$ (second row) have the highest values. This suggests that the `rest' segment, which includes all pixels except those corresponding to the `eyes' and `mouth', significantly influences the model's prediction, particularly when combined with the `mouth' segment. This shows that the incorporation of the `rest' segment is crucial when prediction the emotion of this image. 

\medskip

A challenge when working with the multi-order terms is that their number increases exponentially with the number of input features. This makes it particularly difficult to draw conclusions about the prediction strategies solely by examining the values of these terms. This challenge motivates us to summarize the multi-order terms into human-understandable explanations: the relevance values of logical formulas. In the following sections, we will elucidate how to specify these logical formulas and calculate their relevance values.

\subsection{Calculating relevance of logical formulas} \label{sect:attribute_queries}

We aim to compute the relevance of any logical relationships between input features using the multi-order decomposition in Equation \eqref{eq:pred_decompo_multi_order_subset}. A visual description of the problem formulation and how to attribute relevance to a query, can be seen in Figure \ref{fig:symbxai_prob_attri}.

To understand the advantage of using the multi-order terms from Section \ref{subsect:decompo_model} for relevance attribution of logical formulas, we motivate it with the conjunctive relationship between features, as illustrated in Equation \eqref{eq:inclu_exclu_princ}. Assuming the relevance $\mathcal{A}$ of feature sets $\mathcal{L}$ is measured using Occlusion Sensitivity \cite{zeiler2014occlusion, BLUCHER2022PredDiff}, i.e., $\mathcal{A}(\mathcal{L}) = f(\bm{X}) - f(\bm{X}_{\mathcal{N} \setminus \mathcal{L}})$, we obtain:
\begin{align}\label{eq:wedge_harsanyi_def}
    \mathcal{A}(I \wedge J) = \sum_{\mathcal{L} \subseteq \mathcal{N}} \mu_\mathcal{L} \cdot \mathbbm{1}( I \in \mathcal{L} \wedge J \in \mathcal{L} )
\end{align}
where $\mu_\mathcal{L}$ represents the multi-order terms from the   pertur-\\bation-based approach. This is a well-established result and can be verified, for instance, in \cite{garbish2000EquiRepSet}, Section 2. Our goal is to extend the use of multi-order terms $\bm{\mu}$ to define relevance $\mathcal{A}$ for more general logical formulas. Additionally, we incorporate a weight for each term $\mu_\mathcal{L}$ to enforce desired properties in the relevance function $\mathcal{A}$.

\paragraph{Definition of queries}
We explicitly define what a logical formula, called \emph{query}, is composed of in our explanation framework.
\begin{definition}
    A query $q$ is a sequence of symbols, where each symbol is either referencing to a set of feature indices $\mathcal{S}$, or a logical connective of conjunction $\mathcal{\wedge}$ or negation $\neg$. An example is $q' = ( \mathcal{S}_1, \wedge, \neg, \mathcal{S}_2)$, where $\mathcal{S}_1$ and $\mathcal{S}_2$ are referencing to feature sets. We denote $\mathcal{Q}$ as the set of all well-formed queries, meaning $q \in \mathcal{Q}$ is a readable combination of symbols, referring each to logical connectives or feature sets. For a query $q$ we abbreviate the notation by omitting the parentheses and commas. In the case of the exemplary $q'$, we  write $q'= \mathcal{S}_1 \wedge \neg \mathcal{S}_2$. 
\end{definition}

\paragraph{Definition of query relevance} We define the relevance attribution of queries using the multi-order terms in Equation \eqref{eq:pred_decompo_multi_order_subset}.
\begin{definition}
    The relevance of a query $q \in \mathcal{Q}$ is given by:
\begin{align}\label{eq:query_attribution_def}
    \mathcal{A}(\bm{\eta}, \bm{\mu}, q) = \sum_{\mathcal{L} \subseteq \mathcal{N}} \eta_\mathcal{L} \cdot \mu_\mathcal{L} \cdot \lambda_\mathcal{L}(q)
\end{align}
Here, $\bm{\eta} = (\eta_\mathcal{L})_{\mathcal{L} \subseteq \mathcal{N}}$ is a given \emph{weight vector}, which introduces customized weights of the relevance attribution, and $\bm{\lambda}(q) = (\lambda_\mathcal{L}(q))_{\mathcal{L} \subseteq \mathcal{N}}$ is a \emph{filter vector}, corresponding to the query $q$, that maps each subset $\mathcal{L} \subseteq \mathcal{N}$ to a Boolean value.
\end{definition}
To avoid heavy notation overhead, we do not always specify $\bm{\mu}$ and $\bm{\eta}$ in the arguments of $\mathcal{A}$ when it is clear from the context which weight vector and multi-order terms are being used. We provide a visual description of the relevance attribution of logical formulas in Figure \ref{fig:meth_descr}\textbf{a)}.

In the following Sections \ref{subsec:spec_query} and \ref{sect:eta_chapt}, we provide more details on the filter and weight vectors, respectively.

\subsubsection{Specifying the filter vector of a query}\label{subsec:spec_query}
We provide a more detailed view on how to obtain a filter vector $\bm{\lambda}(q)$ from a query $q$. An visual description how to specify the filter vector is shown in Figure~\ref{fig:symbxai_query_specify}. For an example of the filter vector, see Figure \ref{fig:meth_descr}\textbf{a)}. 

To reduce ambiguity when specifying symbols of subsets $\mathcal{S}$ in a query, we explicitly add in this subsection the prefix `$q: $' in front of the query in the argument of the filter vector, such as $\bm{\lambda}(q: \mathcal{S})$ for the filter vector of the query $q = \mathcal{S}$.

\paragraph{The query of feature presence}
We aim to estimate the relevance of the \emph{presence} of features $\mathcal{S}$, specified by the query $q = \mathcal{S}$. The filter vector $\bm{\lambda}(q: \mathcal{S})$ is defined by:
\begin{align*}
    \lambda_\mathcal{L}(q: \mathcal{S}) = \mathbbm{1}(\mathcal{S} \cap \mathcal{L}) \hspace{1cm} \text{for } \mathcal{L} \subseteq \mathcal{N}
\end{align*}
This means that for the relevance of $q=\mathcal{S}$ using Equation \eqref{eq:query_attribution_def}, we keep all $\mu_\mathcal{L}$ where $\mathcal{S}$ and $\mathcal{L}$ intersect. 
We note that the computation of first-order SHapley Additive exPlanations (SHAP) \cite{lund17unified} or Occlusion Sensitivity \cite{zeiler2014occlusion, BLUCHER2022PredDiff} is equivalent to query relevance of single feature presence, i.e., $q = I$, with different weighting vectors $\bm{\eta}$. More details on the weighting function are given in Section \ref{sect:eta_chapt}.

\paragraph{The query of feature absence}
Conversely, we aim to estimate the relevance of the \emph{absence} of features $\mathcal{S}$, specified by the query $q = \neg \mathcal{S}$. The filter vector $\bm{\lambda}(q: \neg \mathcal{S})$ is defined by:
\begin{align*}
    \lambda_\mathcal{L}(q: \neg \mathcal{S}) = \mathbbm{1}(\mathcal{S}  \subseteq \overline{\mathcal{L}}) \hspace{1cm} \text{for } \mathcal{L} \subseteq \mathcal{N}
\end{align*}
where $\overline{\mathcal{L}} = \mathcal{N} \setminus \mathcal{L}$ is the complement set of $\mathcal{L}$. When computing the relevance of the query $q = \neg \mathcal{S}$ in Equation \eqref{eq:query_attribution_def}, we only sum over those $\mu_\mathcal{L}$ where $\mathcal{S}$ and $\mathcal{L}$ are disjoint.

To the best of our knowledge, the relevance of the absence of features has not been proposed before. When using the perturbation-based approach for specifying $\bm{\mu}$, as in Equation \eqref{eq:multi_ord_def_moebius}, the attribution of the absence of the feature $I$, i.e., $q = \neg I$, with a weight vector $\bm{\eta} = 1$ in Equation \eqref{eq:query_attribution_def}, yields:
\begin{align}\label{eq:neg_occlusion_rel}
     \mathcal{A}(q: \neg I) = f(\bm{X}_{\mathcal{N} \setminus I})
\end{align}
which is the model's prediction when the feature $I$ is absent.

\begin{figure*}[ht]
    \centering
    \includegraphics[width=.9\textwidth]{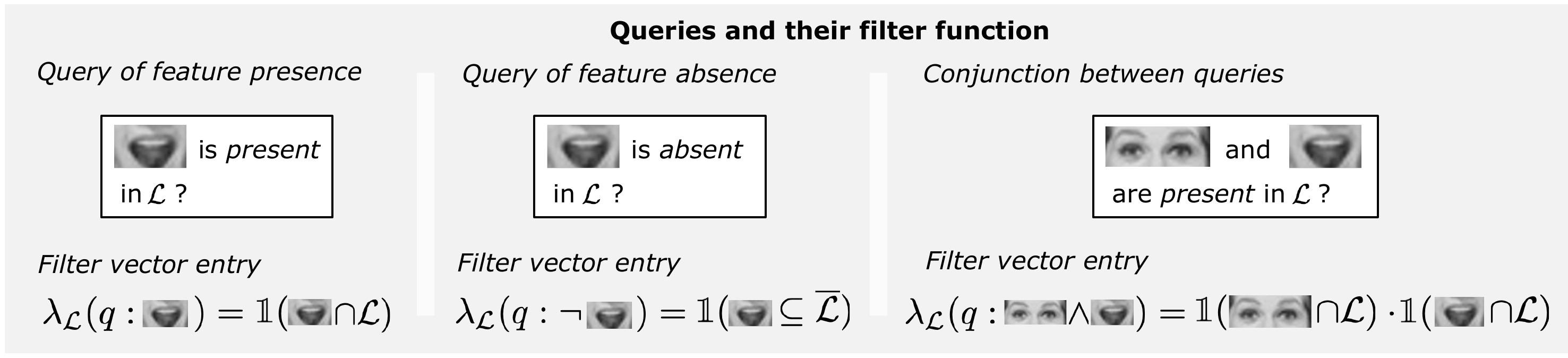}
    \caption{Description of the fundamental rules for the query composition are given alongside the corresponding filter vector. We consider the \emph{presence} and \emph{absence} of feature sets, and the logical conjunction between two queries.}
    \label{fig:symbxai_query_specify}
\end{figure*}

\paragraph{Logical conjunction between queries}
To build more complex queries, we aim to estimate the relevance of the conjunction ($\wedge$) between queries $q_1, q_2 \in \mathcal{Q}$, i.e., $q= q_1 \wedge q_2$. The filter vector $\bm{\lambda}(q: q_1 \wedge q_2)$ is  defined by:
\begin{align*}
    \lambda_\mathcal{L}(q:  q_1 \wedge q_2) = \lambda_\mathcal{L}(q: q_1) \cdot \lambda_\mathcal{L}(q: q_2) \hspace{1cm} \text{for } \mathcal{L} \subseteq \mathcal{N}
\end{align*}
This means $\lambda_\mathcal{L}(q: q_1 \wedge q_2)$ is true if and only if both $\lambda_\mathcal{L}(q: q_1)$ and $\lambda_\mathcal{L}(q: q_2)$ are true. This rule is applied recursively until we reach the cases that $q_1$ and $q_2$ either express the query of feature presence or absence, as described above.

\medskip
We specified the filter function for queries which are composed of feature subsets, the logical connectives $\neg$ and $\wedge$, which enables us to obtain the query relevance for any query $q \in \mathcal{Q}$.

\subsubsection{Specifying weight vectors} \label{sect:eta_chapt}

The choice of the weight vector $\bm{\eta}$ is custom and depends on the user's needs. We aim to offer a straightforward strategy for defining $\bm{\eta}$ with the consideration of desirable properties of the explanation. 

\paragraph{Occlusion values of queries}

One way to specify the weight vector for queries is by setting it to unity, i.e., 
\begin{align*}
\eta_\mathcal{L} = 1 \hspace{1cm} \text{for } \mathcal{L} \subseteq \mathcal{N}
\end{align*}
For the first-order query $q = I$, such a weight vector leads to the so-called Occlusion Sensitivity \cite{zeiler2014occlusion} or PredDiff \cite{BLUCHER2022PredDiff} method. We therefore call it the \textit{occlusion value} of queries. 

This weighting allows for calculating the relevance of some queries without explicitly computing the multi-order terms in Equation \eqref{eq:pred_decompo_multi_order_subset}, which can be computationally expensive. This result is known for the relevance of feature presence, as it is equivalent to Occlusion Sensitivity \cite{zeiler2014occlusion}. We also observed a similar result for logical conjunction in Equation \eqref{eq:inclu_exclu_princ} and feature absence in Equation \eqref{eq:neg_occlusion_rel}. Thus, occlusion values of queries are often computationally efficient.

\paragraph{Shapley values of queries}

In classical explanation tasks, it is desirable for the explanation method to exhibit the \textit{conservation property}. In the context of single-feature explanation methods such as SHAP \cite{lund17unified}, which we can obtain by attributing the queries $q = I$, this means that the sum of all feature attributions equals the model's prediction, i.e.,  $\sum_I \mathcal{A}(I) = f(\bm{X})$. 

We now show that there is a weighting function $\bm{\eta}$ which enforces the conservation property for any set of queries $(q_k)_{k=1}^M$:
\begin{lemma} \label{lemm:conserv_query}
Let $(q_k)_{k=1}^M$ be a set of queries with, such that for every subset $\mathcal{L} \subseteq \mathcal{N}$ we have $\sum_{k=1}^M  \lambda_\mathcal{L}({q_k}) >0$. 
If we set the weight vector to be 
\begin{align*}
    \eta_\mathcal{L} = \left(\sum_{k=1}^M  \lambda_\mathcal{L}({q_k})\right)^{-1},
\end{align*}
we get: $\sum_{k=1}^M \mathcal{A}(\bm{\eta}, \bm{\mu}, q_k) = f(\bm{X})$.
\end{lemma}

\noindent The proof of this lemma can be found in \ref{app:proof_lemma_shap_weight}.
\medskip 

Note that when applying Lemma \ref{lemm:conserv_query} to the query set $(q_I)_{I \in \mathcal{N}}$ of feature presence $q_I = I$, we obtain the weight vector:  
\begin{align*}
    \eta_\mathcal{L} = \frac{1}{|\mathcal{L}|}.
\end{align*}
It is known\cite{FUJIMOTO2006AxiomInterIndex} that the relevance $\mathcal{A}(q_I)$ with this weight vector coincides with the Shapley value of the feature $I$. This motivates us to call the relevance $\mathcal{A}(q_k)$, with $\bm{\eta}$ given by Lemma \ref{lemm:conserv_query}, the \textit{Shapley value of query $q_k$} within the set of queries $(q_k)_{k=1}^M$.

\subsection{Automatic specification of expressive queries}\label{sec:automatization}
In practice, the query specifications are not always given by the user. The options for queries in $\mathcal{Q}$ are large and to find a query that properly expresses the model's prediction is challenging. 

In this subsection, we aim to describe a search algorithm that expresses and summaries best the prediction strategy of the model. We provide a visual description the query search  in Figure \ref{fig:meth_descr}\textbf{b)}.

\paragraph{Problem formulation}

Our framework is based on the multi-order decomposition of the prediction $f(\bm{X})$ into the terms $\mu_\mathcal{L}$, as given in Equation \eqref{eq:pred_decompo_multi_order_subset}. To summarize the prediction strategy of $f(\bm{X})$, we aim to find the query $q^*$ for which the filter vector $\bm{\lambda}(q^*)$ is an adequate description of the multi-order terms $\bm{\mu}$. 
This suggests that we need to identify a measure of \textit{similarity} between $\bm{\lambda}(q^*)$ and $\bm{\mu}$. In our case we consider the correlation between $\bm{\mu}$ and $q^*$, i.e., 
\begin{align}\label{eq:opt_automat_general}
    q^* = \argmax_{q \in \mathcal{Q}} \text{corr}_{\bm{\eta}}(\bm{\lambda}(q), \bm{\mu})
\end{align}
where $\mathcal{Q}$ is the set of all queries, as specified in Section \ref{subsec:spec_query}, and $\text{corr}_{\bm{\eta}}$ is the weighted correlation by the weighting vector $\bm{\eta}$. The explicit definition of $\text{corr}_{\bm{\eta}}$ is given in   \ref{app:corr_def}. Note that, due to the computational complexity, we compute in practice the multi-order terms $\mu_\mathcal{L}$ only for subsets $\mathcal{L}$ that contain at most 4 features.

\section{Showcasing the Application of SymbXAI in Different Application Domains}\label{sect:showcases}

In the following sections, we demonstrate the use of our SymbXAI framework for selected  data and models from the domains of natural language processing, computer vision, and quantum chemistry. 

\subsection{Usage in natural language processing} \label{sect:showcase_nlp}
 In the field of natural language processing (NLP), ML models play a crucial role in learning from text data. These models are employed across a diverse range of applications, such as information extraction \cite{Jurafsky2009InfoExtract}, text generation \cite{IQBAL2022TextGenSurvey}, and sentiment analysis \cite{socher2013recursive, gxai_survey2023yuan}.

In the following sections, we demonstrate the usefulness of our framework in the NLP domain, particularly for gaining insights into the predictions of sentiment analysis models. The objective of sentiment analysis is to determine the sentiment expressed in a given sentence or paragraph, such as a movie review conveying an opinion about a film. In Section \ref{sec:rel_vs_ground_truth_nlp}, we compare our framework’s performance with ground-truth annotations. Additionally, in Section \ref{sect:input_flipping_descr}, we illustrate how the flexibility of modeling queries within our framework significantly impacts common evaluation metrics for explanation methods. Finally, in Section \ref{sect:automat_nlp_cont_conj}, we show how to automatically identify expressive queries that best represent the underlying prediction strategies when predicting sentiment from sentences with common grammatical constructions.

 \begin{figure*}
     \centering
     \includegraphics[width=0.8\textwidth]{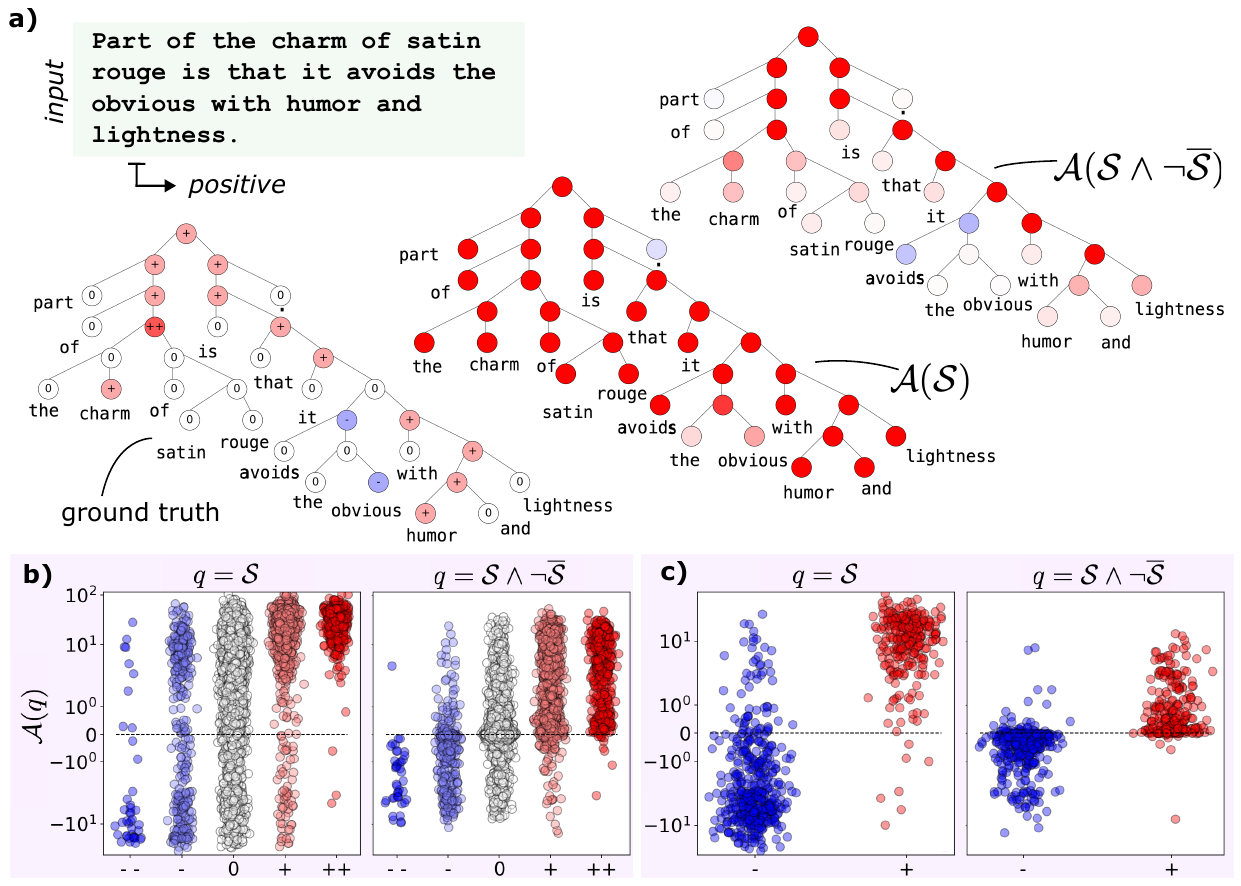}
     \caption{Evaluation of the Symbolic XAI framework on the SST and Movie Reviews datasets. In subfigure \textbf{a)} we see the relevance and ground-truth human annotations of one example with positive sentiment. The nodes in the tree specify the subsentence $\mathcal{S}$ of all leaf-nodes. The color of the nodes specify (from left to right) ground-truth relevance, $\mathcal{A}(\mathcal{S})$ and $\mathcal{A}(\mathcal{S} \wedge \neg \overline{\mathcal{S}})$. In \textbf{b)} and \textbf{c)} we see the relevance of different queries that incorporate different subsentence for 200 samples in each subfigure, for the the SST and Movie Reviews datasets, respectively. The relevance values are organized by their ground-truth values. }
     \label{fig:nlp_sst_imdb_summary}
 \end{figure*}

\paragraph{Datasets}

The SST dataset \cite{socher2013recursive} comprises 11,855 sentences extracted from movie reviews. Each sentence is subdivided into human-readable subsentences, which are categorized into one of five labels: \textit{very negative}, \textit{negative}, \textit{neutral}, \textit{positive}, or \textit{very positive}. This dataset is referred to as SST-5. The annotations of the subsentences were generated by reading each subsentence individually and indicating their associated sentiments.

Another version of this dataset, referred to as SST-2, consists of a subset of sentences from SST-5 with compressed label information. Sentences with a positive or very positive label in SST-5 are labeled \textit{positive} in SST-2. Similarly, sentences with a negative or very negative label in SST-5 are labeled \textit{negative} in SST-2. Neutral sentences are omitted, and the fine-grained sentiment information of subsentences is also omitted in SST-2.

The Movie Reviews dataset \cite{movie_reviews1, movie_reviews2} contains movie reviews with binary sentiment labels: positive or negative. Each review is accompanied by rationale annotations, which are specific parts of the review that support the sentiment label. These annotations are based on the full-text analysis of the reviews, depending on the context of the whole text.

It is essential to highlight the difference in annotation methods between the SST-5 and Movie Reviews datasets. The key distinction lies in the presence or absence of contextual information in these annotations. The annotations for the Movie Reviews dataset are given with the awareness of the whole paragraph, i.e.,  the context of the subsentences. In contrast, the annotations in the SST-5 dataset lack this context information, meaning they are self-contained and autonomous, not relying on the rest of the sentence. We demonstrate how our framework can capture this contextual aspect.

\medskip

\paragraph{Experimental setup} We will use our proposed SymbXAI framework in the NLP domain to interpret  predictions of a popular Transformer model, namely BERT \cite{devlin2019bert}, which was trained for the sentiment analysis task. Further details about the model can be found in  \ref{app:nlp_model_details}.

Our evaluation encompasses a thorough examination and interpretation of the model's predictions on two sentiment analysis datasets: Movie Reviews \cite{movie_reviews1, movie_reviews2} and  Stanford Sentiment Treebank (SST) \cite{socher2013recursive}.

\subsubsection{Assessing the relevance of queries against ground-truth values}\label{sec:rel_vs_ground_truth_nlp}
We evaluate the quality of the relevance values produced by our framework by comparing them with human annotations.

In our experiments, we employ the BERT model, which is pre-trained on the SST-2 dataset, and evaluate its performance on SST-5 to analyze the model's understanding in a more complex setting. We utilize the relevance $\mathcal{A}(q)$ for the queries $q= \mathcal{S}$ and $q= \mathcal{S} \wedge \neg \overline{\mathcal{S}}$, which consider the presence of the subsentence $\mathcal{S}$ with and without surrounding features, respectively.

The qualitative outcome of the query relevance on a sentence from SST-5 is visualized in Figure \ref{fig:nlp_sst_imdb_summary}\textbf{a)}, where positive, negative, and neutral sentiments are represented by the colors red, blue, and white, respectively. It is worth noting that since the model has been trained on only two labels, positive and negative, it lacks familiarity with the nuances of neutral sentiment. The results obtained from the query contribution $\mathcal{A}(\mathcal{S} \wedge \neg \overline{\mathcal{S}})$ demonstrate a stronger alignment with the true labels compared to those derived from the query contribution $\mathcal{A}(\mathcal{S})$. This phenomenon can be attributed to the fact that, during the process of generating ground-truth annotations for the SST-5 dataset, annotators were constrained to evaluate individual phrases in isolation, lacking the information to consider the complete context.

We illustrate a comparable quantitative observation in Figure \ref{fig:nlp_sst_imdb_summary}\textbf{b)}. In this experiment, we evaluate the model's competence in previewing the sentiment of each subsentence from the SST-5 dataset, spanning from very negative to very positive. Alongside its mostly accurate identification of positive and negative sentiments, our model demonstrates an understanding of neutral sentiment, even though it was not explicitly trained for it. Furthermore, when the context is disregarded, i.e., $\mathcal{A}(\mathcal{S}\wedge \neg \overline{\mathcal{S}})$, a more evident correlation emerges between the true labels and the signs of the relevance scores, quantitatively underscoring that this query relevance reflects the method of assigning human annotations to subsentences.

Figure \ref{fig:nlp_sst_imdb_summary}\textbf{c)} provides a quantitative representation of similar findings for the Movie Reviews dataset. We investigate the model's ability to understand the positivity or negativity of the movie rationales. In this figure, we show the relevance distributions of the positive and negative rationales in different contextual settings for the queries. As shown, the model is capable of accurately assigning positive and negative relevance scores to positive and negative rationales, respectively, especially when the full context is taken into account. This underscores that the contribution measure $\mathcal{A}(\mathcal{S})$, which incorporates the full context, aligns more closely with human annotations. Further results on the Movie Reviews dataset can be found in \ref{app:add_nlp_eval}.

\subsubsection{Evaluating SymbXAI via input flipping}\label{sect:input_flipping_descr}

We demonstrate the performance of our framework using the well-established input flipping strategy \cite{samek2016evaluating, blucher2024flipping}. This metric evaluates explanation methods by testing their ability to estimate the effect on the model's output when individual features are removed from or added to the input.

When removing features from the input, we track the \textit{removal curve}, represented by $ RC_j = f(\bm{X}_{\mathcal{N} \setminus \{I_1, \dots, I_j\}}) $ for some sequence $(I_1, \dots, I_n)$ of input features $ I_j \in \mathcal{N} $. When adding features to the input, we track the \textit{generation curve}, represented by $ GC_j = f(\bm{X}_{\{I_1, \dots, I_j\}}) $. Here, $ f(\bm{X}_\mathcal{S}) $ represents the model's output for input $ x $, where features not in $ \mathcal{S} $ are removed or inpainted, as described in Section \ref{subsect:decompo_model}. We consider the areas under the removal or generation curves as AURC and AUGC, respectively.

We measure the performance of an explanation method by its ability to select a sequence of input features $(I_1, \dots, I_n)$ such that the AURC or AUGC is maximized or minimized, resulting in four different evaluation tasks.

\paragraph{Selecting a feature sequence using SymbXAI}

In our framework, we select the relevance of the queries $ q = \neg \mathcal{S} $ and $ q = \mathcal{S} \wedge \neg \overline{\mathcal{S}} $, with $ \bm{\eta} = 1 $, to be predictors of the model's output when the features $ \mathcal{S} $ are removed from or added to the input, respectively. Inspired by the search algorithm known as \textit{local best guess} \cite{schnake22gnnlrp,xiong22asubgraph}, we build the sequence $(I_1, \dots, I_n)$ by iterating over each feature in $ \mathcal{N} $ and selecting the feature for which the attribution of the corresponding query is maximized or minimized, depending on the evaluation task.

For example, when we aim to minimize the removal curve, we select the sequence $(I_1, \dots, I_n)$ by the following iteration scheme:
\begin{align*}
    I_{j+1} = \argmin_{I' \in \mathcal{N}} \mathcal{A}( \neg \{ I_1, \dots, I_{j}, I' \}) && \text{for } j = 1, \ldots, n-1
\end{align*}
For the other evaluation tasks, we adjust accordingly by replacing the target query or using $\argmax$ instead of $\argmin$.

\paragraph{Flipping input features in NLP}
We compare the performance of various explanation methods on the task of flipping input features.
In all experiments, we use the occlusion weight vector $\bm{\eta} = 1$. 
We compare the faithfulness of the explanations generated by our framework against those obtained by existing first-order explanation methods, such as \textit{PredDiff} \cite{BLUCHER2022PredDiff, zeiler2014occlusion}, LRP \cite{bach-plos15, montavon-pr17} and a random baseline. For these first-order methods, we select the order of the flipping in a standard way \cite{blucher2024flipping}, starting form the most or least relevant feature, depending on the perturbation task.

% Adjusting table cell padding
\setlength{\tabcolsep}{4pt} % Default is 6pt
\renewcommand{\arraystretch}{1.1} % Default is 1
\begin{table}[h]
    \centering
    \begin{subtable}{\textwidth}
    \centering
        \begin{tabular}{c|cccc}
        \toprule
    & \makecell{$\min$ \\ AURC} & \makecell{$\max$ \\ AURC} & \makecell{$\min$ \\ AUGC} & \makecell{$\max$ \\ AUGC} \\ 
    \hline
    \textit{SymbXAI} & \textbf{-0.39}  & \textbf{6.71}  & \textbf{-0.38}  & \textbf{6.68}  \\ 
    \textit{LRP} & 0.14  & 6.49  & 0.14  & 6.49  \\ 
    \textit{PredDiff} & 0.42  & 6.62  & 0.42  & 6.62  \\ 
    \textit{random} & 3.91  & 3.65  & 3.81  & 3.83  \\ 
    \bottomrule
    \end{tabular}
    \caption{Area under the flipping curves on the SST dataset for 170 text samples. }\label{tab:perturbation_sst}
    \end{subtable}
    \hfill 
    \begin{subtable}{\textwidth}
        \centering
    \begin{tabular}{c|cccc}
    \toprule
    & \makecell{$\min$ \\ AURC} & \makecell{$\max$ \\ AURC} & \makecell{$\min$ \\ AUGC} & \makecell{$\max$ \\ AUGC} \\ 
    \hline 
    \textit{SymbXAI} & \textbf{-5.42}  & \textbf{6.16}  & \textbf{-5.56}  & \textbf{6.04}  \\ 
    \textit{LRP} & -4.78  & 5.47  & -4.78  & 5.47  \\ 
    \textit{PredDiff} & -4.52  & 5.4  & -4.52  & 5.4  \\ 
    \textit{random} & 0.27  & 0.23  & 0.19  & 0.5  \\ 
    \bottomrule
    \end{tabular}
    \caption{Area under the flipping curves on the Movie Reviews dataset for 97 text samples. }
    \label{tab:perturbation_imdb}
    \end{subtable}
    \caption{Results of the flipping analysis on the SST and Movie Reviews datasets. Each table provides the area under the flipping curve for different optimization tasks, such as maximizing or minimizing the AURC or AUGC.  
    }
\end{table}

In Tables~\ref{tab:perturbation_sst} and \ref{tab:perturbation_imdb}, we see the results of the flipping analysis on the SST \cite{socher2013recursive} and Movie Reviews \cite{imdb} datasets, respectively. To generate the flipping curves, we consider the model's output $f(\bm{X}_\mathcal{S})$ for a subset of word indices $\mathcal{S}$, where the missing words in $\bm{X}_\mathcal{S}$ are simply discarded, not inpainted. The output of the model $f(\bm{X}_\mathcal{S})$ describe the `logits' of the classification task, meaning the output of the model before parsing it into a probability. Negative values describe low or even opposite evidence for the target class, where positive values describe evidence for the target class. 

\begin{figure}[!hb]
    \includegraphics[width=\textwidth]{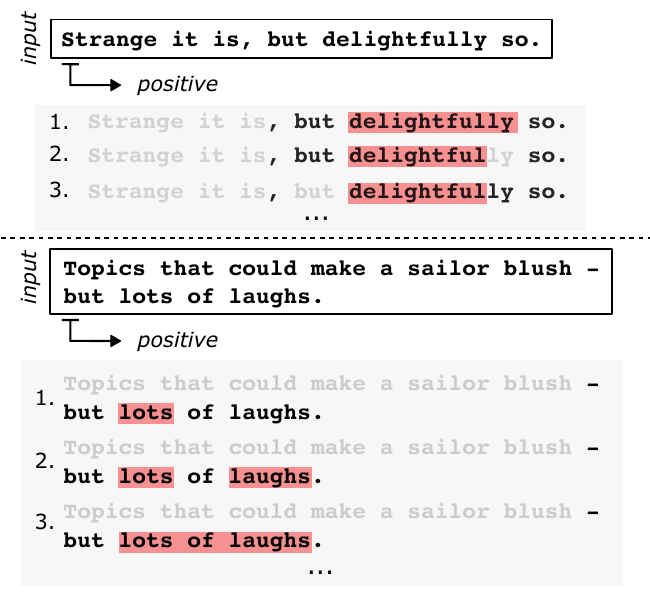}
    \caption{The image shows the most expressive queries for two different contrastive conjunction sentences. Both sentences have positive sentiments. Each sentence is accompanied by its three most expressive queries, ranked based on the similarity measure $\text{corr}(\bm{\mu},q)$, represented as a list. For the query representation, gray font and red background considers the absence and presence of the words in the sentence, respectively. For example the query $q = `\neg (\text{strange it is}) \wedge \text{delightfully}$' is visualized by ``{\color{lightgray}Strange it is}, but \hl{delightfully} so.''} 
    \label{subfig:quali_contr_conj_auto_nlp}
\end{figure}

  \begin{table}[!hb]
    \centering
    \begin{tabular}{c|c|c}
      $q^+ \supseteq $ & $\mathbb{b}$ & $ \neg \mathbb{a}$   \\
      \hline
       & 83.9\%& 93.5\% \\
    \end{tabular}
    \caption{Frequency of sub-queries $\neg \mathbb{a}$ and $ \mathbb{b}$ in the most expressive query $q^+$. This quantitative evaluation is performed using 36 samples of the SST dataset. Here, tokens $\mathbb{a}$ and $\mathbb{b}$ refer to any token in the subsentence  $\mathbb{A}$ and $\mathbb{B}$, respectively, in contrastive conjunction sentences of the form ''$\mathbb{A} \text{ but } \mathbb{B}$''. In all queries, the ``[CLS]'', ``[SEP]'', and punctuation tokens are not considered.} \label{table:perform_qplus_contr_conj}
  \end{table}

Our method, SymbXAI, outperforms other approaches across all tasks, with a particularly notable margin on the Movie Reviews dataset. This success stems from our ability to construct queries that capture both the presence and absence of features during the flipping procedure, which classical explanation methods fail to consider. The relevance of these queries is highly effective in predicting the model's behavior when flipping text data.
In \ref{app:additional_perturbation_analysis}, we provide the explicit flipping curves in Figure~\ref{fig:perturbation_cuves_sst} and \ref{fig:perturbation_cuves_imdb}. For first-order methods, like \textit{LRP} and \textit{PredDiff}, the areas under the flipping curves are identical for both removal and generation tasks, differing only in maximization or minimization scenarios. This is because the sequence of model inputs, and thus the flipping curve values, are the same for AURC and AUGC, but just in reverse order. In contrast, SymbXAI models removal and addition of features differently. This shows the flexibility and applicability of our framework in addressing specific challenges related to input flipping.

\paragraph{Automatic query search for contrastive conjunctions}\label{sect:automat_nlp_cont_conj}
In this subsection, we perform the search algorithm to find most expressive queries for the model's prediction, as proposed in Section \ref{sec:automatization}. To be able to interpret the result of the query search we focus on  on a specific subset of the SST dataset, namely contrastive conjunctions \cite{socher2013recursive}.

Contrastive conjunctions are sentences that  adheres to the structure ``$\mathbb{A}$ `but' $\mathbb{B}$'', where $\mathbb{A}$ and $\mathbb{B}$ represent two distinct phrases. $\mathbb{A}$ and $\mathbb{B}$ exhibit contrasting sentiments (excluding neutral sentiment), where the sentiment of $\mathbb{B}$ aligns with the target class and $\mathbb{A}$ has the opposite sentiment. An example of such sentence is ``Strange it is, but delightfully so.'', which is the first example in  Figure \ref{subfig:quali_contr_conj_auto_nlp}. For this example, we have $\mathbb{A}= ``\text{Strange it is,}$'' and  $\mathbb{B}= ``\text{delightfully so.}$''.

To control the number of queries in Equation \eqref{eq:opt_automat_general}, we introduce small restrictions on the composition of  $q \in \mathcal{Q}$ for our NLP data. Each subsentence identifier, i.e., the identifier of a feature set $\mathcal{S}$, must correspond to consecutive words within the sentence, and each word in any subsentence identifier can only occur once in the query. Additionally, the $\wedge$ operator can appear at most twice in  $q$. These restrictions reduce the number of queries and make the queries more readable for users.
\medskip

The results of of the query search in this setting are given in Figure \ref{subfig:quali_contr_conj_auto_nlp} and Table~\ref{table:perform_qplus_contr_conj}. In Figure \ref{subfig:quali_contr_conj_auto_nlp}, we observe that for both examples, the most expressive queries indicate that the model identifies the presence of words from subsentence $\mathbb{B}$ and the absence of words from subsentence $\mathbb{A}$ as important factors for its predictions. For example, in the first sentence, considering $``\text{delightfully}"$ which is a subsentence in $\mathbb{B}$,  and $``\text{Strange it is}"$ which is a subsentence in $\mathbb{A}$, the search algorithms finds the  query `$ \neg (\text{Strange it is}) \wedge \text{delightfully}$' to be most expressive. We obtain a similar result for the second sentence.

\begin{figure*}
    \centering
    \includegraphics[width=.65\textwidth]{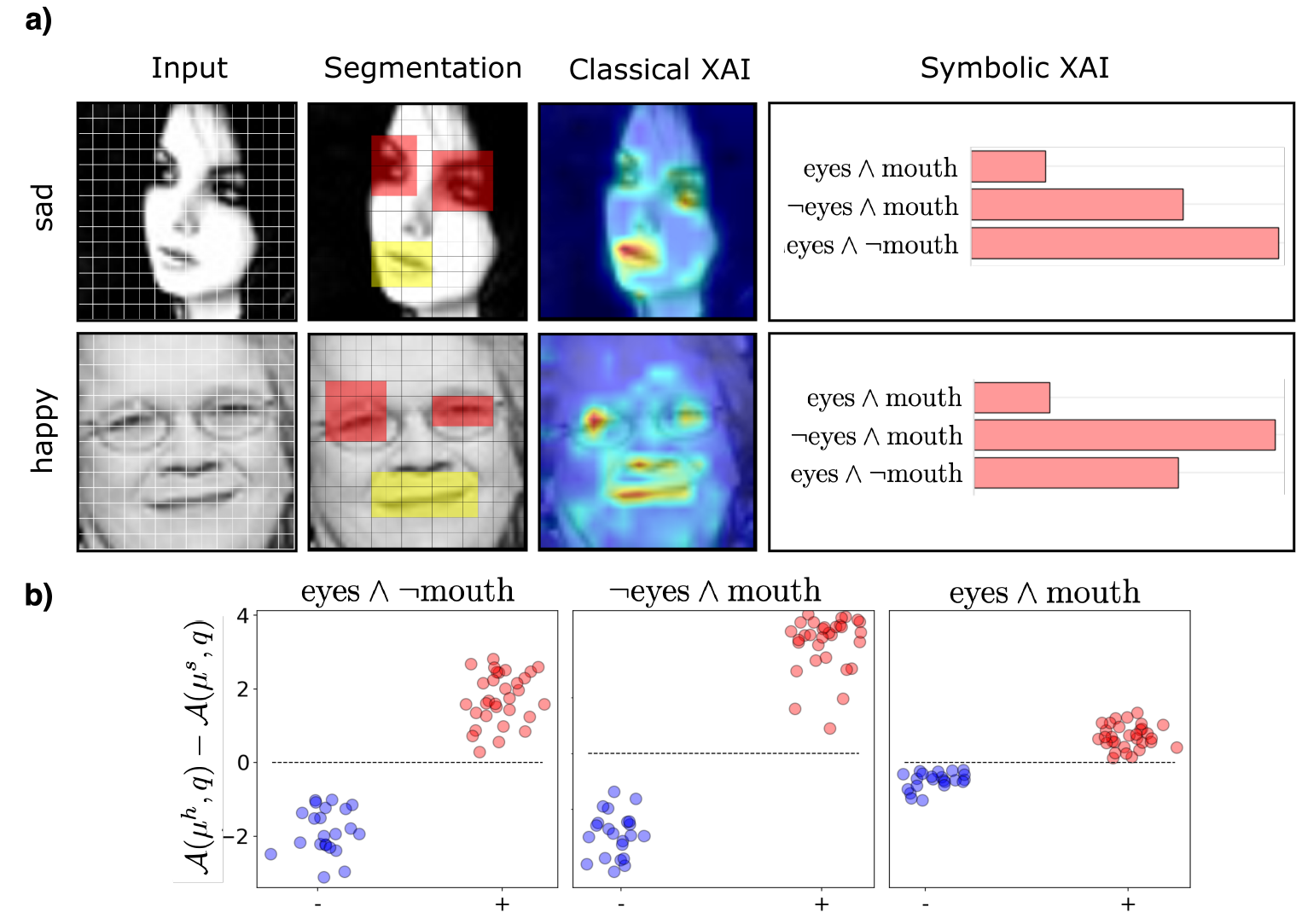}
    \caption{Application of SymbXAI on a ViT-Base model using the FER-2013 dataset. Sub-figure \textbf{a)} presents a table showing the results for exemplary images categorized as `\textit{sad}' and `\textit{happy}'. For each sample, the segmentation results (red for `eyes', yellow for `mouth'), the heatmap of classical feature-wise explanation and the result of the SymbXAI framework are given.
    \textbf{b)} shows the difference between the relevance of queries for the happy and sad classes, i.e., $\mathcal{A}(\bm{\mu}^h, q) - \mathcal{A}(\bm{\mu}^s,q)$. Here, $\bm{\mu}^h$ and $\bm{\mu}^s$ represent the multi-order decomposition of the model's output for the happy and sad classes, respectively. In \textbf{b)} the samples with the target class `\textit{sad}' are annotated in blue, while those with the target class `\textit{happy}' are annotated in red.
    Throughout the image, the specifications `eyes' and `mouth' correspond to the feature sets $\mathcal{S}_e$ and $\mathcal{S}_m$, which represent the eyes and mouth of the face, respectively. In all of the explanation experiments, we use the `\textit{occlusion}' weight vector $\bm{\eta} = 1$.}
    \label{fig:vision_summary}
\end{figure*}

In Table \ref{table:perform_qplus_contr_conj}, we quantitatively evaluate the most expressive query with respect to contrastive conjunction sentences. We can see in the majority of examples, the most expressive query $q^+$ includes words $\mathbb{b} \subseteq \mathbb{B}$ (83.9\%) and exclude words $\mathbb{a} \subseteq \mathbb{A}$ (93.5\%), which is consistent with the intuition developed earlier.

We observe that for predicting the sentiment of contrastive conjunctions (sentences that include the word "but" with contrasting subsentences), Symbolic XAI suggests a simpler approach: focus only on the subsentence following the word "but." By basing the sentiment prediction of the entire sentence solely on this subsentence, we see an accuracy loss of only 0.24\% on 410 samples, decreasing from 84.63\% accuracy when incorporating the full sentence to 84.39\% when predicting the sentiment using only $\mathbb{B}$. Remarkably, this insight was obtained by analyzing just 36 samples. This demonstrates that our explanation framework can extract valuable insights from an ML model using a relatively small set of samples.

In summary, the model is aware of the contrastive nature of sentences. 
The SymbXAI query search shows that the model ignores in the above sentence structure the first part of the sentence $\mathbb{A}$ and focuses on the second part $\mathbb{B}$, to compose the prediction. 
This aligns with human intuition of how one should interpret contrastive conjunction sentences.

\subsection{Usage in vision}
Computer Vision (CV) is one of the central domains heavily using ML (e.g., \cite{lecun1998doc_rec,krizhevsky2012alexnet,deng2009imagenet}). Multiple tasks fall under this domain, including object detection \cite{zhao2029obj_detect}, character recognition \cite{lecun1998doc_rec}, or facial expression recognition \cite{ko2018fer_review}, to name a few. Different ML models, including convolutional neural networks \cite{lecun1998doc_rec,lecun95convolutional, krizhevsky2012alexnet}, Transformer models \cite{vaswani2017attention,dosovitskiy2021ViT} and others \cite{lampert2009kernelCV, hochreiter1997long}, gained astonishing performance on many of the CV tasks. 
A number of explanation strategies has been proposed for vision models \cite{samek2021xaireview,montavon-pr17, bach-plos15, bas2022xai_med_image}, where most of them provide pixel-wise heatmaps to highlight the pixels that are relevant for the prediction.

\paragraph{Dataset}
For the FER task, we focus on the FER2013 dataset \cite{goodfellow2013fer2013}, which contains approximately 30K images of seven different facial expressions: \textit{sad}, \textit{disgust}, \textit{angry}, \textit{neutral}, \textit{fear}, \textit{surprise}, \textit{happy}. 

\paragraph{Experimental setup}
To demonstrate the efficacy of our proposed framework, we leverage it to explain the predictions of a pre-trained vision transformer model \cite{dosovitskiy2021ViT}, namely ViT-Base, which is finetuned on a FER dataset \cite{goodfellow2013fer2013}.  The ViT-Base reaches with around 90\% classification accuracy on the FER2013 dataset. 
For more details on the model, we refer to   \ref{app:cv_model_details}. The analysis is twofold. First, we show in Section \ref{sec:FER} both qualitatively and quantitatively how to construct meaningful queries and calculate their relevance values, thereby gaining insights into the model's prediction strategy. Second, we show in Section \ref{sec:hygienic_mask_eval} how our SymbXAI framework can help to estimate the model's performance on unseen data, such as faces that are partially obstructed, e.g. due to hygienic masks or sunglasses.

\subsubsection{Exploring query relevance on the FER task} \label{sec:FER}

We aim to explore the relevance of different queries to obtain a profound understanding of the model's prediction. In this subsection, we focus on two categories of this dataset, `\textit{sad}' and `\textit{happy}', which represent contrasting emotional states. Given the inherently opposing nature of these emotions, we anticipate that when a facial feature positively influences one emotion, it would inversely affect the other emotion class.

In Figure \ref{fig:vision_summary}\textbf{a)} we see how the SAM algorithm properly segments the eyes and mouth of the corresponding input image. The results from the classic XAI method indicate that the model particularly focuses on these segmented areas. With SymbXAI, we see that the conjunctive interaction between the eyes and mouth, i.e., $q = \text{eyes} \wedge \text{mouth}$, does not significantly influence the target prediction for both classes. The presence of the mouth without eyes, i.e., $q = \neg \text{eyes} \wedge \text{mouth}$, and the presence of the eyes without mouth, i.e., $q = \text{eyes} \wedge \neg \text{mouth}$, play significant roles. Specifically, the mouth without eyes ($q = \neg \text{eyes} \wedge \text{mouth}$) is more important for the happy image, while the eyes without mouth ($q = \text{eyes} \wedge \neg \text{mouth}$) is more important for the sad image. 
This can be a reasonable strategy when examining the images, as a smile is a strong indicator of happiness, whereas the mouth of a sad person can be interpreted more neutrally, with the eyes serving as a stronger indicator of sadness.

\begin{figure}[!htb]
    \centering
    \includegraphics[width=0.8\textwidth]{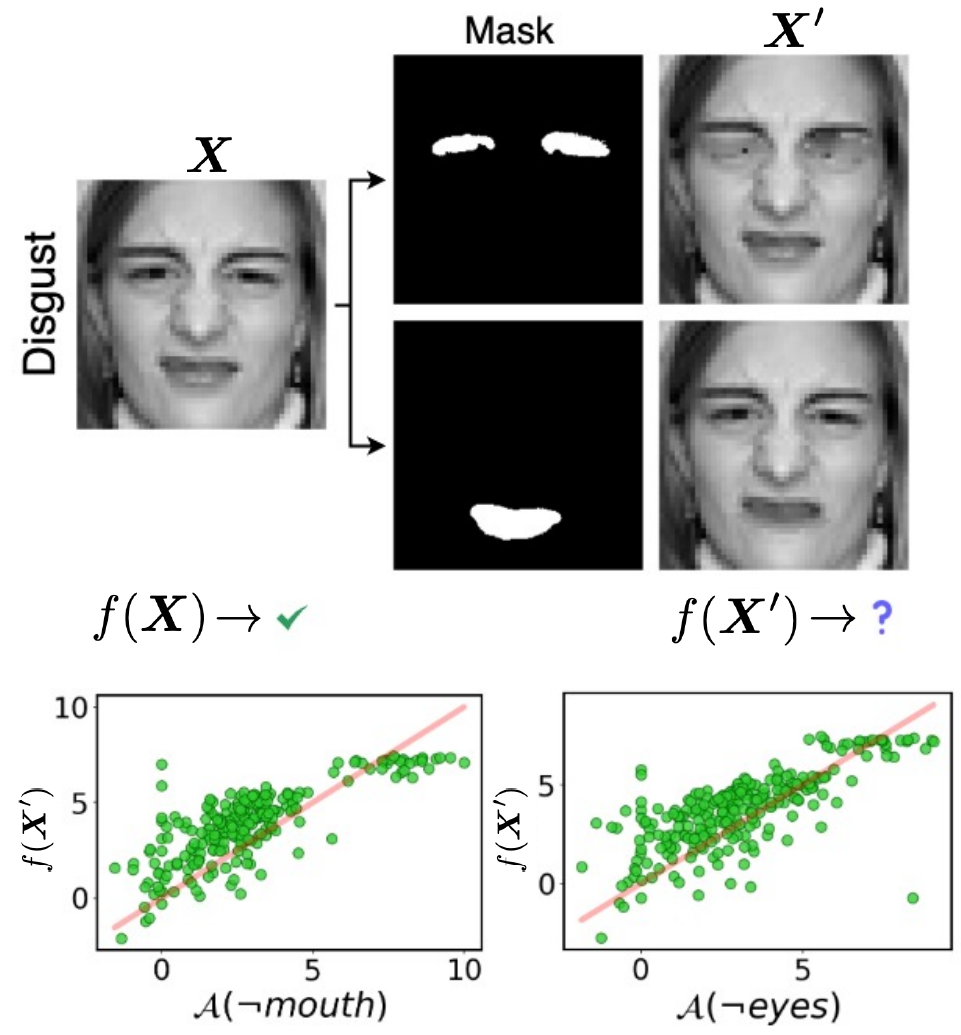}
    \caption{Employing the Symbolic XAI framework to preview the model's performance on unseen face images that are partially obscured. We use the ViT-Base model that is trained on the FER-2013 dataset. $\mathcal{A}( \neg \text{mouth})$ and $\mathcal{A}(\neg \text{eyes})$ is the relevance of the queries $q=\neg \text{mouth}$ and $q=\neg \text{eyes}$ with respect to the multi-order decomposition $\bm{\mu}$ of the prediction $f(\bm{X})$ for the original image $\bm{X}$. $f(\bm{X}')$ is the prediction of the model for the inpainted image $\bm{X}'$.}
    \label{fig:vision_mouth_coverage_summary}
\end{figure}

In Figure \ref{fig:vision_summary}\textbf{b)} we analyze the complimentary nature of the \textit{sad} and \textit{happy} classes, by considering the contributions of the eyes and mouth segments to the happy class minus their contributions to the sad class, i.e., $\mathcal{A}(\bm{\mu}^h,q) - \mathcal{A}(\bm{\mu}^s,q)$, for different queries $q$. In this figure, images that belong to the \textit{sad} class are represented with blue blobs and those in the \textit{happy} class are shown as red blobs. We can observe that when the eyes  and mouth segments, and also their conjunctive interaction contribute positively to the \textit{sad} class,  they either contribute negatively or less positively to the \textit{happy} class, and vice versa. These observations confirms our intuition that eyes and mouth express an emotion. Moreover, when looking at the relevance of the query $q =\text{mouth} \wedge \neg \text{eyes}$, we note that the mouth segment has a higher contribution to the `\textit{happy}' class, which shows that overall the model finds the mouth to be a good indicator for happiness. A similar pattern is observed for the query $q = \neg \text{mouth} \wedge \text{eyes}$ in relation to the \textit{sad} class, although this trend is slightly less pronounced.

\subsubsection{Preview performance on augmented data using Symb- XAI}\label{sec:hygienic_mask_eval}
During the COVID-19 pandemic, it became evident that certain face recognition models encounter difficulties in accurately identifying faces obscured by hygienic masks \cite{MAGHERINI2022covidMasks, elsayed2023FER_masked}. Therefore, it is crucial for these models to be robust on such obscured images. Similarly, facial expression recognition models can also benefit from such robustness, ensuring precise recognition of facial expressions even when facial features are partially obscured due to facial hair, accessories, hygienic masks, or other obscuring factors.  In this experiment, we use the pre-trained ViT-Base model and FER-2013 dataset, which were introduced in Section~\ref{sec:FER}. Note that for all relevance values, we used the constant weight vector $\bm{\eta} = 1$.

A common method to assess model's robustness involves data augmentation by blocking out semantically meaningful patches of the input. However, this strategy, if not carefully implemented, can introduce unnatural artifacts. For instance, a solid black or white patch introduces sharp edges in the image, that can differ significantly from any natural way of obscuring the image (like shadows or other objects) which the model can usually handle. 
 
An example of an augmentation strategy using a simple inpainting method \cite{talea2004inpainter} is given in Figure \ref{fig:vision_mouth_coverage_summary}. 

As mentioned in Section~\ref{sec:FER}, mouth and eyes are important facial features for the FER task. Consequently, obscuring these areas, such as when wearing a hygienic mask or sunglasses, is expected to impact the model's performance. To analyze this hypothesis, we use our SymbXAI framework. This framework enables us to preview the model's performance on obscured data without directly modifying the images, providing valuable insights into the model's robustness when they are obscured. 

We use the queries $q=\neg \text{mouth}$ and $q=\neg \text{eyes}$, which considers all facial features except for the mouth and eyes, respectively. We represent the relevance of these queries with respect to the multi-order decomposition $\bm{\mu}$ of prediction $f(\bm{X})$, for the original image $\bm{X}$, by $\mathcal{A}(\neg \text{mouth})$ and $\mathcal{A}(\neg \text{eyes})$.
In Figure~\ref{fig:vision_mouth_coverage_summary}, we compare the relevance values $\mathcal{A}(\neg \text{mouth})$ and $\mathcal{A}(\neg \text{eyes})$ against the model's prediction $f(\bm{X}')$ on the obscured images $\bm{X}^\prime$. We generate obscured images by detecting the mouth and eyes area using the SAM algorithm and subsequently inpainting these regions with the inpainting algorithm described in \cite{talea2004inpainter}. We can note a correlation between $\mathcal{A}( \neg \text{mouth})$ and $f(\bm{X}')$. This means that the SymbXAI framework is able to reflect how the model performs when mouth area of the images are obscured, and is therefore useful to preview how a model can be generalized to unseen data. We can observe a similar correlation between $\mathcal{A}( \neg \text{eyes})$ and $f(\bm{X}')$.

\subsection{Usage in quantum chemistry} \label{sect:qc_experiments}
The usage of ML models in quantum chemistry are experiencing a high popularity due to their high accuracy with relatively small computational costs~\cite{rupp2012fast,schutt2018schnet, schutt2017schnet, schutt2021equivariant, batzner20223, gasteiger_dimenet_2020, thomas2018tensor, gilmer2017neural, musaelian2023learning, unke2022accurate, han_deep_2017, unke2019physnet, unke2021spookynet, batzner20223, klicpera2020directional, NEURIPS2022_4a36c3c5, frank2022so3krates, NEURIPS2020_15231a7c, thomas2018tensor, haghighatlari2022newtonnet, gugler2022, kahouli2024}.
To ensure the generalizability of these ML models, it is imperative to move beyond simple prediction errors and employ more sophisticated metrics~\cite{unke2021machine, fu2022forces, frank2024euclidean}. It is desirable for these metrics to encompass broader physical concepts and resonate with the intuition of experts in their respective research domains. As a consequence, the significance of XAI is growing within quantum chemistry~\cite{schuett2019xaiqc, schuett2017dtnn, cho2020lrpInterNet, mcclo19usingatt, collins2023xaiMolFrag,bonneau2024peering}. XAI has already demonstrated its utility in predicting toxicity or mutagenicity~\cite{debnath1991structure,baehrens2010explain,Kazius2005DerivMutag, xiong22asubgraph, xiong23bMaxProd, luo2020parameterized}. Moreover, XAI is able to guide strategies in drug discovery~\cite{jose20drugdisc}.

As discussed in Section \ref{sect:back_rel}, many of the standard XAI techniques are limited in their expressiveness.

To illustrate this, consider a common scenario: One prevalent form of post-hoc explanation assigns a scalar value to each atom, indicating its significance for the model's prediction. Yet, when the model predicts the potential energy of a molecule, providing only atom-wise energy contributions is in general not sufficient. Since a great portion of the energy stems from interactions between the atoms, not individual atoms themselves. 
We use the SymbXAI to inspect the relevance of different atom interactions.

\begin{figure*}[h]
        \centering
  \begin{center}
    \includegraphics[width=.8\textwidth]{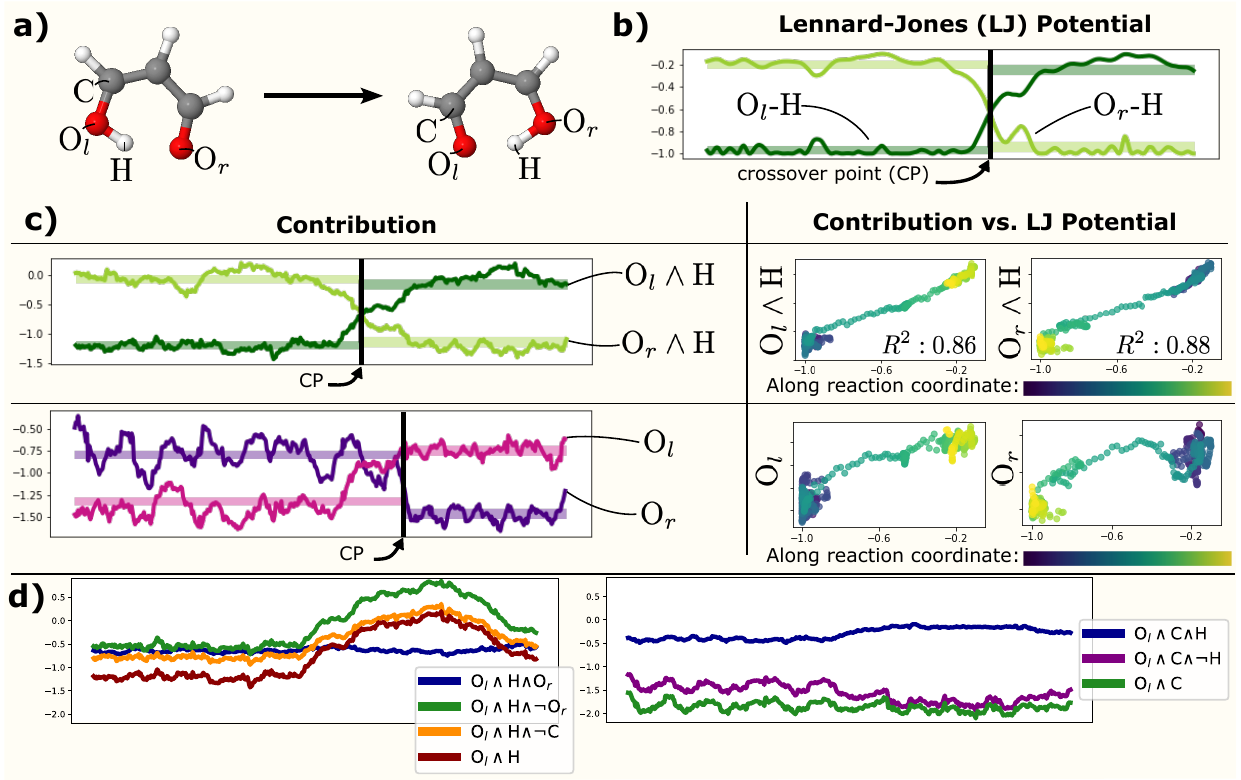}
  \end{center}
  \caption{MD simulation of malondialdehyde. \textbf{a)} shows a proton transfer reaction of MDA. The left conformation is in the early time points of the MD trajectory, the right conformation is close to the final time point. In \textbf{b)}, the Lennard-Jones potential over time between the oxygen and hydrogen atoms on the left (dark green) and right (light green) is visualized. In the table, \textbf{c)}, we show the contribution for two different queries, which contributions are compared to the Lennard-Jones potential. In the first row, the relevance of the conjunctive query between the oxygen and hydrogen atom, is shown.
  The second row shows the classic first-order relevance of the left and right oxygen atom, expressed by the first-order relevance of O$_l$ and O$_r$. Note that in all results we apply the occlusion values for queries, i.e., $\bm{\eta} = 1$. In \textbf{d)} we show the relevance of different queries that are composed of the left oxygen O$_l$ and hydrogen atom H, and of the left oxygen atom O$_l$ and carbon atom C. } \label{fig:qc_intro_image1}
\end{figure*}

\paragraph{Interpreting the multi-order model decomposition}
In quantum chemistry a well known theoretical concept is the many-body expansion \cite{ouyang2014troubleMBE}, expressed by
\begin{align}\label{eq:mbe_chem}
    E = \sum_I E_I + \sum_{I<J} \Delta E_{IJ} + \sum_{I<J<K} \Delta E_{IJK} + \dots
\end{align}
where, $\Delta E_{ I J } := E_{ I J } -  E_{ I} -  E_{ J }$ and $\Delta E_{ I J K} : = E_{ I J K } - \Delta E_{ I J } - \Delta E_{ J K } - \Delta E_{ I K } - E_{ I } -  E_{ J } -E_{ K } $. The term $E_I$ represents the energy contribution of the atom $I$, while $\Delta E_{IJ}$ and $\Delta E_{IJK}$ are two- and three-body correction terms for the atom pair $IJ$ and triplet $IJK$, respectively. In general, for any set of atoms $IJ \dots$, the energy contribution $\Delta E_{IJ \dots}$ is expected to be distinctively dependent on the specified atoms and independent of any other atoms.

It becomes clear that such many-body expansion is conceptually similar to the multi-order decomposition with the Harsanyi dividends in Section \ref{subsect:decompo_model}, meaning that $\Delta E_\mathcal{L} = \mu_\mathcal{L}$ when the model $f(\bm{X})$ predicts the energy. 

\medskip

We present a few experimental scenarios in which we construct and compute the relevance of queries that are useful in the context of quantum chemistry, with a  focus on the comparison between the queries and the many-body terms.

\subsubsection{Explanation of a molecular dynamics trajectory}\label{sec:mda_traj}

We consider the molecular dynamics (MD) trajectory of a proton transfer within malondialdehyde (MDA) in its enol form. Enol form means that one of the hydrogen atoms is bonded to an oxygen atom, which introduces a double bond between two carbon atoms. During this reaction, the hydrogen atom dissociates from the oxygen atom and subsequently forms a new bond with the adjacent oxygen atom (see Figure \ref{fig:qc_intro_image1}\textbf{a)}).

We anticipate that this process would be primarily governed by pairwise interactions between the hydrogen and oxygen atoms that are involved in the reaction. We want to validate this expectation by explaining a ML force field model used for a MDA simulation: We train a SchNet model \cite{schuett2017dtnn, schutt2017schnet} equipped with three interaction blocks on the MDA dataset \cite{schuett2019schnorb} utilizing the software package SchNetPack~\cite{schuett2018schnetpack, schutt2023schnetpack}. This results in a satisfactory accuracy of \SI{0.35}{\milli\electronvolt} for the energy and \SI{2.16}{\milli\electronvolt/\angstrom} for the forces over the course of a trajectory (see   \ref{app:qc_train_data}).

In Figure \ref{fig:qc_intro_image1}\textbf{b)}, we see the Lennard-Jones (LJ)  potential \cite{Lennard_Jones_1931, Jones1924} between the hydrogen and the oxygen atom on the left, O$_l$-H and on the right, O$_r$-H (for details on LJ potential, see~\ref{app:len_jones}). This sub-figure also illustrates a constant approximation and the \textit{crossover point} (CP) of the hydrogen atom: The constant approximations are two step functions, each of which exhibits one step at the same time point, so that the square distance between both time series and their constant approximations is minimized.

To analyze the behavior of the SchNet model with respect to the chemical reaction, we consider three different abstract queries within the SymbXAI framework. We then compare their relevance with the LJ potential between the oxygen and hydrogen atoms that undergo the proton transfer. Generally, both the LJ potential between atoms and the joint contribution of atom pairs serve as proxies for the energy contribution to the overall energy of the molecule.  Therefore, we assess that if these quantities correlate, the model aligns with chemical intuition.

The results in Figure \ref{fig:qc_intro_image1}\textbf{c)} reveal that the contextualized conjunctive contributions O$_l \wedge $H and O$_r \wedge $H have strong positive linear correlation with the LJ potential (see $R^2$ scores in scatter plot), while first-order explanation of the oxygen atoms O$_l$ and O$_r$, lead to no linear correlation with the LJ potential.
The first-order contribution of the oxygen atoms has only a weak correlation with the LJ potential, and for values that are distant to the CP no correlation is visible.
We also see that the CP aligns well with the CP of the LJ potential.
This shows qualitatively that the proton transfer is best captured with the given symbolic queries.

Figure \ref{fig:qc_intro_image1}\textbf{d)} shows additional query relevance of the MDA trajectory. On the left, queries O$_l \wedge $H$ \wedge \neg$O$_r$, O$_l \wedge $H $ \wedge \neg$C, and O$_l \wedge $H show a similar trend, remaining stable initially and increasing in later time points. This behavior aligns with the LJ potential, as discussed in Figure \ref{fig:qc_intro_image1}\textbf{c)}. The absence of C and O$_r$, shown by the orange and green lines, respectively, causes an energy shift towards greater instability (positive energy). The relevance for the query  O$_l \wedge $H$ \wedge $O$_r$ does not change much over time. We interpret this 3-body term relevance as primarily governed by the pairwise contributions of queries O$_l \wedge $H and  O$_r \wedge $H. The sum of the relevance for these atom pairs results in an almost constant function over time, as well.

On the right side of Figure \ref{fig:qc_intro_image1}\textbf{d)} we see contributions of queries that involve the C and O$_l$ atoms.
At the trajectory's start, there is a single bond between two atoms.
Over the course of the reaction, this bond transitions to a double bond, which leads to a shortening of the bond.
The queries O$_l \wedge $C$\wedge \neg$H and O$_l \wedge$C show a slight drop in relevance over time, resembling the LJ potential behavior. The absence of the H atom reinforces this relevance difference after the CP. This is natural since we know from previous investigations that the interaction between the H and O$_l$ atom has positive influence on the energy contribution after the CP, so that the absence naturally leads to more negative values. The relevance of the 3-body query O$_l \wedge $C$\wedge$H shows exactly this positive effect of the H atom after the CP.

SymbXAI effectively recovers and interprets chemically plausible dependencies in the model. This framework is therefore able to provide insights about the complex prediction strategy of ML models on QC tasks and has the potential to be applicable in other natural sciences. Hence, SymbXAI helps to determine whether the model aligns with chemical intuition and reveals interesting ML behaviors, which can even lead to new scientific insights.

\section{Conclusion and Future Work}
We found that it is possible to develop an explanation framework that attributes relevance to human-readable logical relationships in a model. Additionally, we designed a search algorithm that automatically identifies logical formulas that best summarize the model's prediction strategy. Across application domains such as vision, NLP, and chemistry, our novel framework offers insights into the abstract prediction strategies of complex AI models, tailored to the user's interests expressed in a formal language—an abstraction that was previously unattainable in this form.

A further finding is that both perturbation-based and propagation-based explanation methods can be utilized within our framework, extending the application of cooperative game theory, which was traditionally associated with perturbation-based methods.

Analyzing and interpreting deep learning in complex scenarios is challenging. Simple heatmaps from first-order XAI or even higher-order XAI methods may not present information in an abstract, problem-relevant manner. The quest for a formal foundation for more abstract knowledge representation in XAI frameworks is ongoing, with recent developments in concept-based XAI, among others \cite{achtibat23concept_rel, schramowski2020making}. Moreover, users may want to specify the level of abstraction that XAI techniques should provide, which is precisely what our proposed SymbXAI framework offers by leveraging logic. SymbXAI provides an interpretable abstraction of explanatory queries, consisting of abstract formulas composed of logical connectives and input features. 

In the future, it would be beneficial to explore and address the complexity of this framework. For instance, by distilling unimportant multi-order terms from the model prediction in Equation \eqref{eq:pred_decompo_multi_order_subset}, leading to sparser models, or by finding alternative ways to compute the relevance of queries beyond the multi-order terms. Further automation in identifying meaningful queries would be desirable. One avenue for exploration could involve comparing different similarity measures as in Equation \eqref{eq:opt_automat_general}, or developing efficient search algorithms for a large query space.

In summary, our work represents a step towards a human-centered explanation of AI models that is both flexible and provides human-readable explanatory features.

\paragraph{Authors contributions}
TS, GM, KRM and JL where involved in the conceptualization of the project. TS and FRJ managed the data curation of the experiments. KRM is responsible for the funding acquisition of this project. TS and FRJ conducted the investigation of the experiments of the NLP and vision experiments, where TS, JL and SG conducted the investigation of the experiments on the quantum chemical domain. TS, GM and SN developed the methodology of the framework, KRM contributed to it. KRM, GM and TS where responsible for the project administration and supervision. TS and FRJ developed the software that lead to the results and validated the outcome. TS and FRJ created all visualizations of the manuscript. TS, FRJ and JL wrote the original draft. GM, SN, KRM, SG
and PX conceptualized, iterated and discussed the project's experiments and analyses. GM, SN, KRM, SG and PX reviewed and edited the manuscript.

\paragraph{Declaration of generative AI and AI-assisted 
technologies in the writing process}
During the preparation of this work some of the author(s) (TS, FRJ and JL) used ChatGPT (OpenAI) in order to improve the readability and spelling of the text. After using this tool/service, the author(s) reviewed and edited the content as needed and take(s) full responsibility for the content of the publication.

\paragraph{Declaration of competing interest}
The authors declare that they have no known competing financial interests or personal relationships that could have appeared to influence the work reported in this paper.

\section*{Acknowledgment}
We gratefully acknowledge valuable discussions with Oliver Unke, Stefan Chmiela, Alexandre Tkatchenko, O. Anatole von Lilienfeld, Kristof Schütt, Michael Gastegger,  Stefan Blücher and Marco Morik. This work was in part supported by the Federal German Ministry for Education and Research (BMBF) under Grants BIFOLD24B, 01IS14013A-E, 01GQ1115, 01GQ0850, 01IS18025A,\\ 031L0207D, and 01IS18037A. K.R.M.\ was partly supported by the Institute of Information \& Communications Technology Planning \& Evaluation (IITP) grants funded by the Korea government (MSIT) (No. 2019-0-00079, Artificial Intelligence Graduate School Program, Korea University and No. 2022-0-00984, Development of Artificial Intelligence Technology for Personalized Plug-and-Play Explanation and Verification of Explanation). Correspondence to TS, GM and KRM.

\bibliographystyle{abbrvnat}
\bibliography{references}

\begin{thebibliography}{133}
\providecommand{\natexlab}[1]{#1}
\providecommand{\url}[1]{\texttt{#1}}
\expandafter\ifx\csname urlstyle\endcsname\relax
  \providecommand{\doi}[1]{doi: #1}\else
  \providecommand{\doi}{doi: \begingroup \urlstyle{rm}\Url}\fi

\bibitem[Achtibat et~al.(2023)Achtibat, Dreyer, Eisenbraun, Bosse, Wiegand,
  Samek, and Lapuschkin]{achtibat23concept_rel}
R.~Achtibat, M.~Dreyer, I.~Eisenbraun, S.~Bosse, T.~Wiegand, W.~Samek, and
  S.~Lapuschkin.
\newblock From attribution maps to human-understandable explanations through
  concept relevance propagation.
\newblock \emph{Nature Machine Intelligence}, 5\penalty0 (9):\penalty0
  1006--1019, 2023.

\bibitem[Ali et~al.(2022)Ali, Schnake, Eberle, Montavon, M{\"{u}}ller, and
  Wolf]{ali2022xaiTrans}
A.~Ali, T.~Schnake, O.~Eberle, G.~Montavon, K.-R. M{\"{u}}ller, and L.~Wolf.
\newblock {XAI} for transformers: Better explanations through conservative
  propagation.
\newblock In \emph{International Conference on Machine Learning, {ICML} 2022,
  17-23 July 2022, Baltimore, Maryland, {USA}}, volume 162 of \emph{Proceedings
  of Machine Learning Research}, pages 435--451. {PMLR}, 2022.

\bibitem[Andrews(1986)]{andrews1986math_logic_type}
P.~B. Andrews.
\newblock \emph{An introduction to mathematical logic and type theory - to
  truth through proof}.
\newblock Computer science and applied mathematics. Academic Press, 1986.
\newblock ISBN 978-0-12-058535-9.

\bibitem[Arras et~al.(2019)Arras, Arjona-Medina, Widrich, Montavon, Gillhofer,
  M{\"u}ller, Hochreiter, and Samek]{arras-lncs19}
L.~Arras, J.~Arjona-Medina, M.~Widrich, G.~Montavon, M.~Gillhofer, K.-R.
  M{\"u}ller, S.~Hochreiter, and W.~Samek.
\newblock Explaining and interpreting lstms.
\newblock In W.~Samek, G.~Montavon, A.~Vedaldi, L.~K. Hansen, and K.-R.
  M{\"u}ller, editors, \emph{Explainable AI: Interpreting, Explaining and
  Visualizing Deep Learning}, volume 11700 of \emph{Lecture Notes in Computer
  Science}, pages 211--238. 2019.
\newblock \doi{10.1007/978-3-030-28954-6_11}.

\bibitem[Bach et~al.(2015)Bach, Binder, Montavon, Klauschen, M{\"u}ller, and
  Samek]{bach-plos15}
S.~Bach, A.~Binder, G.~Montavon, F.~Klauschen, K.-R. M{\"u}ller, and W.~Samek.
\newblock On pixel-wise explanations for non-linear classifier decisions by
  layer-wise relevance propagation.
\newblock \emph{PLoS ONE}, 10\penalty0 (7):\penalty0 e0130140, 07 2015.
\newblock \doi{10.1371/journal.pone.0130140}.

\bibitem[Baehrens et~al.(2010)Baehrens, Schroeter, Harmeling, Kawanabe, Hansen,
  and M{\"u}ller]{baehrens2010explain}
D.~Baehrens, T.~Schroeter, S.~Harmeling, M.~Kawanabe, K.~Hansen, and K.-R.
  M{\"u}ller.
\newblock How to explain individual classification decisions.
\newblock \emph{The Journal of Machine Learning Research}, 11:\penalty0
  1803--1831, 2010.

\bibitem[Bahdanau et~al.(2015)Bahdanau, Cho, and Bengio]{bahdanau2015}
D.~Bahdanau, K.~Cho, and Y.~Bengio.
\newblock Neural machine translation by jointly learning to align and
  translate.
\newblock In Y.~Bengio and Y.~{L}e{C}un, editors, \emph{3rd International
  Conference on Learning Representations, {ICLR} 2015, San Diego, CA, USA, May
  7-9, 2015, Conference Track Proceedings}, 2015.

\bibitem[Batatia et~al.(2022)Batatia, Kovacs, Simm, Ortner, and
  Csanyi]{NEURIPS2022_4a36c3c5}
I.~Batatia, D.~P. Kovacs, G.~Simm, C.~Ortner, and G.~Csanyi.
\newblock Mace: Higher order equivariant message passing neural networks for
  fast and accurate force fields.
\newblock In S.~Koyejo, S.~Mohamed, A.~Agarwal, D.~Belgrave, K.~Cho, and A.~Oh,
  editors, \emph{Advances in Neural Information Processing Systems}, volume~35,
  pages 11423--11436. Curran Associates, Inc., 2022.

\bibitem[Batzner et~al.(2022)Batzner, Musaelian, Sun, Geiger, Mailoa,
  Kornbluth, Molinari, Smidt, and Kozinsky]{batzner20223}
S.~Batzner, A.~Musaelian, L.~Sun, M.~Geiger, J.~P. Mailoa, M.~Kornbluth,
  N.~Molinari, T.~E. Smidt, and B.~Kozinsky.
\newblock E (3)-equivariant graph neural networks for data-efficient and
  accurate interatomic potentials.
\newblock \emph{Nature communications}, 13\penalty0 (1):\penalty0 2453, 2022.

\bibitem[Becker et~al.(2024)Becker, Vielhaben, Ackermann, M{\"u}ller,
  Lapuschkin, and Samek]{becker2024audiomnist}
S.~Becker, J.~Vielhaben, M.~Ackermann, K.-R. M{\"u}ller, S.~Lapuschkin, and
  W.~Samek.
\newblock Audiomnist: Exploring explainable artificial intelligence for audio
  analysis on a simple benchmark.
\newblock \emph{Journal of the Franklin Institute}, 361\penalty0 (1):\penalty0
  418--428, 2024.

\bibitem[Behrmann et~al.(2019)Behrmann, Grathwohl, Chen, Duvenaud, and
  Jacobsen]{behrmann19InvResNet}
J.~Behrmann, W.~Grathwohl, R.~T.~Q. Chen, D.~Duvenaud, and J.-H. Jacobsen.
\newblock Invertible residual networks.
\newblock In K.~Chaudhuri and R.~Salakhutdinov, editors, \emph{Proceedings of
  the 36th International Conference on Machine Learning}, volume~97 of
  \emph{Proceedings of Machine Learning Research}, pages 573--582, Long Beach,
  California, USA, 09--15 Jun 2019. PMLR.

\bibitem[Bengio et~al.(2024)Bengio, Hinton, Yao, Song, Abbeel, Darrell, Harari,
  Zhang, Xue, Shalev-Shwartz, Hadfield, Clune, Maharaj, Hutter, Baydin,
  McIlraith, Gao, Acharya, Krueger, Dragan, Torr, Russell, Kahneman, Brauner,
  and Mindermann]{bengio2023managing}
Y.~Bengio, G.~Hinton, A.~Yao, D.~Song, P.~Abbeel, T.~Darrell, Y.~N. Harari,
  Y.-Q. Zhang, L.~Xue, S.~Shalev-Shwartz, G.~Hadfield, J.~Clune, T.~Maharaj,
  F.~Hutter, A.~G. Baydin, S.~McIlraith, Q.~Gao, A.~Acharya, D.~Krueger,
  A.~Dragan, P.~Torr, S.~Russell, D.~Kahneman, J.~Brauner, and S.~Mindermann.
\newblock Managing extreme ai risks amid rapid progress.
\newblock \emph{Science}, 384\penalty0 (6698):\penalty0 842--845, 2024.

\bibitem[Blücher et~al.(2022)Blücher, Vielhaben, and
  Strodthoff]{BLUCHER2022PredDiff}
S.~Blücher, J.~Vielhaben, and N.~Strodthoff.
\newblock Preddiff: Explanations and interactions from conditional
  expectations.
\newblock \emph{Artificial Intelligence}, 312:\penalty0 103774, 2022.
\newblock ISSN 0004-3702.
\newblock \doi{https://doi.org/10.1016/j.artint.2022.103774}.

\bibitem[Blücher et~al.(2024)Blücher, Vielhaben, and
  Strodthoff]{blucher2024flipping}
S.~Blücher, J.~Vielhaben, and N.~Strodthoff.
\newblock Decoupling pixel flipping and occlusion strategy for consistent xai
  benchmarks.
\newblock \emph{Transactions on Machine Learning Research}, 2024.
\newblock ISSN 2835-8856.

\bibitem[Bommasani et~al.(2021)Bommasani, Hudson, Adeli, Altman, Arora, von
  Arx, Bernstein, Bohg, Bosselut, et~al.]{Bommasani2021FoundationModels}
R.~Bommasani, D.~A. Hudson, E.~Adeli, R.~Altman, S.~Arora, S.~von Arx, M.~S.
  Bernstein, J.~Bohg, A.~Bosselut, et~al.
\newblock On the opportunities and risks of foundation models.
\newblock \emph{ArXiv}, 2021.

\bibitem[Bonneau et~al.(2024)Bonneau, Lederer, Templeton, Rosenberger,
  M{\"u}ller, and Clementi]{bonneau2024peering}
K.~Bonneau, J.~Lederer, C.~Templeton, D.~Rosenberger, K.-R. M{\"u}ller, and
  C.~Clementi.
\newblock Peering inside the black box: Learning the relevance of many-body
  functions in neural network potentials.
\newblock \emph{arXiv preprint arXiv:2407.04526}, 2024.

\bibitem[Bronstein et~al.(2017)Bronstein, Bruna, {L}e{C}un, Szlam, and
  Vandergheynst]{Bronstein17GeomDL}
M.~M. Bronstein, J.~Bruna, Y.~{L}e{C}un, A.~Szlam, and P.~Vandergheynst.
\newblock Geometric deep learning: Going beyond {E}uclidean data.
\newblock \emph{{IEEE} Signal Process. Mag.}, 34\penalty0 (4):\penalty0 18--42,
  2017.

\bibitem[Chen et~al.(2018)Chen, Rubanova, Bettencourt, and
  Duvenaud]{chen2018neuralODE}
R.~T.~Q. Chen, Y.~Rubanova, J.~Bettencourt, and D.~K. Duvenaud.
\newblock Neural ordinary differential equations.
\newblock In \emph{Advances in Neural Information Processing Systems},
  volume~31. Curran Associates, Inc., 2018.

\bibitem[Cho et~al.(2020)Cho, Lee, and Choi]{cho2020lrpInterNet}
H.~Cho, E.~K. Lee, and I.~S. Choi.
\newblock Layer-wise relevance propagation of interactionnet explains
  protein--ligand interactions at the atom level.
\newblock \emph{Scientific Reports}, 10\penalty0 (1):\penalty0 21155, 2020.

\bibitem[Cho et~al.(2014)Cho, van Merri{\"e}nboer, Gulcehre, Bahdanau,
  Bougares, Schwenk, and Bengio]{cho2014gru}
K.~Cho, B.~van Merri{\"e}nboer, C.~Gulcehre, D.~Bahdanau, F.~Bougares,
  H.~Schwenk, and Y.~Bengio.
\newblock Learning phrase representations using {RNN} encoder{--}decoder for
  statistical machine translation.
\newblock In \emph{Proceedings of the 2014 Conference on Empirical Methods in
  Natural Language Processing ({EMNLP})}, pages 1724--1734, Doha, Qatar, Oct.
  2014. Association for Computational Linguistics.

\bibitem[Chormai et~al.(2024)Chormai, Herrmann, M\"uller, and
  Montavon]{chormai24relsubspace}
P.~Chormai, J.~Herrmann, K.-R. M\"uller, and G.~Montavon.
\newblock Disentangled explanations of neural network predictions by finding
  relevant subspaces.
\newblock \emph{{IEEE} Trans. Pattern Anal. Mach. Intell.}, 2024.
\newblock \doi{10.1109/TPAMI.2024.3388275}.

\bibitem[Ciravegna et~al.(2023)Ciravegna, Barbiero, Giannini, Gori, Liò,
  Maggini, and Melacci]{CIRAVEGNA2023103822}
G.~Ciravegna, P.~Barbiero, F.~Giannini, M.~Gori, P.~Liò, M.~Maggini, and
  S.~Melacci.
\newblock Logic explained networks.
\newblock \emph{Artificial Intelligence}, 314:\penalty0 103822, 2023.
\newblock ISSN 0004-3702.
\newblock \doi{https://doi.org/10.1016/j.artint.2022.103822}.

\bibitem[Collins and Raghavachari(2023)]{collins2023xaiMolFrag}
E.~M. Collins and K.~Raghavachari.
\newblock Interpretable graph-network-based machine learning models via
  molecular fragmentation.
\newblock \emph{Journal of Chemical Theory and Computation}, 19\penalty0
  (10):\penalty0 2804--2810, 2023.
\newblock \doi{10.1021/acs.jctc.2c01308}.
\newblock PMID: 37134275.

\bibitem[Commission(2021)]{EU2021aiact}
E.~Commission.
\newblock Proposal for a regulation laying down harmonisedrules on artificial
  intelligence and amending certain union legislative acts.
\newblock 2021.
\newblock
  \url{https://eur-lex.europa.eu/legal-content/EN/TXT/?uri=celex%3A52021PC0206}.

\bibitem[d'Avila Garcez et~al.(2015)d'Avila Garcez, Besold, De~Raedt,
  F{\"o}ldi{\'a}k, Hitzler, Icard, Kiihnberger, Lamb, Miikkulainen, and
  Silver]{garcez2015neural_symb}
A.~d'Avila Garcez, T.~R. Besold, L.~De~Raedt, P.~F{\"o}ldi{\'a}k, P.~Hitzler,
  T.~Icard, K.-U. Kiihnberger, L.~C. Lamb, R.~Miikkulainen, and D.~L. Silver.
\newblock Neural-symbolic learning and reasoning : Contributions and
  challenges.
\newblock In \emph{AAAI Spring Symposium - Knowledge Representation and
  Reasoning: Integrating Symbolic and Neural Approaches, Stanford University,
  Palo Alto, CA, USA, March 23-25, 2015}, volume SS-15-03, pages 18--21. AAAI
  Press, 2015.

\bibitem[d'Avila Garcez et~al.(2002)d'Avila Garcez, Broda, and
  Gabbay]{artur2002neural_symb}
A.~S. d'Avila Garcez, K.~Broda, and D.~M. Gabbay.
\newblock \emph{Neural-symbolic learning systems - foundations and
  applications}.
\newblock Perspectives in neural computing. Springer, 2002.

\bibitem[Debnath et~al.(1991)Debnath, Lopez~de Compadre, Debnath, Shusterman,
  and Hansch]{debnath1991structure}
A.~K. Debnath, R.~L. Lopez~de Compadre, G.~Debnath, A.~J. Shusterman, and
  C.~Hansch.
\newblock Structure-activity relationship of mutagenic aromatic and
  heteroaromatic nitro compounds. correlation with molecular orbital energies
  and hydrophobicity.
\newblock \emph{Journal of Medicinal Chemistry}, 34\penalty0 (2):\penalty0
  786--797, 1991.
\newblock \doi{10.1021/jm00106a046}.

\bibitem[Delfosse et~al.(2023)Delfosse, Shindo, Dhami, and
  Kersting]{Kersting2023XSymbAI}
Q.~Delfosse, H.~Shindo, D.~S. Dhami, and K.~Kersting.
\newblock Interpretable and explainable logical policies via neurally guided
  symbolic abstraction.
\newblock In \emph{Advances in Neural Information Processing Systems 36: Annual
  Conference on Neural Information Processing Systems 2023, NeurIPS 2023, New
  Orleans, LA, USA, December 10 - 16, 2023}, 2023.

\bibitem[Deng et~al.(2009)Deng, Dong, Socher, Li, Li, and
  Fei-Fei]{deng2009imagenet}
J.~Deng, W.~Dong, R.~Socher, L.-J. Li, K.~Li, and L.~Fei-Fei.
\newblock Imagenet: A large-scale hierarchical image database.
\newblock In \emph{2009 IEEE Conference on Computer Vision and Pattern
  Recognition}, pages 248--255, 2009.

\bibitem[Devlin et~al.(2019)Devlin, Chang, Lee, and Toutanova]{devlin2019bert}
J.~Devlin, M.-W. Chang, K.~Lee, and K.~Toutanova.
\newblock {BERT}: Pre-training of deep bidirectional transformers for language
  understanding.
\newblock In \emph{Proceedings of the 2019 Conference of the North {A}merican
  Chapter of the Association for Computational Linguistics: Human Language
  Technologies, Volume 1 (Long and Short Papers)}, pages 4171--4186,
  Minneapolis, Minnesota, June 2019. Association for Computational Linguistics.

\bibitem[DeYoung et~al.(2020)DeYoung, Jain, Rajani, Lehman, Xiong, Socher, and
  Wallace]{movie_reviews2}
J.~DeYoung, S.~Jain, N.~F. Rajani, E.~Lehman, C.~Xiong, R.~Socher, and B.~C.
  Wallace.
\newblock {ERASER}: {A} benchmark to evaluate rationalized {NLP} models.
\newblock In \emph{Proceedings of the 58th Annual Meeting of the Association
  for Computational Linguistics}, pages 4443--4458, Online, July 2020.
  Association for Computational Linguistics.
\newblock \doi{10.18653/v1/2020.acl-main.408}.

\bibitem[Dhamdhere et~al.(2020)Dhamdhere, Agarwal, and
  Sundararajan]{Kedar2020shapTaylor}
K.~Dhamdhere, A.~Agarwal, and M.~Sundararajan.
\newblock The shapley taylor interaction index.
\newblock In \emph{Proceedings of the 37th International Conference on Machine
  Learning}, ICML'20. JMLR.org, 2020.

\bibitem[Dosovitskiy et~al.(2021)Dosovitskiy, Beyer, Kolesnikov, Weissenborn,
  Zhai, Unterthiner, Dehghani, Minderer, Heigold, Gelly, Uszkoreit, and
  Houlsby]{dosovitskiy2021ViT}
A.~Dosovitskiy, L.~Beyer, A.~Kolesnikov, D.~Weissenborn, X.~Zhai,
  T.~Unterthiner, M.~Dehghani, M.~Minderer, G.~Heigold, S.~Gelly, J.~Uszkoreit,
  and N.~Houlsby.
\newblock An image is worth 16x16 words: Transformers for image recognition at
  scale.
\newblock In \emph{International Conference on Learning Representations}, 2021.

\bibitem[Dubey et~al.(2022)Dubey, Singh, and Chaudhuri]{DUBEY2022activation}
S.~R. Dubey, S.~K. Singh, and B.~B. Chaudhuri.
\newblock Activation functions in deep learning: A comprehensive survey and
  benchmark.
\newblock \emph{Neurocomputing}, 503:\penalty0 92--108, 2022.
\newblock ISSN 0925-2312.

\bibitem[Díaz-Rodríguez et~al.(2022)Díaz-Rodríguez, Lamas, Sanchez,
  Franchi, Donadello, Tabik, Filliat, Cruz, Montes, and
  Herrera]{DIAZRODRIGUEZ2022XNESYL}
N.~Díaz-Rodríguez, A.~Lamas, J.~Sanchez, G.~Franchi, I.~Donadello, S.~Tabik,
  D.~Filliat, P.~Cruz, R.~Montes, and F.~Herrera.
\newblock Explainable neural-symbolic learning (x-nesyl) methodology to fuse
  deep learning representations with expert knowledge graphs: The monumai
  cultural heritage use case.
\newblock \emph{Information Fusion}, 79:\penalty0 58--83, 2022.
\newblock ISSN 1566-2535.

\bibitem[Eberle et~al.(2022)Eberle, B{\"u}ttner, Kräutli, M{\"u}ller,
  Valleriani, and Montavon]{eberle22bilrp}
O.~Eberle, J.~B{\"u}ttner, F.~Kräutli, K.-R. M{\"u}ller, M.~Valleriani, and
  G.~Montavon.
\newblock Building and interpreting deep similarity models.
\newblock \emph{IEEE Transactions on Pattern Analysis and Machine
  Intelligence}, 44\penalty0 (3):\penalty0 1149--1161, 2022.
\newblock \doi{10.1109/TPAMI.2020.3020738}.

\bibitem[Eberle et~al.(2023)Eberle, B{\"u}ttner, El-Hajj, Montavon, M{\"u}ller,
  and Valleriani]{eberle2023insightful}
O.~Eberle, J.~B{\"u}ttner, H.~El-Hajj, G.~Montavon, K.-R. M{\"u}ller, and
  M.~Valleriani.
\newblock Insightful analysis of historical sources at scales beyond human
  capabilities using unsupervised machine learning and xai.
\newblock \emph{arXiv preprint arXiv:2310.09091}, 2023.

\bibitem[ELsayed et~al.(2023)ELsayed, ELSayed, and
  Abdou]{elsayed2023FER_masked}
Y.~ELsayed, A.~ELSayed, and M.~A. Abdou.
\newblock An automatic improved facial expression recognition for masked faces.
\newblock \emph{Neural Computing and Applications}, 35\penalty0 (20):\penalty0
  14963--14972, 2023.

\bibitem[Faber et~al.(2021)Faber, K.~Moghaddam, and
  Wattenhofer]{faber21evalGNNXAI}
L.~Faber, A.~K.~Moghaddam, and R.~Wattenhofer.
\newblock When comparing to ground truth is wrong: On evaluating gnn
  explanation methods.
\newblock In \emph{Proceedings of the 27th ACM SIGKDD Conference on Knowledge
  Discovery \& Data Mining}, KDD '21, page 332–341, New York, NY, USA, 2021.
  Association for Computing Machinery.
\newblock ISBN 9781450383325.

\bibitem[Frank et~al.(2024)Frank, Unke, M{\"u}ller, and
  Chmiela]{frank2024euclidean}
J.~T. Frank, O.~T. Unke, K.-R. M{\"u}ller, and S.~Chmiela.
\newblock A euclidean transformer for fast and stable machine learned force
  fields.
\newblock \emph{Nature Communications}, 15\penalty0 (1):\penalty0 6539, 2024.

\bibitem[Frank et~al.(2022)Frank, Unke, and M{\"u}ller]{frank2022so3krates}
T.~Frank, O.~Unke, and K.-R. M{\"u}ller.
\newblock So3krates: Equivariant attention for interactions on arbitrary
  length-scales in molecular systems.
\newblock \emph{Advances in Neural Information Processing Systems},
  35:\penalty0 29400--29413, 2022.

\bibitem[Fu et~al.(2023)Fu, Wu, Wang, Xie, Keten, Gomez-Bombarelli, and
  Jaakkola]{fu2022forces}
X.~Fu, Z.~Wu, W.~Wang, T.~Xie, S.~Keten, R.~Gomez-Bombarelli, and T.~S.
  Jaakkola.
\newblock Forces are not enough: Benchmark and critical evaluation for machine
  learning force fields with molecular simulations.
\newblock \emph{Transactions on Machine Learning Research}, 2023.
\newblock Survey Certification.

\bibitem[Fuchs et~al.(2020)Fuchs, Worrall, Fischer, and
  Welling]{NEURIPS2020_15231a7c}
F.~Fuchs, D.~Worrall, V.~Fischer, and M.~Welling.
\newblock Se(3)-transformers: 3d roto-translation equivariant attention
  networks.
\newblock In H.~Larochelle, M.~Ranzato, R.~Hadsell, M.~Balcan, and H.~Lin,
  editors, \emph{Advances in Neural Information Processing Systems}, volume~33,
  pages 1970--1981. Curran Associates, Inc., 2020.

\bibitem[Fujimoto et~al.(2006)Fujimoto, Kojadinovic, and
  Marichal]{FUJIMOTO2006AxiomInterIndex}
K.~Fujimoto, I.~Kojadinovic, and J.-L. Marichal.
\newblock Axiomatic characterizations of probabilistic and
  cardinal-probabilistic interaction indices.
\newblock \emph{Games and Economic Behavior}, 55\penalty0 (1):\penalty0 72--99,
  2006.
\newblock ISSN 0899-8256.

\bibitem[Gasteiger et~al.(2020)Gasteiger, Gro{\ss}, and
  G{\"u}nnemann]{gasteiger_dimenet_2020}
J.~Gasteiger, J.~Gro{\ss}, and S.~G{\"u}nnemann.
\newblock Directional message passing for molecular graphs.
\newblock In \emph{International Conference on Learning Representations
  (ICLR)}, 2020.

\bibitem[Gilmer et~al.(2017)Gilmer, Schoenholz, Riley, Vinyals, and
  Dahl]{gilmer2017neural}
J.~Gilmer, S.~S. Schoenholz, P.~F. Riley, O.~Vinyals, and G.~E. Dahl.
\newblock Neural message passing for quantum chemistry.
\newblock In \emph{International conference on machine learning}, pages
  1263--1272. PMLR, 2017.

\bibitem[Goodfellow et~al.(2013)Goodfellow, Erhan, Carrier, Courville, Mirza,
  Hamner, Cukierski, Tang, Thaler, et~al.]{goodfellow2013fer2013}
I.~J. Goodfellow, D.~Erhan, P.~L. Carrier, A.~Courville, M.~Mirza, B.~Hamner,
  W.~Cukierski, Y.~Tang, D.~Thaler, et~al.
\newblock Challenges in representation learning: A report on three machine
  learning contests.
\newblock In \emph{Neural Information Processing}, pages 117--124, Berlin,
  Heidelberg, 2013. Springer Berlin Heidelberg.

\bibitem[Grabisch and Roubens(1999)]{garbis1999shapInter}
M.~Grabisch and M.~Roubens.
\newblock An axiomatic approach to the concept of interaction among players in
  cooperative games.
\newblock \emph{International Journal of Game Theory}, 28\penalty0
  (4):\penalty0 547--565, 1999.

\bibitem[Grabisch et~al.(2000)Grabisch, Marichal, and
  Roubens]{garbish2000EquiRepSet}
M.~Grabisch, J.-L. Marichal, and M.~Roubens.
\newblock Equivalent representations of set functions.
\newblock \emph{Math. Oper. Res.}, 25:\penalty0 157--178, 2000.

\bibitem[Gugler and Reiher(2022)]{gugler2022}
S.~Gugler and M.~Reiher.
\newblock Quantum chemical roots of machine-learning molecular similarity
  descriptors.
\newblock \emph{Journal of Chemical Theory and Computation}, 18\penalty0
  (11):\penalty0 6670--6689, 2022.
\newblock \doi{10.1021/acs.jctc.2c00718}.

\bibitem[Haghighatlari et~al.(2022)Haghighatlari, Li, Guan, Zhang, Das, Stein,
  Heidar-Zadeh, Liu, Head-Gordon, Bertels, et~al.]{haghighatlari2022newtonnet}
M.~Haghighatlari, J.~Li, X.~Guan, O.~Zhang, A.~Das, C.~J. Stein,
  F.~Heidar-Zadeh, M.~Liu, M.~Head-Gordon, L.~Bertels, et~al.
\newblock Newtonnet: A newtonian message passing network for deep learning of
  interatomic potentials and forces.
\newblock \emph{Digital Discovery}, 1\penalty0 (3):\penalty0 333--343, 2022.

\bibitem[Hailesilassie(2016)]{hailesilassie2016rule_extract_xai_review}
T.~Hailesilassie.
\newblock Rule extraction algorithm for deep neural networks: A review.
\newblock \emph{(IJCSIS) International Journal of Computer Science and
  Information Security}, Vol. 14:\penalty0 376-- 381, 2016.

\bibitem[Han et~al.(2018)Han, Zhang, Car, and E]{han_deep_2017}
J.~Han, L.~Zhang, R.~Car, and W.~E.
\newblock Deep potential: A general representation of a many-body potential
  energy surface.
\newblock \emph{Communications in Computational Physics}, 23\penalty0
  (3):\penalty0 629--639, 2018.

\bibitem[Harsanyi(1963)]{harsanyi1963dividedCoprGame}
J.~C. Harsanyi.
\newblock A simplified bargaining model for the n-person cooperative game.
\newblock \emph{International Economic Review}, 4\penalty0 (2):\penalty0
  194--220, 1963.
\newblock ISSN 00206598, 14682354.

\bibitem[Hense et~al.(2024)Hense, Idaji, Eberle, Schnake, Dippel, Ciernik,
  Buchstab, Mock, Klauschen, and Müller]{hense2024xmil}
J.~Hense, M.~J. Idaji, O.~Eberle, T.~Schnake, J.~Dippel, L.~Ciernik,
  O.~Buchstab, A.~Mock, F.~Klauschen, and K.-R. Müller.
\newblock x{MIL}: Insightful explanations for multiple instance learning in
  histopathology.
\newblock \emph{arXiv}, 2406.04280, 2024.

\bibitem[Hochreiter and Schmidhuber(1997)]{hochreiter1997long}
S.~Hochreiter and J.~Schmidhuber.
\newblock Long short-term memory.
\newblock \emph{Neural computation}, 9\penalty0 (8):\penalty0 1735--1780, 1997.

\bibitem[Iqbal and Qureshi(2022)]{IQBAL2022TextGenSurvey}
T.~Iqbal and S.~Qureshi.
\newblock The survey: Text generation models in deep learning.
\newblock \emph{Journal of King Saud University - Computer and Information
  Sciences}, 34\penalty0 (6, Part A):\penalty0 2515--2528, 2022.
\newblock ISSN 1319-1578.
\newblock \doi{https://doi.org/10.1016/j.jksuci.2020.04.001}.

\bibitem[Jafari et~al.(2024)Jafari, Montavon, Müller, and
  Eberle]{jafari2024mambalrp}
F.~R. Jafari, G.~Montavon, K.-R. Müller, and O.~Eberle.
\newblock Mamba{LRP}: Explaining selective state space sequence models.
\newblock \emph{arXiv}, 2406.07592, 2024.

\bibitem[Janizek et~al.(2021)Janizek, Sturmfels, and Lee]{janizek2021IntHess}
J.~D. Janizek, P.~Sturmfels, and S.-I. Lee.
\newblock Explaining explanations: Axiomatic feature interactions for deep
  networks.
\newblock \emph{Journal of Machine Learning Research}, 22\penalty0
  (104):\penalty0 1--54, 2021.

\bibitem[Jim{\'e}nez-Luna et~al.(2020)Jim{\'e}nez-Luna, Grisoni, and
  Schneider]{jose20drugdisc}
J.~Jim{\'e}nez-Luna, F.~Grisoni, and G.~Schneider.
\newblock Drug discovery with explainable artificial intelligence.
\newblock \emph{Nature Machine Intelligence}, 2\penalty0 (10):\penalty0
  573--584, 2020.
\newblock \doi{10.1038/s42256-020-00236-4}.

\bibitem[Jones(1924)]{Jones1924}
J.~E. Jones.
\newblock On the determination of molecular fields. ii. from the equation of
  state of a gas.
\newblock \emph{Proceedings of the Royal Society of London. Series A,
  Containing Papers of a Mathematical and Physical Character}, 106\penalty0
  (738):\penalty0 pp. 463--477, 1924.

\bibitem[Jurafsky and Martin(2009)]{Jurafsky2009InfoExtract}
D.~Jurafsky and J.~H. Martin.
\newblock \emph{Speech and language processing : an introduction to natural
  language processing, computational linguistics, and speech recognition}.
\newblock Pearson Prentice Hall, 2009.

\bibitem[Kahouli et~al.(2024)Kahouli, Hessmann, M\"uller, Nakajima, Gugler, and
  Gebauer]{kahouli2024}
K.~Kahouli, S.~S.~P. Hessmann, K.-R. M\"uller, S.~Nakajima, S.~Gugler, and
  N.~W.~A. Gebauer.
\newblock Molecular relaxation by reverse diffusion with time step prediction.
\newblock \emph{Machine Learning: Science and Technology}, 2024.
\newblock ISSN 2632-2153.
\newblock \doi{10.1088/2632-2153/ad652c}.

\bibitem[Kazius et~al.(2005)Kazius, McGuire, and Bursi]{Kazius2005DerivMutag}
J.~Kazius, R.~McGuire, and R.~Bursi.
\newblock Derivation and validation of toxicophores for mutagenicity
  prediction.
\newblock \emph{Journal of Medicinal Chemistry}, 48\penalty0 (1):\penalty0
  312--320, 2005.
\newblock \doi{10.1021/jm040835a}.
\newblock PMID: 15634026.

\bibitem[Keyl et~al.(2022)Keyl, Bockmayr, Heim, Dernbach, Montavon, M{\"u}ller,
  and Klauschen]{keyl2022patient}
P.~Keyl, M.~Bockmayr, D.~Heim, G.~Dernbach, G.~Montavon, K.-R. M{\"u}ller, and
  F.~Klauschen.
\newblock Patient-level proteomic network prediction by explainable artificial
  intelligence.
\newblock \emph{NPJ Precision Oncology}, 6\penalty0 (1):\penalty0 35, 2022.

\bibitem[Keyl et~al.(2023)Keyl, Bischoff, Dernbach, Bockmayr, Fritz, Horst,
  Bl{\"u}thgen, Montavon, M{\"u}ller, and Klauschen]{keyl2023single}
P.~Keyl, P.~Bischoff, G.~Dernbach, M.~Bockmayr, R.~Fritz, D.~Horst,
  N.~Bl{\"u}thgen, G.~Montavon, K.-R. M{\"u}ller, and F.~Klauschen.
\newblock Single-cell gene regulatory network prediction by explainable ai.
\newblock \emph{Nucleic Acids Research}, 51\penalty0 (4):\penalty0 e20--e20,
  2023.

\bibitem[Kingma and Ba(2015)]{kingma2014adam}
D.~P. Kingma and J.~Ba.
\newblock Adam: {A} method for stochastic optimization.
\newblock In \emph{3rd International Conference on Learning Representations,
  {ICLR} 2015, San Diego, CA, USA, May 7-9, 2015, Conference Track
  Proceedings}, 2015.

\bibitem[Kipf and Welling(2017)]{kipf2017semisupervised}
T.~N. Kipf and M.~Welling.
\newblock Semi-supervised classification with graph convolutional networks.
\newblock In \emph{International Conference on Learning Representations}, 2017.

\bibitem[Klauschen et~al.(2024)Klauschen, Dippel, Keyl, Jurmeister, Bockmayr,
  Mock, Buchstab, Alber, Ruff, Montavon, and M\"{u}ller]{klauschen-pathology24}
F.~Klauschen, J.~Dippel, P.~Keyl, P.~Jurmeister, M.~Bockmayr, A.~Mock,
  O.~Buchstab, M.~Alber, L.~Ruff, G.~Montavon, and K.-R. M\"{u}ller.
\newblock Toward explainable artificial intelligence for precision pathology.
\newblock \emph{Annual Review of Pathology: Mechanisms of Disease}, 19\penalty0
  (1):\penalty0 541–570, Jan. 2024.

\bibitem[Klicpera et~al.(2020)Klicpera, Gro{\ss}, and
  G{\"u}nnemann]{klicpera2020directional}
J.~Klicpera, J.~Gro{\ss}, and S.~G{\"u}nnemann.
\newblock Directional message passing for molecular graphs.
\newblock In \emph{International Conference on Learning Representations
  (ICLR)}, 2020.

\bibitem[Ko(2018)]{ko2018fer_review}
B.~C. Ko.
\newblock A brief review of facial emotion recognition based on visual
  information.
\newblock \emph{Sensors}, 18\penalty0 (2), 2018.
\newblock ISSN 1424-8220.

\bibitem[Krizhevsky et~al.(2012)Krizhevsky, Sutskever, and
  Hinton]{krizhevsky2012alexnet}
A.~Krizhevsky, I.~Sutskever, and G.~E. Hinton.
\newblock Imagenet classification with deep convolutional neural networks.
\newblock In \emph{Advances in Neural Information Processing Systems},
  volume~25. Curran Associates, Inc., 2012.

\bibitem[Lampert(2009)]{lampert2009kernelCV}
C.~H. Lampert.
\newblock \emph{Kernel Methods in Computer Vision}, volume Computer Graphics
  and computer vision.
\newblock 2009.
\newblock \doi{10.1561/0600000027}.

\bibitem[{L}e{C}un and Bengio(1995)]{lecun95convolutional}
Y.~{L}e{C}un and Y.~Bengio.
\newblock \emph{Convolutional Networks for Images, Speech and Time Series},
  pages 255--258.
\newblock The MIT Press, 1995.

\bibitem[{L}e{C}un et~al.(1998){L}e{C}un, Bottou, Bengio, and
  Haffner]{lecun1998doc_rec}
Y.~{L}e{C}un, L.~Bottou, Y.~Bengio, and P.~Haffner.
\newblock Gradient-based learning applied to document recognition.
\newblock \emph{Proceedings of the IEEE}, 86\penalty0 (11):\penalty0
  2278--2324, 1998.

\bibitem[Lennard-Jones(1931)]{Lennard_Jones_1931}
J.~E. Lennard-Jones.
\newblock Cohesion.
\newblock \emph{Proceedings of the Physical Society}, 43\penalty0 (5):\penalty0
  461, sep 1931.
\newblock \doi{10.1088/0959-5309/43/5/301}.

\bibitem[Loshchilov and Hutter(2019)]{loshchilov2017decoupled}
I.~Loshchilov and F.~Hutter.
\newblock Decoupled weight decay regularization.
\newblock In \emph{International Conference on Learning Representations}, 2019.

\bibitem[Lundberg and Lee(2017)]{lund17unified}
S.~M. Lundberg and S.-I. Lee.
\newblock A unified approach to interpreting model predictions.
\newblock In \emph{Proceedings of the 31st International Conference on Neural
  Information Processing Systems}, NIPS'17, page 4768–4777, Red Hook, NY,
  USA, 2017. Curran Associates Inc.

\bibitem[Lundberg et~al.(2020)Lundberg, Erion, Chen, DeGrave, Prutkin, Nair,
  Katz, Himmelfarb, Bansal, and Lee]{Lund2020shapInterVal}
S.~M. Lundberg, G.~Erion, H.~Chen, A.~DeGrave, J.~M. Prutkin, B.~Nair, R.~Katz,
  J.~Himmelfarb, N.~Bansal, and S.-I. Lee.
\newblock From local explanations to global understanding with explainable ai
  for trees.
\newblock \emph{Nature Machine Intelligence}, 2\penalty0 (1):\penalty0 56--67,
  2020.

\bibitem[Luo et~al.(2020)Luo, Cheng, Xu, Yu, Zong, Chen, and
  Zhang]{luo2020parameterized}
D.~Luo, W.~Cheng, D.~Xu, W.~Yu, B.~Zong, H.~Chen, and X.~Zhang.
\newblock Parameterized explainer for graph neural network.
\newblock \emph{Advances in Neural Information Processing Systems}, 33, 2020.

\bibitem[Maas et~al.(2011)Maas, Daly, Pham, Huang, Ng, and Potts]{imdb}
A.~L. Maas, R.~E. Daly, P.~T. Pham, D.~Huang, A.~Y. Ng, and C.~Potts.
\newblock Learning word vectors for sentiment analysis.
\newblock In \emph{Proceedings of the 49th Annual Meeting of the Association
  for Computational Linguistics: Human Language Technologies}, pages 142--150,
  Portland, Oregon, USA, June 2011. Association for Computational Linguistics.

\bibitem[Magherini et~al.(2022)Magherini, Mussi, Servi, and
  Volpe]{MAGHERINI2022covidMasks}
R.~Magherini, E.~Mussi, M.~Servi, and Y.~Volpe.
\newblock Emotion recognition in the times of covid19: Coping with face masks.
\newblock \emph{Intelligent Systems with Applications}, 15:\penalty0 200094,
  2022.
\newblock ISSN 2667-3053.

\bibitem[McCloskey et~al.(2019)McCloskey, Taly, Monti, Brenner, and
  Colwell]{mcclo19usingatt}
K.~McCloskey, A.~Taly, F.~Monti, M.~P. Brenner, and L.~J. Colwell.
\newblock Using attribution to decode binding mechanism in neural network
  models for chemistry.
\newblock \emph{Proceedings of the National Academy of Sciences}, 116\penalty0
  (24):\penalty0 11624--11629, 2019.

\bibitem[Montavon et~al.(2017)Montavon, Bach, Binder, Samek, and
  M{\"u}ller]{montavon-pr17}
G.~Montavon, S.~Bach, A.~Binder, W.~Samek, and K.-R. M{\"u}ller.
\newblock Explaining nonlinear classification decisions with deep taylor
  decomposition.
\newblock \emph{Pattern Recognition}, 65:\penalty0 211--222, 2017.

\bibitem[Montavon et~al.(2018)Montavon, Samek, and
  M{\"u}ller]{montavon2018methods}
G.~Montavon, W.~Samek, and K.-R. M{\"u}ller.
\newblock Methods for interpreting and understanding deep neural networks.
\newblock \emph{Digital signal processing}, 73:\penalty0 1--15, 2018.

\bibitem[Musaelian et~al.(2023)Musaelian, Batzner, Johansson, Sun, Owen,
  Kornbluth, and Kozinsky]{musaelian2023learning}
A.~Musaelian, S.~Batzner, A.~Johansson, L.~Sun, C.~J. Owen, M.~Kornbluth, and
  B.~Kozinsky.
\newblock Learning local equivariant representations for large-scale atomistic
  dynamics.
\newblock \emph{Nature Communications}, 14\penalty0 (1):\penalty0 579, 2023.

\bibitem[Ouyang et~al.(2014)Ouyang, Cvitkovic, and
  Bettens]{ouyang2014troubleMBE}
J.~F. Ouyang, M.~W. Cvitkovic, and R.~P.~A. Bettens.
\newblock Trouble with the many-body expansion.
\newblock \emph{Journal of Chemical Theory and Computation}, 10\penalty0
  (9):\penalty0 3699--3707, 2014.
\newblock \doi{10.1021/ct500396b}.

\bibitem[Panigutti et~al.(2023)Panigutti, Hamon, Hupont, Fernandez~Llorca,
  Fano~Yela, Junklewitz, Scalzo, Mazzini, Sanchez, Soler~Garrido, and
  Gomez]{Panigutti2023XAIinAIAct}
C.~Panigutti, R.~Hamon, I.~Hupont, D.~Fernandez~Llorca, D.~Fano~Yela,
  H.~Junklewitz, S.~Scalzo, G.~Mazzini, I.~Sanchez, J.~Soler~Garrido, and
  E.~Gomez.
\newblock The role of explainable ai in the context of the ai act.
\newblock In \emph{Proceedings of the 2023 ACM Conference on Fairness,
  Accountability, and Transparency}, FAccT '23, page 1139–1150, New York, NY,
  USA, 2023. Association for Computing Machinery.

\bibitem[Pope et~al.(2019)Pope, Kolouri, Rostami, Martin, and
  Hoffmann]{pope2019gcnxai}
P.~E. Pope, S.~Kolouri, M.~Rostami, C.~E. Martin, and H.~Hoffmann.
\newblock Explainability methods for graph convolutional neural networks.
\newblock In \emph{2019 IEEE/CVF Conference on Computer Vision and Pattern
  Recognition (CVPR)}, pages 10764--10773, 2019.

\bibitem[Radford et~al.(2018)Radford, Narasimhan, Salimans, and
  Sutskever]{radford2018gpt}
A.~Radford, K.~Narasimhan, T.~Salimans, and I.~Sutskever.
\newblock Improving language understanding by generative pre-training.
\newblock \emph{Technical report, OpenAI}, 2018.

\bibitem[Ren et~al.(2023{\natexlab{a}})Ren, Li, Chen, Deng, and
  Zhang]{concept_interactions_3}
J.~Ren, M.~Li, Q.~Chen, H.~Deng, and Q.~Zhang.
\newblock Defining and quantifying the emergence of sparse concepts in dnns.
\newblock In \emph{Proceedings of the IEEE/CVF Conference on Computer Vision
  and Pattern Recognition (CVPR)}, pages 20280--20289, June 2023{\natexlab{a}}.

\bibitem[Ren et~al.(2023{\natexlab{b}})Ren, Zhou, Chen, and
  Zhang]{concept_interactions_1}
J.~Ren, Z.~Zhou, Q.~Chen, and Q.~Zhang.
\newblock Can we faithfully represent absence states to compute shapley values
  on a {DNN}?
\newblock In \emph{The Eleventh International Conference on Learning
  Representations}, 2023{\natexlab{b}}.
\newblock URL \url{https://openreview.net/forum?id=YV8tP7bW6Kt}.

\bibitem[Ribeiro et~al.(2016)Ribeiro, Singh, and Guestrin]{ribeiro2016trust}
M.~Ribeiro, S.~Singh, and C.~Guestrin.
\newblock {``}why should {I} trust you?{''}: Explaining the predictions of any
  classifier.
\newblock In J.~DeNero, M.~Finlayson, and S.~Reddy, editors, \emph{Proceedings
  of the 2016 Conference of the North {A}merican Chapter of the Association for
  Computational Linguistics: Demonstrations}, pages 97--101, San Diego,
  California, 2016. Association for Computational Linguistics.

\bibitem[Rupp et~al.(2012)Rupp, Tkatchenko, M{\"u}ller, and
  Von~Lilienfeld]{rupp2012fast}
M.~Rupp, A.~Tkatchenko, K.-R. M{\"u}ller, and O.~A. Von~Lilienfeld.
\newblock Fast and accurate modeling of molecular atomization energies with
  machine learning.
\newblock \emph{Physical review letters}, 108\penalty0 (5):\penalty0 058301,
  2012.

\bibitem[Ryser(1963)]{ryser1963combinatorial}
H.~J. Ryser.
\newblock \emph{Combinatorial Mathematics}.
\newblock The Carus Mathematical Monographs. American Mathematical Society,
  1963.

\bibitem[Samek et~al.(2016)Samek, Binder, Montavon, Lapuschkin, and
  M{\"u}ller]{samek2016evaluating}
W.~Samek, A.~Binder, G.~Montavon, S.~Lapuschkin, and K.-R. M{\"u}ller.
\newblock Evaluating the visualization of what a deep neural network has
  learned.
\newblock \emph{IEEE Transactions on Neural Networks and Learning Systems},
  28\penalty0 (11):\penalty0 2660--2673, 2016.

\bibitem[Samek et~al.(2021)Samek, Montavon, Lapuschkin, Anders, and
  Müller]{samek2021xaireview}
W.~Samek, G.~Montavon, S.~Lapuschkin, C.~J. Anders, and K.-R. Müller.
\newblock Explaining deep neural networks and beyond: A review of methods and
  applications.
\newblock \emph{Proceedings of the IEEE}, 109\penalty0 (3):\penalty0 247--278,
  2021.

\bibitem[Scarselli et~al.(2009)Scarselli, Gori, Tsoi, Hagenbuchner, and
  Monfardini]{Scarselli2009GNN}
F.~Scarselli, M.~Gori, A.~C. Tsoi, M.~Hagenbuchner, and G.~Monfardini.
\newblock The graph neural network model.
\newblock \emph{{IEEE} Trans. Neural Networks}, 20\penalty0 (1):\penalty0
  61--80, 2009.

\bibitem[Schnake et~al.(2022)Schnake, Eberle, Lederer, Nakajima, Sch{\"u}tt,
  M{\"u}ller, and Montavon]{schnake22gnnlrp}
T.~Schnake, O.~Eberle, J.~Lederer, S.~Nakajima, K.~T. Sch{\"u}tt, K.-R.
  M{\"u}ller, and G.~Montavon.
\newblock Higher-order explanations of graph neural networks via relevant
  walks.
\newblock \emph{IEEE Transactions on Pattern Analysis and Machine
  Intelligence}, 44\penalty0 (11):\penalty0 7581--7596, 2022.

\bibitem[Schramowski et~al.(2020)Schramowski, Stammer, Teso, Brugger, Herbert,
  Shao, Luigs, Mahlein, and Kersting]{schramowski2020making}
P.~Schramowski, W.~Stammer, S.~Teso, A.~Brugger, F.~Herbert, X.~Shao, H.-G.
  Luigs, A.-K. Mahlein, and K.~Kersting.
\newblock Making deep neural networks right for the right scientific reasons by
  interacting with their explanations.
\newblock \emph{Nature Machine Intelligence}, 2\penalty0 (8):\penalty0
  476--486, 2020.

\bibitem[Sch{\"u}tt et~al.(2017{\natexlab{a}})Sch{\"u}tt, Arbabzadah, Chmiela,
  M{\"u}ller, and Tkatchenko]{schuett2017dtnn}
K.~T. Sch{\"u}tt, F.~Arbabzadah, S.~Chmiela, K.~R. M{\"u}ller, and
  A.~Tkatchenko.
\newblock Quantum-chemical insights from deep tensor neural networks.
\newblock \emph{Nature Communications}, 8\penalty0 (1):\penalty0 13890,
  2017{\natexlab{a}}.

\bibitem[Sch{\"u}tt et~al.(2017{\natexlab{b}})Sch{\"u}tt, Kindermans,
  Sauceda~Felix, Chmiela, Tkatchenko, and M{\"u}ller]{schutt2017schnet}
K.~T. Sch{\"u}tt, P.-J. Kindermans, H.~E. Sauceda~Felix, S.~Chmiela,
  A.~Tkatchenko, and K.-R. M{\"u}ller.
\newblock Schnet: A continuous-filter convolutional neural network for modeling
  quantum interactions.
\newblock \emph{Advances in neural information processing systems},
  30:\penalty0 992–1002, 2017{\natexlab{b}}.

\bibitem[Sch\"utt et~al.(2018)Sch\"utt, Kessel, Gastegger, Nicoli, Tkatchenko,
  and M\"uller]{schuett2018schnetpack}
K.~T. Sch\"utt, P.~Kessel, M.~Gastegger, K.~A. Nicoli, A.~Tkatchenko, and K.-R.
  M\"uller.
\newblock Schnetpack: A deep learning toolbox for atomistic systems.
\newblock \emph{Journal of chemical theory and computation}, 15\penalty0
  (1):\penalty0 448--455, 2018.

\bibitem[Sch{\"u}tt et~al.(2018)Sch{\"u}tt, Sauceda, Kindermans, Tkatchenko,
  and M{\"u}ller]{schutt2018schnet}
K.~T. Sch{\"u}tt, H.~E. Sauceda, P.-J. Kindermans, A.~Tkatchenko, and K.-R.
  M{\"u}ller.
\newblock Schnet--a deep learning architecture for molecules and materials.
\newblock \emph{The Journal of Chemical Physics}, 148:\penalty0 241722, 2018.

\bibitem[Sch{\"{u}}tt et~al.(2019)Sch{\"{u}}tt, Gastegger, Tkatchenko, and
  M{\"{u}}ller]{schuett2019xaiqc}
K.~T. Sch{\"{u}}tt, M.~Gastegger, A.~Tkatchenko, and K.-R. M{\"{u}}ller.
\newblock Quantum-chemical insights from interpretable atomistic neural
  networks.
\newblock In \emph{Explainable {AI:} Interpreting, Explaining and Visualizing
  Deep Learning}, volume 11700 of \emph{Lecture Notes in Computer Science},
  pages 311--330. Springer, 2019.

\bibitem[Sch{\"u}tt et~al.(2019)Sch{\"u}tt, Gastegger, Tkatchenko, M{\"u}ller,
  and Maurer]{schuett2019schnorb}
K.~T. Sch{\"u}tt, M.~Gastegger, A.~Tkatchenko, K.-R. M{\"u}ller, and R.~J.
  Maurer.
\newblock Unifying machine learning and quantum chemistry with a deep neural
  network for molecular wavefunctions.
\newblock \emph{Nature Communications}, 10\penalty0 (1):\penalty0 5024, 2019.
\newblock \doi{10.1038/s41467-019-12875-2}.

\bibitem[Sch{\"u}tt et~al.(2021)Sch{\"u}tt, Unke, and
  Gastegger]{schutt2021equivariant}
K.~T. Sch{\"u}tt, O.~Unke, and M.~Gastegger.
\newblock Equivariant message passing for the prediction of tensorial
  properties and molecular spectra.
\newblock In \emph{International Conference on Machine Learning}, pages
  9377--9388. PMLR, 2021.

\bibitem[Schütt et~al.(2023)Schütt, Hessmann, Gebauer, Lederer, and
  Gastegger]{schutt2023schnetpack}
K.~T. Schütt, S.~S.~P. Hessmann, N.~W.~A. Gebauer, J.~Lederer, and
  M.~Gastegger.
\newblock {SchNetPack 2.0: A neural network toolbox for atomistic machine
  learning}.
\newblock \emph{The Journal of Chemical Physics}, 158\penalty0 (14):\penalty0
  144801, 2023.

\bibitem[Socher et~al.(2013)Socher, Perelygin, Wu, Chuang, Manning, Ng, and
  Potts]{socher2013recursive}
R.~Socher, A.~Perelygin, J.~Wu, J.~Chuang, C.~D. Manning, A.~Ng, and C.~Potts.
\newblock Recursive deep models for semantic compositionality over a sentiment
  treebank.
\newblock In \emph{Proceedings of the 2013 Conference on Empirical Methods in
  Natural Language Processing}, pages 1631--1642, Seattle, Washington, USA,
  Oct. 2013. Association for Computational Linguistics.

\bibitem[Stray(2020)]{stray2020aialignment}
J.~Stray.
\newblock Aligning ai optimization to community well-being.
\newblock \emph{International Journal of Community Well-Being}, 3\penalty0
  (4):\penalty0 443--463, 2020.

\bibitem[Sucholutsky et~al.(2023)Sucholutsky, Muttenthaler, Weller, Peng, Bobu,
  Kim, Love, Grant, Achterberg, Tenenbaum, et~al.]{sucholutsky2023getting}
I.~Sucholutsky, L.~Muttenthaler, A.~Weller, A.~Peng, A.~Bobu, B.~Kim, B.~C.
  Love, E.~Grant, J.~Achterberg, J.~B. Tenenbaum, et~al.
\newblock Getting aligned on representational alignment.
\newblock \emph{arXiv preprint arXiv:2310.13018}, 2023.

\bibitem[Sundararajan et~al.(2017)Sundararajan, Taly, and
  Yan]{sundar17axiomatic}
M.~Sundararajan, A.~Taly, and Q.~Yan.
\newblock Axiomatic attribution for deep networks.
\newblock In \emph{Proceedings of the 34th International Conference on Machine
  Learning - Volume 70}, ICML'17, page 3319–3328. JMLR.org, 2017.

\bibitem[Sutton and Barto(2018)]{sutton2018RLintro}
R.~S. Sutton and A.~G. Barto.
\newblock \emph{Reinforcement Learning: An Introduction}.
\newblock A Bradford Book, Cambridge, MA, USA, 2018.
\newblock ISBN 0262039249.

\bibitem[Telea(2004)]{talea2004inpainter}
A.~Telea.
\newblock An image inpainting technique based on the fast marching method.
\newblock \emph{Journal of Graphics Tools}, 9\penalty0 (1):\penalty0 23--34,
  2004.

\bibitem[Thomas et~al.(2018)Thomas, Smidt, Kearnes, Yang, Li, Kohlhoff, and
  Riley]{thomas2018tensor}
N.~Thomas, T.~Smidt, S.~Kearnes, L.~Yang, L.~Li, K.~Kohlhoff, and P.~Riley.
\newblock Tensor field networks: Rotation-and translation-equivariant neural
  networks for 3d point clouds.
\newblock \emph{arXiv preprint arXiv:1802.08219}, 2018.

\bibitem[Thrun(1994)]{thrun1994rule_extr_distr_rep}
S.~Thrun.
\newblock Extracting rules from artificial neural networks with distributed
  representations.
\newblock In \emph{Proceedings of the 7th International Conference on Neural
  Information Processing Systems}, NIPS'94, page 505–512, Cambridge, MA, USA,
  1994. MIT Press.

\bibitem[Tsukimoto(2000)]{tsukimoto200rule_extract_NNs}
H.~Tsukimoto.
\newblock Extracting rules from trained neural networks.
\newblock \emph{IEEE Transactions on Neural Networks}, 11\penalty0
  (2):\penalty0 377--389, 2000.
\newblock \doi{10.1109/72.839008}.

\bibitem[Unke and Meuwly(2019)]{unke2019physnet}
O.~T. Unke and M.~Meuwly.
\newblock Physnet: A neural network for predicting energies, forces, dipole
  moments, and partial charges.
\newblock \emph{Journal of chemical theory and computation}, 15\penalty0
  (6):\penalty0 3678--3693, 2019.

\bibitem[Unke et~al.(2021{\natexlab{a}})Unke, Chmiela, Gastegger, Sch{\"u}tt,
  Sauceda, and M{\"u}ller]{unke2021spookynet}
O.~T. Unke, S.~Chmiela, M.~Gastegger, K.~T. Sch{\"u}tt, H.~E. Sauceda, and
  K.-R. M{\"u}ller.
\newblock Spookynet: Learning force fields with electronic degrees of freedom
  and nonlocal effects.
\newblock \emph{Nature communications}, 12\penalty0 (1):\penalty0 7273,
  2021{\natexlab{a}}.

\bibitem[Unke et~al.(2021{\natexlab{b}})Unke, Chmiela, Sauceda, Gastegger,
  Poltavsky, Sch{\"u}tt, Tkatchenko, and M{\"u}ller]{unke2021machine}
O.~T. Unke, S.~Chmiela, H.~E. Sauceda, M.~Gastegger, I.~Poltavsky, K.~T.
  Sch{\"u}tt, A.~Tkatchenko, and K.-R. M{\"u}ller.
\newblock Machine learning force fields.
\newblock \emph{Chemical Reviews}, 121\penalty0 (16):\penalty0 10142--10186,
  2021{\natexlab{b}}.

\bibitem[Unke et~al.(2024)Unke, St{\"o}hr, Ganscha, Unterthiner, Maennel,
  Kashubin, Ahlin, Gastegger, Medrano~Sandonas, Berryman, Tkatchenko, and
  M{\"u}ller]{unke2022accurate}
O.~T. Unke, M.~St{\"o}hr, S.~Ganscha, T.~Unterthiner, H.~Maennel, S.~Kashubin,
  D.~Ahlin, M.~Gastegger, L.~Medrano~Sandonas, J.~T. Berryman, A.~Tkatchenko,
  and K.-R. M{\"u}ller.
\newblock Biomolecular dynamics with machine-learned quantum-mechanical force
  fields trained on diverse chemical fragments.
\newblock \emph{Science Advances}, 10\penalty0 (14):\penalty0 eadn4397, 2024.

\bibitem[{van der Velden} et~al.(2022){van der Velden}, Kuijf, Gilhuijs, and
  Viergever]{bas2022xai_med_image}
B.~H. {van der Velden}, H.~J. Kuijf, K.~G. Gilhuijs, and M.~A. Viergever.
\newblock Explainable artificial intelligence (xai) in deep learning-based
  medical image analysis.
\newblock \emph{Medical Image Analysis}, 79:\penalty0 102470, 2022.

\bibitem[Vaswani et~al.(2017)Vaswani, Shazeer, Parmar, Uszkoreit, Jones, Gomez,
  Kaiser, and Polosukhin]{vaswani2017attention}
A.~Vaswani, N.~Shazeer, N.~Parmar, J.~Uszkoreit, L.~Jones, A.~N. Gomez,
  L.~Kaiser, and I.~Polosukhin.
\newblock Attention is all you need.
\newblock In \emph{Advances in Neural Information Processing Systems},
  volume~30. Curran Associates, Inc., 2017.

\bibitem[Vielhaben et~al.(2024)Vielhaben, Lapuschkin, Montavon, and
  Samek]{VielhabenPR24virt_insp_layer}
J.~Vielhaben, S.~Lapuschkin, G.~Montavon, and W.~Samek.
\newblock Explainable ai for time series via virtual inspection layers.
\newblock \emph{Pattern Recognition}, 150:\penalty0 110309, 2024.
\newblock \doi{10.1016/j.patcog.2024.110309}.

\bibitem[Xiong et~al.(2022)Xiong, Schnake, Montavon, M{\"u}ller, and
  Nakajima]{xiong22asubgraph}
P.~Xiong, T.~Schnake, G.~Montavon, K.-R. M{\"u}ller, and S.~Nakajima.
\newblock Efficient computation of higher-order subgraph attribution via
  message passing.
\newblock In \emph{Proceedings of the 39th International Conference on Machine
  Learning}, volume 162 of \emph{Proceedings of Machine Learning Research},
  pages 24478--24495. PMLR, 17--23 Jul 2022.

\bibitem[Xiong et~al.(2023)Xiong, Schnake, Gastegger, Montavon, M\"uller, and
  Nakajima]{xiong23bMaxProd}
P.~Xiong, T.~Schnake, M.~Gastegger, G.~Montavon, K.~R. M\"uller, and
  S.~Nakajima.
\newblock Relevant walk search for explaining graph neural networks.
\newblock In \emph{Proceedings of the 40th International Conference on Machine
  Learning}, volume 202 of \emph{Proceedings of Machine Learning Research},
  pages 38301--38324. PMLR, 23--29 Jul 2023.

\bibitem[Ying et~al.(2019)Ying, Bourgeois, You, Zitnik, and
  Leskovec]{ying19gnnexplainer}
R.~Ying, D.~Bourgeois, J.~You, M.~Zitnik, and J.~Leskovec.
\newblock \emph{GNNExplainer: Generating Explanations for Graph Neural
  Networks}.
\newblock Curran Associates Inc., Red Hook, NY, USA, 2019.

\bibitem[Yuan et~al.(2023)Yuan, Yu, Gui, and Ji]{gxai_survey2023yuan}
H.~Yuan, H.~Yu, S.~Gui, and S.~Ji.
\newblock Explainability in graph neural networks: A taxonomic survey.
\newblock \emph{IEEE Transactions on Pattern Analysis and Machine
  Intelligence}, 45\penalty0 (05):\penalty0 5782--5799, may 2023.
\newblock ISSN 1939-3539.

\bibitem[Zaidan et~al.(2008)Zaidan, Eisner, and Piatko]{movie_reviews1}
O.~F. Zaidan, J.~Eisner, and C.~Piatko.
\newblock Machine learning with annotator rationales to reduce annotation cost.
\newblock In \emph{Proceedings of the NIPS*2008 Workshop on Cost Sensitive
  Learning}, December 2008.

\bibitem[Zeiler and Fergus(2014)]{zeiler2014occlusion}
M.~D. Zeiler and R.~Fergus.
\newblock Visualizing and understanding convolutional networks.
\newblock In \emph{Computer Vision -- ECCV 2014}, pages 818--833, Cham, 2014.
  Springer International Publishing.

\bibitem[Zhao et~al.(2019)Zhao, Zheng, Xu, and Wu]{zhao2029obj_detect}
Z.-Q. Zhao, P.~Zheng, S.-T. Xu, and X.~Wu.
\newblock Object detection with deep learning: A review.
\newblock \emph{IEEE Transactions on Neural Networks and Learning Systems},
  30\penalty0 (11):\penalty0 3212--3232, 2019.

\bibitem[Zhou et~al.(2024)Zhou, Zhang, Deng, Liu, Shen, Chan, and
  Zhang]{concept_interactions_2}
H.~Zhou, H.~Zhang, H.~Deng, D.~Liu, W.~Shen, S.-H. Chan, and Q.~Zhang.
\newblock Explaining generalization power of a dnn using interactive concepts.
\newblock \emph{Proceedings of the AAAI Conference on Artificial Intelligence},
  38\penalty0 (15):\penalty0 17105--17113, Mar. 2024.
\newblock \doi{10.1609/aaai.v38i15.29655}.

\bibitem[Zhuang and Hadfield-Menell(2020)]{zhuang2020misalignedAI}
S.~Zhuang and D.~Hadfield-Menell.
\newblock Consequences of misaligned ai.
\newblock In \emph{Advances in Neural Information Processing Systems},
  volume~33, pages 15763--15773. Curran Associates, Inc., 2020.

\end{thebibliography}

\newpage
\onecolumn
\appendix

\section{Notation}
\setlength{\tabcolsep}{10pt} % Default value: 6pt
\renewcommand{\arraystretch}{1.5} % Default value: 1
\begin{table}[!htb]
    \centering \small
    \begin{tabularx}{\textwidth}{c|X}
    \toprule
        $\bm{X}$ & Input sample for the ML model. $
        \bm{X}= (x_I)_{I \in \mathcal{N}}$ \\
        $\bm{X},x$ & Vector/matrix, scalar\\
        $\mathcal{N} = \{ 1, \dots, n \}$ & Set of input feature indices  \\
        $\mathcal{S}, \mathcal{L}$ & Subset of input feature indices, i.e., $\mathcal{S}, \mathcal{L} \subseteq \mathcal{N}$ \\
          $I,J,K \dots $ & Single index of features. $I,J,K \dots \in \mathcal{N}$\\
          $\mathcal{L} \subseteq \mathcal{S}$ & $\mathcal{L}$ is any subset of $\mathcal{S}$, including $\mathcal{L} = \mathcal{S}$\\
          $\mathcal{L} \subsetneq \mathcal{S}$ & $\mathcal{L}$ is a strict subset of $\mathcal{S}$, therefore $\mathcal{L} \neq \mathcal{S}$\\
           $2^\mathcal{S}$ & Power set of $\mathcal{S}$, i.e., $2^\mathcal{S} := \{ \mathcal{L}: \, \mathcal{L} \subseteq \mathcal{S} \} $ \\
          $\mathbbm{1}(\mathcal{L})$ & Indicator function for a set $\mathcal{L}$, i.e., it is 1 if $\mathcal{L} \neq \varnothing$ and else 0 \\
         $\overline{\mathcal{S}}$ & Complement of $\mathcal{S}$, i.e., $\overline{\mathcal{S}} = \mathcal{N} \setminus \mathcal{S}$\\
         $|\mathcal{S}|$ & Number of elements in the set $\mathcal{S}$ \\
         $\mathcal{W} = (I,J, K \dots)$ & Sequence of features also called walk.\\
         $\times $ & Cartesian product\\
         \bottomrule
    \end{tabularx}
    \caption{Notation throughout the paper}
    \label{tab:notation}
\end{table}

\section{Proof of Lemma \ref{lemm:conserv_query}} \label{app:proof_lemma_shap_weight}
\begin{proof}
From the definition of the query relevance $\mathcal{A}(q)$ we obtain
    \begin{align*}
         \sum_{k=1}^M \mathcal{A}(\bm{\eta}, \bm{\mu},  q_k) &= \sum_{k=1}^M \sum_{\mathcal{L} \subseteq \mathcal{N} }  \eta_\mathcal{L} \cdot  \mu_\mathcal{L} \cdot \lambda_\mathcal{L}( q_k) 
    \end{align*}
    Now, pulling the outer sum inside and using the definition of $\eta_\mathcal{L}$ we obtain
    \begin{align*}
         \sum_{k=1}^M \mathcal{A}(\bm{\eta}, \bm{\mu},  q_k) &=  \sum_{\mathcal{L} \subseteq \mathcal{N} }  \cancel{\eta_\mathcal{L}} \cdot  \mu_\mathcal{L} \cdot \cancel{[ \sum_{k=1}^M \lambda_\mathcal{L}( q_k)  ] }\\
         &=  \sum_{\mathcal{L} \subseteq \mathcal{N} }   \mu_\mathcal{L}\\
         & = f(\bm{X})
    \end{align*}
    where the last equation comes from the natural properties of the Harsanyi dividends.
\end{proof}

\section{More Details on Extracting Higher-Order Feature Relevance} \label{app:ho_xai}

\subsection{The machine learning model}
In our investigation, we focus on a particular class of neural network architectures, including graph neural networks (GNN) \cite{Scarselli2009GNN, Bronstein17GeomDL, kipf2017semisupervised}, Transformers \cite{vaswani2017attention, devlin2019bert, radford2018gpt}, and neural ODEs \cite{chen2018neuralODE, behrmann19InvResNet}. 
GNNs typically have a graph as input, while Transformer models can learn from various input types, such as text, images, and others. In all cases, the input sample is pre-processed into $n$ token embeddings of dimension $d_0$, resulting in an input matrix $\bm{H}_0 \in \mathbb{R}^{n \times d_0}$. These tokens represent nodes in a graph, words in a sentence, or predefined patches in an image. The token embeddings are learnable and trained during the model’s learning task.

Both GNNs and Transformer models generally consist of an \emph{encoder} and a \emph{decoder}. The encoder takes the token embeddings $\bm{H}_0$ as input and forwards them through $T$ \emph{building blocks} $\mathcal{B}_t: \mathbb{R}^{n \times d_t} \rightarrow \mathbb{R}^{n \times d_{t+1}}$. This produces a latent representation of dimension $d_t$ for all $n$ tokens in each layer $t= 1, \ldots, T$, resulting in a matrix $\bm{H}_t \in \mathbb{R}^{n \times d_t}$. The decoder, specified as the function $g$, takes the output of the encoder $\bm{H}_T$ as input and processes the target prediction $y$.

Formally, this class of models follows the forward propagation scheme:
\begin{align}
&\text{encoder: } &\bm{H}_t &= \mathcal{B}_t\big( \bm{H}_{t-1} \big) & t = 1, \dots, T \label{eq:encoder_DPM} \\
&\text{decoder: } &y &= g\big(\bm{H}_T\big)\label{eq:decoder_DPM}
\end{align}
The encoder and decoder contain learnable parameters which are trained during the learning process. The building blocks are comparable to layers in a feed-forward neural network. In GNNs, the building blocks also incorporate the structure of the input graph, which is discussed further in \cite{schnake22gnnlrp}. The encoder and decoder usually consists of another set of neural network layers, such a feed-forward neural networks, but also Attention heads or LayerNorm \cite{vaswani2017attention}. 

In \cite{schnake22gnnlrp}, the authors demonstrate how we can leverage the fact that each building block $t$ contains a latent representation $\bm{H}_{t,I}$ of an input token $I$. In practice, this vector can be associated with a processed representation of the input token, enabling us to specify a higher-order explanation method for these types of models.

\subsection{Introducing multi-linear local surrogates}
\label{sect:build_mls}

We build the multi-linear local surrogates (MLS) gradually by substituting the nonlinear components of the model
by linear components. Our approach avoids any time-consuming retraining steps. However, it has hyperparameters that can be tuned to achieve the desired characteristics of the MLS. Such substitution is informed by the data used to collect the activations, thereby bringing locality.

We now provide an exemplary computation of the MLS for a GNN model with linear aggregation and feed forward combine function (e.g. given in \cite{kipf2017semisupervised}). In this context, we differentiate between two components in each building block $\mathcal{B}_t$: the \textit{combine function} and the \textit{aggregation function}. The \textit{combine function} is a function that acts on each feature representation individually, without the incorporation of other feature representations. The \textit{aggregation function} aggregates information across the input features. This distinction was also introduced by Gilmer et al. 2017 \cite{gilmer2017neural}.

We can then specify the combine function as
\begin{align}
	\bm{H}_{K,a}^{(t)} = \sigma(W_a^\top \bm{Z}_K^{(t-1)})
\end{align}
where $\sigma$ is a function that crosses the origin. We build the MLS by applying the following substitution at each layer
\begin{align}
	\bm{H}_{K,a}^{(t)} = \hat{V}_{K,a}^\top \bm{Z}_K^{(t-1)} \quad \text{with} \quad \hat{V}_{K,a}= \rho(W_a) \cdot \frac{\sigma(W_a^\top \bm{Z}_K^{(t-1)})}{\rho(W_a)^\top \bm{Z}_K^{(t-1)}}
\end{align}
Here, $\rho$ is a positive homogeneous function applied element-wise, and we use the convention that $0/0=0$. The linear aggregation function can be expressed by $\bm{Z}^{(t)}_K  = \sum_{J } \lambda_{JK} \bm{H}^{(t-1)}_J $, where $\lambda_{JK}$ represents the connectivity weight between nodes $J$ and $K$.
For this aggregation, we do not do any substitution.
Together, we obtain an alternative forward formulation, or alternative building block, of the model by
\begin{align}
	\bm{H}_K^{(t)} &=  \sum_{J} \underbrace{\lambda_{JK} \hat{\bm{V}}_K^{(t)}}_{=\bm{V}_{JK}^{(t)}} \bm{H}_J^{(t-1)} \label{eq:multilin_model}
\end{align}
Similarly, we proceed with the output function $g\big(  (\bm{H}_L^{(T)})_L \big)$ and obtain a locally linear function $\hat{g}$, such that $\hat{g}((\bm{H}_L^{(T)})_L \big)) = f(\bm{X})$.

We stress that almost any model, that can be expressed by Equations \eqref{eq:encoder_DPM} and \eqref{eq:decoder_DPM}, can be transformed into a multi-linear local surrogate of the form given in Equation \eqref{eq:multilin_model}. For more details on how other models can have such a representation, we refer to \cite{schnake22gnnlrp}.

\subsection{Using LRP for higher-order feature relevance}
The biggest bottleneck in making neural network predictions understandable lies in different non-linearities present in the model, such as neural activation functions \cite{DUBEY2022activation}, gates \cite{hochreiter1997long,cho2014gru}, or attention heads \cite{bahdanau2015, vaswani2017attention} in each layer. This increases the complexity of the model massively and makes its predictions intractable for a human. Yet, we argue that for multi-linear models, the contributions of feature index sets and higher-order interactions can be explicitly extracted.

As specified by the MLS, layer-wise relevance propagation (LRP) is a linearization method for neural networks. It provides a local approximation of the model's prediction at the input $\bm{X}$.
Within the model's description from Equations \eqref{eq:encoder_DPM} and \eqref{eq:decoder_DPM}, there exists therefore  matrices $\bm{V}_{JK}^{(t)}$ and  $\hat{g}$ such that
\begin{align*}
     \bm{H}_K^{(t)} =  \sum_{J} \bm{V}_{JK}^{(t)} \bm{H}_J^{(t-1)} \hspace{1cm} t= 1, \dots T \hspace{.3cm}\text{and } \hspace{1.5cm}
     f(\bm{X}) = \hat{g}\big(\bm{H}^{(T)}\big)
\end{align*}
With this expression we can deduce a simple strategy to extract high-order explanatory features.

\paragraph{Higher-order layer-wise relevance propagation (HO-LRP)}
For the propagation scheme introduced in Equation \eqref{eq:multilin_model}, we can extract features of the form $R_\mathcal{W}$ for a chain of feature indices $\mathcal{W}=(I,J, \dots L)$. This is simply given by $ R_{IJ \dots L} = \langle \hat{g}_L\bm{V}^{(T)}_{L\cdot} \dots \,  \bm{V}^{(2)}_{{\cdot J}} \bm{V}_{{JI}}^{(1)},\bm{H}^{(0)}_I \rangle $. We expect that $R_{\mathcal{W}}$ is mainly dependent on the features indexed by the elements in $\mathcal{W}$. Although this expectation is not rigorously true for models that incorporate non-linearities in their building blocks $\mathcal{B}_t$, this is statistically a quite accurate assumption, as we will demonstrate later. This method is referred to as \textit{higher-order layer-wise relevance propagation} (HO-LRP), extending GNN-LRP \cite{schnake22gnnlrp} to a broader class of models that fulfill Equations \eqref{eq:encoder_DPM} and \eqref{eq:decoder_DPM}.

\section{Link Between Propagation- and Perturbation-based Methods} \label{app:express_multi_with_walks}
We aim to formally show under which circumstances the specification of $\mu$ in Equations \eqref{eq:multi_order_gnnlrp} and \eqref{eq:multi_ord_def_moebius} coincide. We formulate it in a Lemma:

\begin{lemma}
    If we use the subset relevance  $R_\mathcal{L} = \sum_{\mathcal{W} \in \mathcal{L} \times \mathcal{L} \times \dots } R_\mathcal{W}$ as an estimate for the model's prediction $f(\bm{X}_\mathcal{L})$ for a set of features $\mathcal{L}$, then the Harsanyi dividends in Equation \eqref{eq:multi_ord_def_moebius} coincide with the terms in Equation \eqref{eq:multi_order_gnnlrp}. Formally, we want to show that for any subset $\mathcal{L}$:
    \begin{align}\label{eq:lemma_harsanyi_propagation_link}
        \sum_{\mathcal{W}: \, \text{set}(\mathcal{W}) = \mathcal{L}} R_\mathcal{W} = \underbrace{\sum_{\mathcal{S} \subseteq \mathcal{L}} \left( -1 \right)^{|\mathcal{L} \setminus \mathcal{S}|} R_{\mathcal{S}}}_{=: \Delta R_\mathcal{L}}
    \end{align}
\end{lemma}

\begin{proof}
    We prove this by induction over $|\mathcal{L}|$.

    \textbf{Base Case:} For $|\mathcal{L}| = 1$, let $\mathcal{L} = \{J\}$. The left-hand side of Equation \eqref{eq:lemma_harsanyi_propagation_link} is:
    \begin{align*}
        \sum_{\mathcal{W}: \text{set}(\mathcal{W}) = \{J\}} R_\mathcal{W} = R_{JJJ \dots}
    \end{align*}
    The right-hand side is:
    \begin{align*}
        \sum_{\mathcal{S} \subseteq \{J\}} \left( -1 \right)^{|\{ J \} \setminus \mathcal{S}|} R_{\mathcal{S}} = -\underbrace{R_{\varnothing}}_{=0} + R_{\{J\}} = R_{JJJ \dots}
    \end{align*}
    Hence, the base case is satisfied.

    \textbf{Induction Step:} Assume the statement holds for $|\mathcal{L}| \leq l$. We want to show it for $|\mathcal{L}| = l+1$.

    It is known (see \cite{FUJIMOTO2006AxiomInterIndex}) that $R_{\mathcal{L}} = \sum_{\mathcal{S} \subseteq \mathcal{L}} \Delta R_{\mathcal{S}}$. From this, we can conclude:
    \begin{align*}
        \Delta R_{\mathcal{L}} = R_{\mathcal{L}} - \sum_{\mathcal{S} \subsetneq \mathcal{L}} \Delta R_{\mathcal{S}}
    \end{align*}
    For each $\mathcal{S} \subsetneq \mathcal{L}$, we have $|\mathcal{S}| \leq l$. By the induction hypothesis, we get:
    \begin{align}\label{eq:app_eq1}
        \Delta R_{\mathcal{L}} = R_{\mathcal{L}} - \sum_{\mathcal{S} \subsetneq \mathcal{L}} \ \ \sum_{\mathcal{W}:\, \text{set}(\mathcal{W}) = \mathcal{S}} R_\mathcal{W}
    \end{align}
    For $\mathcal{S} \subseteq \mathcal{L}$, the sets $\{ \mathcal{W}: \text{set}(\mathcal{W}) = \mathcal{S} \}$ are pairwise disjoint and their union specifies all walks in $\mathcal{L} \times \mathcal{L} \times \dots$. This implies:
    \begin{align*}
        R_{\mathcal{L}} = \sum_{\mathcal{S} \subseteq \mathcal{L}} \ \ \sum_{\mathcal{W}:\, \text{set}(\mathcal{W}) = \mathcal{S}} R_\mathcal{W}
    \end{align*}
    Inserting this into Equation \eqref{eq:app_eq1}, we directly obtain:
    \begin{align*}
        \Delta R_{\mathcal{L}} = \sum_{\mathcal{W}:\, \text{set}(\mathcal{W}) = \mathcal{L}} R_\mathcal{W}
    \end{align*}
    which is what we wanted to show.
\end{proof}

\section{Definition of the weighted correlation}\label{app:corr_def}
We define the weighted correlation $\text{corr}_{\bm{\eta}}(\bm{X},\bm{Y})$ for some weight vector $\bm{\eta} \in \mathbb{R}^d$ and some input vectors $\bm{X},\bm{Y} \in \mathbb{R}^d$. To do so, we first define the weighted covariance $\text{cov}_{\bm{\eta}}$ which is defined by
\begin{align*}
    \text{cov}_{\bm{\eta}}(\bm{X},\bm{Y}) = \mathbb{E}_{\bm{\eta}}[(\bm{X} - \mathbb{E}_{\bm{\eta}} [\bm{X}
    ])(
    \bm{X}- \mathbb{E}_{\bm{\eta}} [\bm{Y}
    ])]
\end{align*}
where $\mathbb{E}_{\bm{\eta}}[\bm{X}] := \frac{1}{\sum_{i}\eta_i}\sum_i x_i\eta_i$. Now 
\begin{align*}
    \text{corr}_{\bm{\eta}}(\bm{X},\bm{Y}) = \frac{\text{cov}_{\bm{\eta}}(\bm{X},\bm{Y})}{\sqrt{\text{Var}_{\bm{\eta}}(\bm{X}) \text{Var}_{\bm{\eta}}(\bm{Y})}}
\end{align*}
where $\text{Var}_{\bm{\eta}}(\bm{X}) := \text{cov}_{\bm{\eta}}(\bm{X},\bm{X})$.

\section{Details on the Transformer models }\label{app:transformer_additionals}
\subsection{On facial expression task}\label{app:cv_model_details}
In our vision experiments, we used a pretrained ViT-Base \cite{dosovitskiy2021ViT} model\footnote{\url{https://huggingface.co/dima806/facial_emotions_image_detection}}, which is finetuned on the FER-2013 \cite{goodfellow2013fer2013} dataset. The FER-2013 dataset is a facial expression dataset that is composed of 28,709 and 3,589 train and test examples, respectively. Each example in this dataset is a grayscale 48$\times$48 image. The images are categorized into seven emotion classes: `angry', `disgust', `fear', `happy', `sad', `surprise', and `neutral'. The overall accuracy of the pretrained model on this dataset is 90\%. Further details about the ViT-Base model is represented in Table~\ref{tab:model_details}.

\subsection{On the sentiment analysis task}\label{app:nlp_model_details}
In our NLP experiments, we used pretrained BERT-Base, Uncased \cite{devlin2019bert} models, which are finetuned on the SST-2 \cite{socher2013recursive}\footnote{\url{https://huggingface.co/textattack/bert-base-uncased-SST-2}} and IMDB \cite{movie_reviews2}\footnote{\url{https://huggingface.co/textattack/bert-base-uncased-imdb}} datasets. The classification accuracies on the SST-2 and IMDB datasets are 92.43\% and 91.90\%, respectively. We show further details about the BERT-Base model in Table~\ref{tab:model_details}.

\begin{table}[]
    \centering
    \begin{tabular}{c|cccc}
         \toprule
         \textbf{Model} & \textbf{Parameters} & \textbf{Layers} & \textbf{Hidden size $D$} & \textbf{Heads}\\
         \midrule
         BERT-Base, Uncased & 110M & 12 & 768 & 12\\
         ViT-Base & 86M & 12 & 768 & 12\\
         \bottomrule
    \end{tabular}
    \caption{Further details of Transformer models used in the NLP and vision experiments.}
    \label{tab:model_details}
\end{table}

\section{Additional results on the usage of SymbXAI framework in NLP} \label{app:add_nlp_eval}
\subsection{Movie Reviews evaluation}

In our experiments using the Movie Review dataset, we leverage a BERT-Base, Uncased model that has been pre-trained on the Movie Reviews dataset. Figures \ref{fig:nlp_imdb_neg} and \ref{fig:nlp_imdb_pos} display two exemplary reviews from this dataset, along with their corresponding explanations provided by our SymbXAI framework. We utilize the queries $q=\mathcal{S}$ and $q=\mathcal{S} \wedge \neg\overline{\mathcal{S}}$ to assess the model's ability to comprehend the sentiment expressed through these rationales in two different contextual settings. In these queries, we use $\mathcal{S}$ to denote a set of word indices that represent a subsentence. When using the query $q=\mathcal{S}$, the context is disregarded. However, when using $q=\mathcal{S} \wedge \neg\overline{\mathcal{S}}$, the entire context is taken into consideration. Semantically, the relevance for queries that consider the full context is more akin to the strategy of the human annotators, as they also take the complete context into account.

In Figure \ref{fig:nlp_imdb_neg}, we present a negative review accompanied by 8 pieces of supporting evidence that explain why the review is considered negative. By examining the relevance scores of the rationales, depicted as bar plots, it becomes evident that the model accurately comprehends the sentiment of all evidence pieces when considering the complete context. However, its performance in this regard diminishes when the context is absent. A similar observation can be made by studying Figure \ref{fig:nlp_imdb_pos}, which showcases a positive review along with five rationales. When the full context is taken into account, the model precisely identifies all rationales as positive. Conversely, in the absence of the complete context, the relevance scores of the second and fifth rationales nearly drop to zero.

\begin{figure*}[ht]
    \centering
    \includegraphics[width=.9\textwidth]{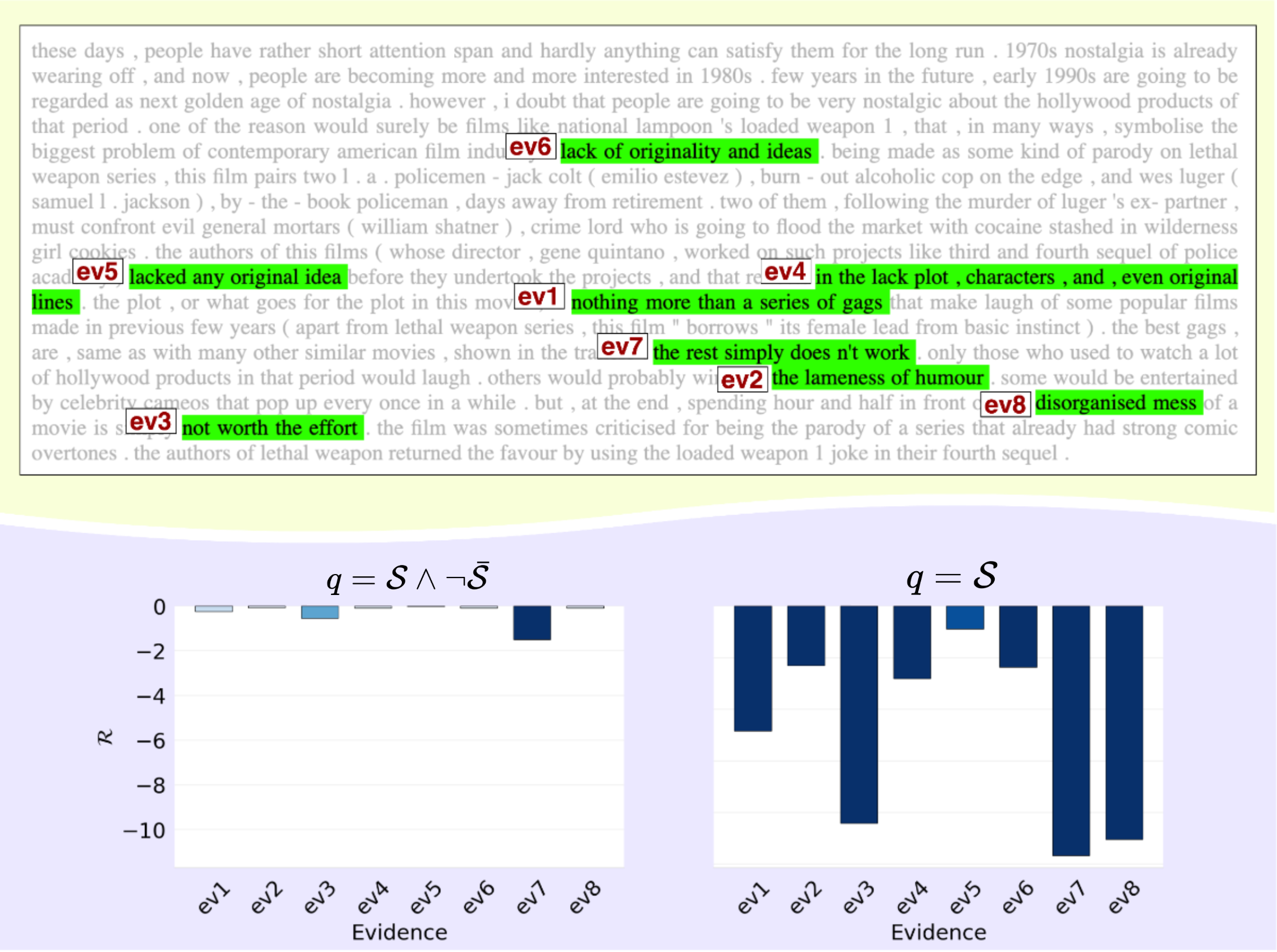}
    \caption{An exemplary negative review of the Movie Reviews dataset. In this example, the annotators provided eight pieces of evidence to demonstrate why this review is negative. To assess the model's ability to understand the negativity conveyed by the rationales, we used the queries
    $q=S$ and $q=S \wedge \neg\overline{\mathcal{S}}$. The relevance scores are visualized as bar plots, where negativity and positivity are distinguished by varying shades of blue and red, respectively.
    }
    \label{fig:nlp_imdb_neg}
\end{figure*}

\begin{figure*}[ht]
    \centering
    \includegraphics[width=.9\textwidth]{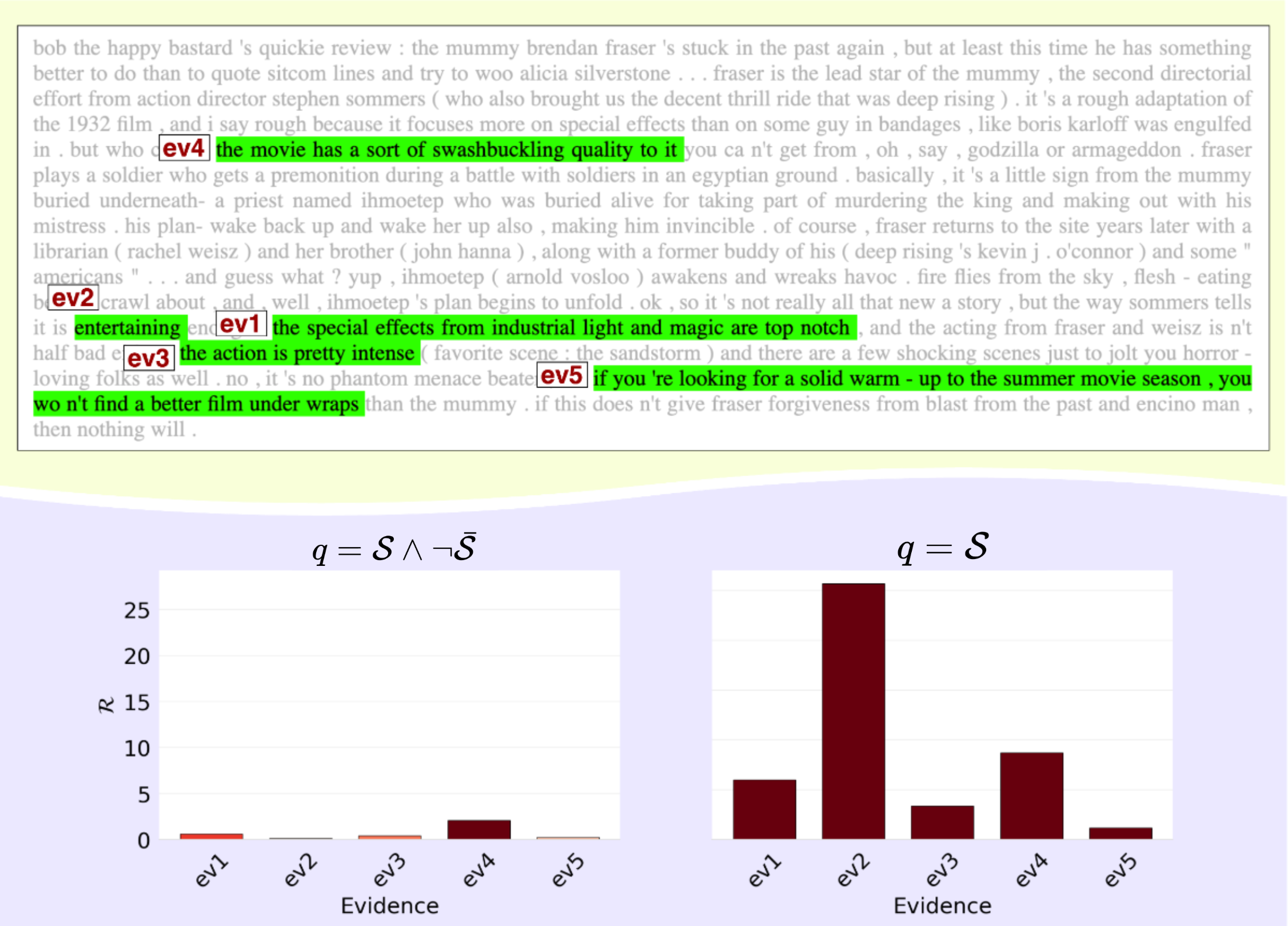}
    \caption{An exemplary positive review of the Movie Reviews dataset. In this particular instance, the annotators have presented five pieces of evidence supporting the positive nature of this review. To assess the model's ability to comprehend the positivity reflected in the rationales, we employed the queries
    $q=\mathcal{S}$ and $q=\mathcal{S} \wedge \neg\overline{\mathcal{S}}$. The relevance scores are visualized as bar plots, where negativity and positivity are distinguished by varying shades of blue and red, respectively.
    }
    \label{fig:nlp_imdb_pos}
\end{figure*}

\subsection{Additional to the perturbation analysis}\label{app:additional_perturbation_analysis}
We see in Figure \ref{fig:perturbation_cuves_sst} and \ref{fig:perturbation_cuves_imdb} the perturbation curves on the SST \cite{socher2013recursive} and Movie Reviews \cite{movie_reviews1, movie_reviews2} dataset, respectively. The perturbation curves that the order of the perturbation has a big effect on the curves. The curvature is approximately convex and concave for all methods on all datasets when considering the maximization and minimization task respectively. Particularly in the Movie Reviews dataset, the SymbXAI framework the perturbation curves have a visible margin to the first order methods. 
\begin{figure}
    \centering
    \includegraphics[width=\textwidth]{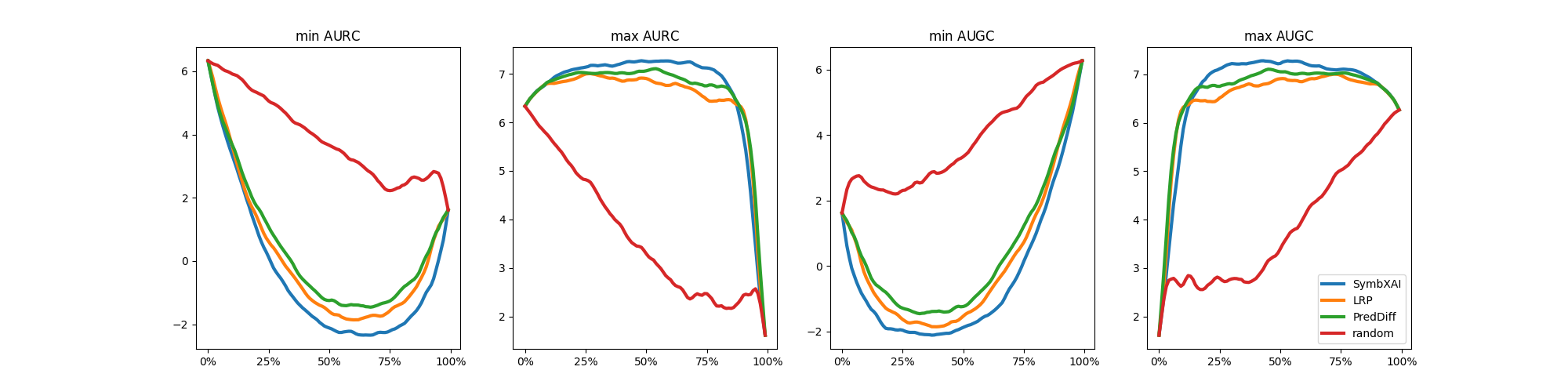}
    \caption{Perturbation curves for the different explanation methods on the SST dataset, averaged over 170 samples.}
    \label{fig:perturbation_cuves_sst}
\end{figure}

\begin{figure}
    \centering
    \includegraphics[width=\textwidth]{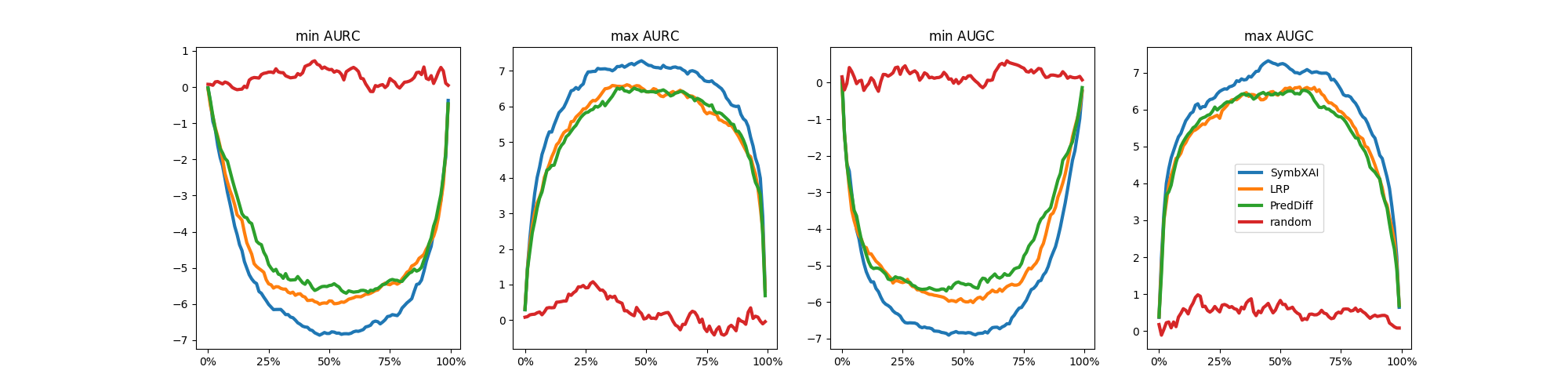}
    \caption{Perturbation curves for the different explanation methods on the Movie Reviews dataset, averaged over 97 samples.}
    \label{fig:perturbation_cuves_imdb}
\end{figure}

\section{ Details about quantum chemistry model architecture, training and performance}\label{app:qc_train_data}

For the analysis of the proton transfer in malondialdehyde in Section~\ref{sec:mda_traj}, we have trained SchNet~\cite{schutt2017schnet, schutt2018schnet} on the reference data introduced by Schütt~et.~al.~\cite{schuett2019schnorb} using the software package SchNetPack~\cite{schuett2018schnetpack, schutt2023schnetpack}. The model exhibits three interaction blocks, a feature size of 128, and a distance expansion by 20 Gaussian radial basis functions. The continuous filter convolutions are performed for interatomic distances up to a cutoff radius of \SI{5}{\angstrom}.

The loss function of our optimization problem for training reads
\begin{equation}
    \mathcal{L} = \alpha \vert \epsilon - \hat{\epsilon}\vert^2 + 
    \frac{\beta}{3N} \Vert \mathbf{F} - \hat{\mathbf{F}}(\hat{\epsilon})\Vert^2~.
\end{equation}
with the trade-off factors $\alpha$, $\beta$, the predicted energy value $\hat{\epsilon}$, the total number of atoms $N$, and the predicted forces $\hat{\mathbf{F}}(\hat{\epsilon})$. Here we choose $\alpha=0.1$ and $\beta=0.9$. We utilize the AdamW~\cite{kingma2014adam, loshchilov2017decoupled} optimizer provided by PyTorch with a weight decay of $0.01$. We randomly split the entire dataset into training set, validation set and test set. Training and validation set comprise $85\%$ and $10\%$ of the dataset respectively. The remaining samples are assigned to the test set. Starting from $a_\mathrm{lr}=5\cdot 10^{-4}$, we gradually adapt the learning rate using a learning rate scheduler with a patience of $80$ epochs and a threshold of $0.0001$; i.e., each time the validation loss did not decrease more than the threshold after $80$ epochs, the learning rate is reduced by a factor of $0.8$. Eventually, when the training has not improved the model prediction error on the validation set for more than 1000 epochs the training is stopped. The training has converged after $7755$ epochs. The mean absolute prediction error on the left out test samples is \SI{0.35}{\milli\electronvolt} for the energy and \SI{2.16}{\milli\electronvolt/\angstrom} for the forces.

\section{Details on the Lennard-Jones Potential} \label{app:len_jones}
Lennard-Jones-Potential $V$ is given by
\begin{align*}
   V(u) = V_0  \left[(\frac{u_0}{u})^{12} - 2  (\frac{u_0}{u})^6 \right]
\end{align*}
with $u_0$ begin the equilibrium state distance of the considered bond, and $V_0$ is the depth of the potential well, which is in our case always fixed to 1. The equilibrium distance $u_0$ we obtain by performing an MD simulation with the ML model, towards the zero potential. The resulting MDA molecule is considered to be in equilibrium state. We then use the distance between the bonded O-H atoms as the reference value for $u_0$ which is $1.04$ Å.

\end{document}